%% file: main.tex
\documentclass{article}

\usepackage[preprint]{neurips_2023}

\usepackage[utf8]{inputenc} %
\usepackage[T1]{fontenc}    %
\usepackage{hyperref}
\usepackage{url}            %
\usepackage{booktabs}       %
\usepackage{amsfonts}       %
\usepackage{nicefrac}       %
\usepackage{microtype}      %
\usepackage{xcolor}         %

\usepackage{graphicx}
\usepackage[position=top]{subfig}
\usepackage{multirow}
\usepackage{multicol}
\usepackage[normalem]{ulem}
\usepackage{wrapfig}

\usepackage[toc,page,header]{appendix}
\usepackage{minitoc}

\definecolor{mydarkblue}{rgb}{0,0.08,0.45}
\definecolor{mydarkgreen}{rgb}{0,0.45,0.08}
  \hypersetup{
    colorlinks=true,
    linkcolor=mydarkblue,
    citecolor=mydarkblue,
    filecolor=mydarkblue,
    urlcolor=mydarkblue
}

\makeatletter
\def\@fnsymbol#1{\ensuremath{\ifcase#1\or * \or w \or p \or mcp \or \dagger\or \mathsection\or
   \ddager\or \mathparagraph\or \|\or **\or \dagger\dagger
   \or \ddagger\ddagger \else\@ctrerr\fi}}
\makeatletter

\title{Teaching Arithmetic to Small Transformers}

\author{%
  Nayoung Lee\thanks{Authors contributed equally to this paper.} \\
  University of Wisconsin-Madison\\
  \texttt{nayoung.lee@wisc.edu} \\
  \And
  Kartik Sreenivasan\footnotemark[1] \\
  University of Wisconsin-Madison\\
  \texttt{ksreenivasa2@wisc.edu} \\
  \And
  Jason D. Lee \\
  Princeton University\\
  \texttt{jasonlee@princeton.edu} \\
  \And
  Kangwook Lee \\
  University of Wisconsin-Madison\\
  \texttt{kangwook.lee@wisc.edu} \\
  \And
  Dimitris Papailiopoulos \\
  University of Wisconsin-Madison\\
  \texttt{dimitris@papail.io} \\
}

\input{math_commands.tex}

\usepackage{amsmath}
\usepackage{amssymb}
\usepackage{mathtools}
\usepackage{amsthm}
\usepackage{thm-restate}

\usepackage{cuted}
\usepackage{flushend}

\usepackage{soul}

\usepackage[linesnumbered,ruled,vlined]{algorithm2e}
\usepackage{algorithmicx}
\usepackage{algpseudocode}

\usepackage[capitalize,noabbrev]{cleveref}

\usepackage[most]{tcolorbox}

\definecolor{commentcolour}{rgb}{0.3,0.7,0.2}
\definecolor{backcolour}{rgb}{0.96,0.96,0.96}

\usepackage{soul}
\usepackage{adjustbox}

\definecolor{lightgreen}{RGB}{222, 242, 216}
\sethlcolor{lightgreen}
\makeatletter
\def\SOUL@hlpreamble{%
    \setul{}{3.5ex}%
    \let\SOUL@stcolor\SOUL@hlcolor
    \dimen@\SOUL@ulthickness
    \dimen@i=-.75ex %
    \advance\dimen@i-.5\dimen@
    \edef\SOUL@uldepth{\the\dimen@i}%
    \let\SOUL@ulcolor\SOUL@stcolor
    \SOUL@ulpreamble
}
\makeatother

\lstdefinelanguage{markdown}{
    comment=[l]{\#},
    morestring=[s]{```}{```},
    commentstyle=\color{commentcolour}\bfseries,
    stringstyle=\color{blue},
    basicstyle=\scriptsize\ttfamily,
    showstringspaces=false,
    breaklines=true,
    breakautoindent=false,
    breakindent=0pt,
    backgroundcolor=\color{backcolour},
    escapeinside={(*@}{@*)},
}

\newcommand{\highlighttext}[1]{%
  \adjustbox{margin=0.2ex 0.23ex 1ex 0.23ex, bgcolor=lightgreen}{#1}%
}
\usepackage[frozencache,cachedir=minted-cache]{minted}

\lstdefinelanguage{markdown2}{
    morestring=[s]{```}{```},
    commentstyle=\color{commentcolour}\bfseries,
    stringstyle=\color{blue},
    basicstyle=\scriptsize\ttfamily,
    showstringspaces=false,
    breaklines=true,
    breakautoindent=false,
    breakindent=0pt,
    backgroundcolor=\color{lightgreen},
    escapeinside={(*@}{@*)},
}

\lstdefinestyle{mystyle}{
    morekeywords={self},
    basicstyle=\scriptsize\ttfamily,
    keywordstyle=\color{blue},
    commentstyle=\color{commentcolour}\bfseries,
    breaklines=true,
    breakautoindent=false,
    showstringspaces=false,
    backgroundcolor=\color{backcolour},
    stringstyle=\color{red}
}
\lstdefinelanguage{PythonPlus}[]{Python}{
  morekeywords=[1]{,as,assert,nonlocal,with,yield,self,True,False,None,} %
  morekeywords=[2]{,__init__,__add__,__mul__,__div__,__sub__,__call__,__getitem__,__setitem__,__eq__,__ne__,__nonzero__,__rmul__,__radd__,__repr__,__str__,__get__,__truediv__,__pow__,__name__,__future__,__all__,}, %
  morekeywords=[3]{,object,type,isinstance,copy,deepcopy,zip,enumerate,reversed,list,set,len,dict,tuple,range,xrange,append,execfile,real,imag,reduce,str,repr,}, %
  morekeywords=[4]{,Exception,NameError,IndexError,SyntaxError,TypeError,ValueError,OverflowError,ZeroDivisionError,}, %
  morekeywords=[5]{,ode,fsolve,sqrt,exp,sin,cos,arctan,arctan2,arccos,pi, array,norm,solve,dot,arange,isscalar,max,sum,flatten,shape,reshape,find,any,all,abs,plot,linspace,legend,quad,polyval,polyfit,hstack,concatenate,vstack,column_stack,empty,zeros,ones,rand,vander,grid,pcolor,eig,eigs,eigvals,svd,qr,tan,det,logspace,roll,min,mean,cumsum,cumprod,diff,vectorize,lstsq,cla,eye,xlabel,ylabel,squeeze,}, %
}

\tcbset{
  aibox/.style={
    width=\textwidth,
    top=0pt, bottom=0pt, left=5pt, right=5pt,
    colback=white,
    colframe=black,
    colbacktitle=black,
    enhanced,
    center,
    attach boxed title to top left={yshift=-0.1in,xshift=0.15in},
    boxed title style={boxrule=0pt,colframe=white,},
  }
}
\newtcolorbox{AIbox}[2][]{aibox,title=#2,#1}

\newtcolorbox{mytextbox}[1][]{%
  sharp corners,
  enhanced,
  colback=white,
  height=10cm,
  attach title to upper,
  #1
}

\newtcolorbox{highlightedCodeBox}{
  colback=lightgreen,
  boxrule=0pt,
  arc=0pt,
  boxsep=0pt,
  enhanced jigsaw,
  breakable
}

\usepackage{dsfont}
\newcommand{\eg}{{\it e.g.}, }
\newcommand{\ie}{{\it i.e.}, }

\theoremstyle{definition}

\renewcommand{\arraystretch}{1.1} %

\ifdefined\usebigfont

\usepackage{times}
\usepackage[fontsize=13pt]{scrextend}
\makeatletter
\@ifpackageloaded{geometry}{\AtBeginDocument{\newgeometry{letterpaper,left=1.56in,right=1.56in,top=1.71in,bottom=1.77in}}}{\usepackage[letterpaper,left=1.56in,right=1.56in,top=1.71in,bottom=1.77in]{geometry}}
\makeatother
\else
\fi

\usepackage{thm-restate}

\appto\appendix{\addtocontents{toc}{\protect\setcounter{tocdepth}{0}}}

\begin{document}
\doparttoc %

\maketitle

\begin{abstract}
Large language models like GPT-4 exhibit emergent capabilities across general-purpose tasks, such as basic arithmetic, when trained on extensive text data, even though these tasks are not explicitly encoded by the unsupervised, next-token prediction objective. This study investigates how small transformers, trained from random initialization, can efficiently learn arithmetic operations such as addition, multiplication, and elementary functions like square root, using the next-token prediction objective.
We first demonstrate that conventional training data is not the most effective for arithmetic learning, and simple formatting changes can significantly improve accuracy. This leads to sharp phase transitions as a function of training data scale, which, in some cases, can be explained through connections to low-rank matrix completion. Building on prior work, we then train on chain-of-thought style data that includes intermediate step results. Even in the complete absence of pretraining, this approach significantly and simultaneously improves accuracy, sample complexity, and convergence speed.
We also study the interplay between arithmetic and text data during training and examine the effects of few-shot prompting, pretraining, and model scale. Additionally, we discuss length generalization challenges. Our work highlights the importance of high-quality, instructive data that considers the particular characteristics of the next-word prediction objective for rapidly eliciting arithmetic capabilities.\footnote{Our code is available at \url{https://github.com/lee-ny/teaching_arithmetic}}
\newpage
\tableofcontents
\clearpage

\end{abstract}

\input{_Intro}

\input{_RelatedWorks}

\input{_Addition}

\input{_Discussion}
\input{_Conclusion}

\bibliography{ref}
\bibliographystyle{icml2023}

\appendix
\addcontentsline{toc}{section}{Appendix} %
\part{Appendix}
\parttoc %

\input{_Appendix_New_Exp}

\input{_Appendix_Setup}

\input{_Appendix_Prompts}

\end{document}

%% file: math_commands.tex
\usepackage{amsmath,amsfonts,bm}

\def\eqref#1{equation~\ref{#1}}

\def\1{\bm{1}}

\def\vone{{\bm{1}}}

\def\mM{{\bm{M}}}
\def\mN{{\bm{N}}}

\DeclareMathAlphabet{\mathsfit}{\encodingdefault}{\sfdefault}{m}{sl}
\SetMathAlphabet{\mathsfit}{bold}{\encodingdefault}{\sfdefault}{bx}{n}

\def\gA{{\mathcal{A}}}

\def\gD{{\mathcal{D}}}

\def\gO{{\mathcal{O}}}

%% file: _Intro.tex
\section{Introduction}\label{sec:intro}

Large language models like GPT-3/4, PaLM, LaMDA~\citep{brown2020language, chowdhery2022palm, thoppilan2022lamda} have demonstrated general-purpose properties, often referred to as \emph{emergent abilities}~\citep{wei2022emergent}, for a wide range of downstream tasks like language and code translation, compositional reasoning, and basic arithmetic operations~\citep{webb2022emergent, nye2021show, wei2022chain, shi2022language, wang2022self, srivastava2022beyond, chen2023teaching}. 
What is perhaps surprising, is that these tasks are not explicitly encoded in the model's training objective, which typically is an auto-regressive, next-token-prediction loss. 

Prior research has delved into exploring these capabilities and how they emerge as the scale and of training compute, type of data, and model size vary~\citep{wei2022emergent, chung2022scaling, tay2022transcending}. Untangling the factors, however, remains challenging due to the data complexity and the variety of tasks examined.
Driven by the curiosity to understand the factors that elicit these capabilities in next-token predictors, we set out to pinpoint the key contributors that accelerate the emergence of such abilities. These contributors may include the format and scale of data, model scale, the presence of pre-training, and the manner of prompting. 

To provide a more precise examination of these factors, our study is conducted in a controlled setting: we focus on teaching arithmetic to small transformer models, such as NanoGPT and GPT-2, when trained from random init. Starting with a model of $10.6$ million parameters and scaling up to $124$ million parameters, we use the standard autoregressive next-token prediction loss. Our objective is to understand how these models can efficiently learn basic arithmetic operations like addition, subtraction, multiplication, square root, and sine, thereby providing us with a clearer lens through which to view the elicitation of emergent abilities. Below, we summarize our findings.

 \textbf{Data format and sampling matters. } We first observe that teaching a model addition (or any other operation) using standard addition samples, \ie `$\mathsf{A_3A_2A_1+B_3B_1B_1= C_3C_2C_1}$', is suboptimal, as it requires the model to evaluate the most significant digit $\mathsf{C_3}$ of the result first, which depends globally on all the digits of the two summands. By training on samples with reversed results, \ie `$\mathsf{A_3A_2A_1+B_3B_1B_1= C_1C_2C_3}$', we enable the model to learn a simpler function, significantly improving sample complexity. Additionally, balanced sampling of different ``variations'' of addition, based on the number of carries and digits involved, further enhances learning. Even in this simple setting, we observe relatively sharp phase transitions from 0 to 100\% accuracy as a function of the size of the training data. Although this may seem surprising, we observe that learning an addition map on $n$ digits from random samples is equivalent to completing a low-rank matrix. This connection allows us to offer a reasonable explanation for such phase transitions.

\textbf{Chain-of-thought data during training.} 
Building on these findings, we then explore the potential benefits of chain-of-thought (CoT) data during training. This format includes step-by-step operations and intermediate results, allowing the model to learn the individual components of complex tasks. This format is directly borrowed from related literature, e.g., ~\citep{ling2017program, nye2021show, wei2022chain, zhou2022least, anil2022exploring, zhou2022teaching}. We found that CoT-type training data significantly improved learning in terms of both sample complexity and accuracy in agreement with CoT fine-tuning literature~\citep{nye2021show, chung2022scaling}, though our observation holds {\it even in the absence of language pretraining.}
 We conjecture that this is because breaking down the required compositional function to be learned into individual components allows the model to learn a higher-dimensional but easier-to-learn function map. In Figure~\ref{fig:input_formatting}, we provide examples of the four data formatting methods explored in our work.

\textbf{Training on text and arithmetic mixtures and the role of few-shot prompting.} We also explore the interplay between arithmetic and text data during training, as LLMs are trained on massive amounts of data scraped from the internet~\citep{bubeck2023sparks, openwebtext}, where it is impractical to carefully separate different types of data. We observe how the model's perplexity and accuracy vary with the ratio of text to arithmetic data. We find that learning all arithmetic operations discussed earlier (from addition to square root) can improve the individual performance of each task, and that going from zero-shot to 1-shot prompting (showing one arithmetic example) yields a large accuracy improvement, but there is no significant improvement in accuracy by showing more examples.

\paragraph{The role of pre-training and model scale.}
 We also investigate the role of pretraining by fine-tuning models like GPT-2 and GPT-3 (\texttt{davinci}) and observe that while the zero-shot performance on arithmetic operations is poor, the prior ``skills'' acquired during pretraining facilitate reasonable performance on some basic arithmetic  tasks even with a small number of finetuning samples.
 However, finetuning with non-standard formatting, such as reverse formatting, can interfere with the model's performance when pretrained on standard-formatted operations, leading to decreased accuracy.
 Finally, we conduct studies on how performance in arithmetic changes with scale, and although we find that scale does indeed aid in learning arithmetic operations, it is not a necessary trait.

\textbf{Compositional and length generalization.} One might question if our trained models truly grasp arithmetic. Our findings present a nuanced answer. We find length generalization beyond trained digit lengths difficult. For instance, if a model is trained on all $n$-digit lengths, excluding a specific length, it struggles to compensate and accurately calculate this missing digit length. Consequently, the models achieve high accuracy within trained digit lengths but struggle significantly beyond this range. This suggests that the models learn arithmetic not as a flexible algorithm, but more as a mapping function constrained to trained digit lengths. While this surpasses mere memorization, it falls short of comprehensive arithmetic ``understanding''. 

\textbf{Novelty over prior work.} Our approach heavily builds upon prior work that uses instructive data to enhance model performance, and we do not claim novelty in the style of training data employed. What sets our work apart is the primary focus on randomly initialized models and extensive ablation studies on various sampling/data formatting and model scale settings to isolate the factors that contribute to the fast emergence of arithmetic capabilities. Furthermore, our work offers a few simple but perhaps insightful theoretical justifications of some of the phenomena we observe.

\begin{figure}[t] 
\centering
\includegraphics[width=\textwidth]{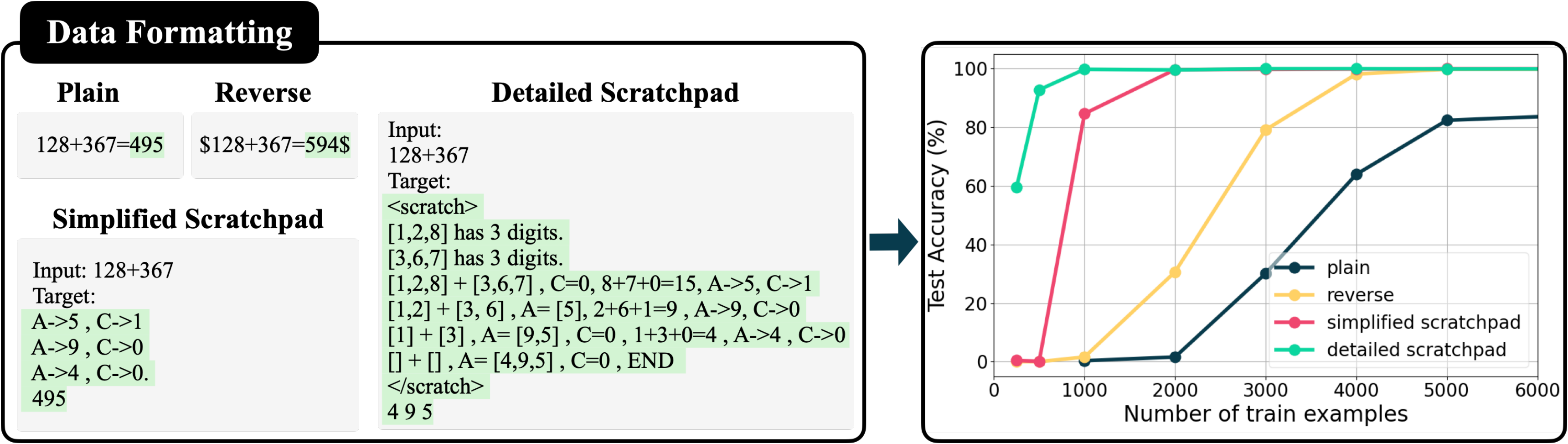}
    \caption{%
    The four data formatting methods investigated in this work: (i) Plain: standard addition formatting (Section~\ref{sec:plain_reverse}), (ii) Reverse: reversing the output (Section~\ref{sec:plain_reverse}), (iii) Simplified Scratchpad: recording the digit-wise sum and carry-ons (Section~\ref{sec:scratchpad-and-ar}), and (iv) Detailed Scratchpad: providing detailed intermediate steps of addition (Section~\ref{sec:scratchpad-and-ar}). We train small transformer models from scratch using data transformed with these various formatting methods for addition. The results (shown on the right) highlight the crucial role of data formatting in performance and sample efficiency. Plain never reaches $100\%$ accuracy and the sample complexity for the remaining methods to learn addition perfectly steadily reduces as we increase the level of detail in the data format.
    } 
\label{fig:input_formatting}
\vspace{-4mm}
\end{figure}

%% file: _RelatedWorks.tex
\section{Related Works}\label{sec:related_work}

\textbf{Instructional data/chain-of-thought. } The idea of using detailed reasoning training data predates Transformers~\citep{vaswani2017attention}. \citet{ling2017program, cobbe2021training, nye2021show} use natural language to generate reasoning steps while \citet{roy2016solving, reed2015neural, chen2017towards, cai2017making} show that symbolic reasoning may suffice. \citet{nogueira2021investigating} note that large number of samples with small digits is important for arithmetic tasks~\citep{yuan2023well}. \citet{razeghi2022impact} observe a correlation between the frequency of numbers in the dataset and the performance involving them whereas we find that transformers can learn to add numbers that were not seen during training. Chain-of-thought~\citep{wei2022chain} refers to the model's improved performance when prompted to produce rationale. \citet{zhou2022teaching} show that this can be achieved by providing sufficiently informative exemplars as a few-shot prompt~\citep{brown2020language}. \citet{zhou2022least} showed that \emph{least-to-most} prompting can help GPT-3 solve problems that can be decomposed into simpler sub-problems. Least-to-most prompting consists of first decomposing a complex problem into easier subproblems, and then sequentially solving these subproblems. We extend this notion to simple addition and show that asking the model to output the least significant bit first has a similar effect. \citet{kojima2022large} shows that very often even just prompting the model with \emph{``let's think step by step''} is sufficient to achieve competitive zero-shot accuracy on several benchmark datasets.

\textbf{Arithmetic using Transformer models. } Our work focuses on decoder-only models since they are well-suited for text generation and are widely used in LLMs~\citep{brown2020language, touvron2023llama, mosaicml2023Introducing}. However, encoder-decoder models have also been extensively studied in the literature in the context of learning arithmetic~\citep{kim2021have, wang2021exploring}. \citet{qian2022limitations, lightman2023let, uesato2022solving} explore techniques to improve the arithmetic abilities of pretrained LLMs. \citet{wallace2019nlp} on the other hand, focus on the impact of the learned embeddings. Most results that show Turing-completeness or the universal approximation typically rely on encoder models~\citep{yun2019transformers, perez2021attention, wei2022statistically, giannou2023looped}. \citet{ontanon2021making} study the problem of compositional generalization extensively on benchmark datasets such as SCAN~\citep{lake2018generalization, drozdov2022compositional} and conclude that design changes like relative position encoding~\citep{shaw2018self} can improve performance significantly. \citet{charton2022my, charton2021linear} show that Transformers can learn linear algebra operations with carefully chosen encodings. \citet{hanna2023does} use mechanistic interpretability techniques to explain the limited numerical reasoning capabilities of GPT-2.

\textbf{Beyond Transformers. } While we focus our attention on GPT-like models, there is a rich literature studying other sequence-to-sequence models such as recurrent neural networks (RNNs)~\citep{bowman2013can, bowman2014recursive, zaremba2014learning}. \citet{zaremba2014learningtoexec} show that RNNs can learn how to execute simple programs with for-loops provided they are trained with curriculum learning. \citet{sutskever2014sequence} show that LSTMs show improved performance on text-based tasks such as translation when the source sentences are reversed, which is closely related to what we observe in addition. \citet{kaiser2015neural} propose Neural GPUs which outperform prior RNNs on binary arithmetic tasks and even show length generalization \ie they can perform arithmetic on inputs of lengths that were unseen during training. This is yet to be seen even in modern pre-trained models~\citep{bubeck2023sparks} and therefore it is interesting to see if we can leverage some of these techniques and apply them to existing modern architectures. \citet{dehghani2018universal} propose Universal Transformers (UTs) which introduce a recurrent transition function to apply recurrence over revisions of the vector representation at each position as opposed to the different positions in the input. They show that on the tasks from \citet{zaremba2014learningtoexec}, UTs outperform traditional Transformers and RNNs.

\textbf{Data-centric AI. } More recently, there has been increasing interest in \emph{Data-Centric AI} which emphasizes techniques to improve datasets in order to ensure better performance~\citep{motamedi2021data, hajij2021data}. \citet{gadre2023datacomp} propose a new benchmark where the training code is fixed and the only way to improve performance is to construct new training sets. Several works have also tried to see if the model's reasoning ability can be leveraged to generate explanations and leverage it to solve complicated reasoning tasks~\citep{rajani2019explain, talmor2020leap, zelikman2022star, huang2022large}.

%% file: _Addition.tex
\section{Preliminaries and Experimental Setup}\label{sec:exp}

In this section, we provide a detailed description of our experimental setup, including the model architecture and an overview of the different data formatting and sampling techniques that we employ and evaluate.

\paragraph{Model and Data.}
To examine the individual factors at play, we use NanoGPT~\citep{nanogpt}, a lightweight implementation of the GPT family of models, chosen primarily for its feasibility to train from random initialization under numerous settings. NanoGPT features a decoder-only transformer architecture with six self-attention layers, six heads, and an embedding dimension of $384$, resulting in approximately $10.6$ million parameters. Unless stated otherwise, we use character-level tokenization and absolute position encoding. We train the NanoGPT model from random initialization, which we refer to as training from `scratch', using the conventional next-token prediction objective.

To understand the effect of scale, we extend our experiments to GPT-2 and GPT-3 in Section~\ref{sec:finetuning_pretrained_models}.
We investigate teaching arithmetic from scratch as well as fine-tuning using a pretrained GPT-2. However, for GPT-3, we exclusively use supervised fine-tuning on a pretrained model. Refer to Appendix~\ref{appdx:model} for a more detailed description.

For arithmetic tasks like addition, subtraction, and multiplication, we define the training dataset for a binary operator $f(\cdot)$ as $\gD_\text{train}=\{(a_i, b_i), y_i\}_{i=1}^{N}$ where $y_i = f(a_i, b_i)$. For unary operations such as the sine and square root functions, the training dataset is formulated as $\gD_\text{train}=\{a_i, y_i\}_{i=1}^{N}$, where $y_i = f(a_i)$. The test dataset $\gD_\text{test}$ is constructed by randomly sampling pairs of operands not included in $\gD_\text{train}$.
Throughout training and inference, we apply different \emph{data formatting} techniques on each data sample from the training dataset, creating the final sequence that serves as the model's input.

\paragraph{Data Formatting.}

In the following sections, we will delve into the detailed intuition, and results of the four data formatting approaches that we have deployed in our arithmetic experiments. For this section, we provide a high-level summary of these approaches, each progressively incorporating additional information to form a more comprehensive format. The scratchpad formats are largely adopted from the literature of chain-of-thought (CoT) training~\citep{nye2021show, zhou2022teaching}. 
See Figure~\ref{fig:data_format_examples} and Appendix~\ref{sec:prompt_examples} for detailed examples.

\begin{AIbox}{\bf{\large Different data formatting methods for addition}}
\vspace{5mm}
{\normalsize Four input formatting methods used for the addition task:\\
\textbf{(i) Plain}: standard formatting of addition\\
\textbf{(ii) Reverse}: flips the order of the output and encapsulates each data sample with the`\$' symbol at the start and end.\\
\textbf{(iii) Simplified Scratchpad}: provides carry and digit-sum information for each step of addition, from the LSB to the MSB\footnotemark.\\
\textbf{(iv) Detailed Scratchpad}: provides explicit details of intermediate steps of addition.\\}

\begin{minipage}[t]{0.39\linewidth}
\centering
\textbf{Plain}
\begin{lstlisting}[language=markdown]
128+367=495
\end{lstlisting}

\textbf{Reverse}
\begin{lstlisting}[language=markdown]
$128+367=594$
\end{lstlisting}

\textbf{Simplified Scratchpad}
\vspace{0.6mm}
\begin{lstlisting}[language=markdown]
Input: 128+367 
Target: 
A->5 , C->1 
A->9 , C->0 
A->4 , C->0. 
495
\end{lstlisting}

\end{minipage}
\begin{minipage}[t]{0.61\linewidth}
\centering
\textbf{Detailed Scratchpad}
\begin{lstlisting}[language=markdown, showlines=true,]
 Input: 
 128+367 
 Target: 
 <scratch> 
 [1,2,8] has 3 digits. 
 [3,6,7] has 3 digits. 
 [1,2,8] + [3,6,7] , C=0, 8+7+0=15, A->5, C->1 
 [1,2] + [3, 6] , A= [5], 2+6+1=9 , A->9, C->0 
 [1] + [3] , A= [9,5] , C=0 , 1+3+0=4 , A->4 , C->0 
 [] + [] , A= [4,9,5] , C=0 , END  
 </scratch> 
 4 9 5

\end{lstlisting}
\end{minipage}
\end{AIbox}
\noindent\begin{minipage}{\textwidth}
\captionsetup{type=figure}

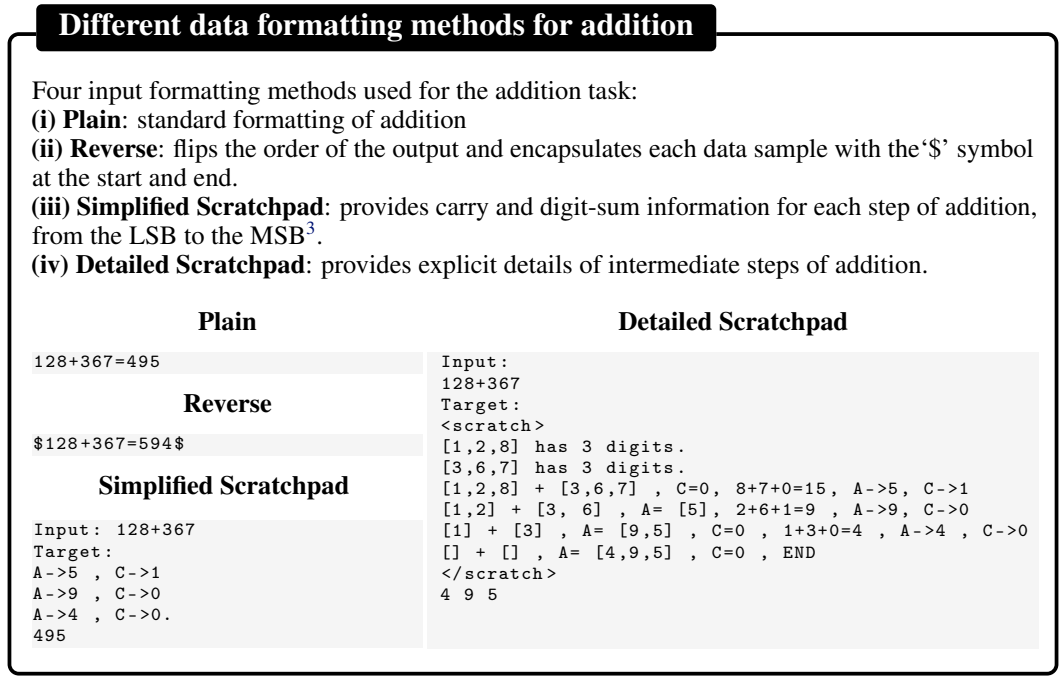
\captionof{figure}{The four input formatting methods used for the addition task. We progressively increase the amount of detail with each format.}\label{fig:data_format_examples}
\end{minipage}
\footnotetext{We deviate from the strict definition of ``most significant bit'' (MSB) and ``least significant bit'' (LSB), typically associated with binary numbers, and reinterpret them for the purpose of this paper as the most significant ``digit'' and least significant ``digit'', respectively.}

Note that we wrap each data sample in the reverse format with the `\$' symbol at the beginning and end as a delimiter. We originally observed improved performance in both the plain and reverse formats when the operands and outputs were zero-padded to a fixed length (\eg $3$ and $4$ digits, respectively, for $3$-digit addition). But later realized that a single symbol can effectively replace zero-padding. While we maintain the original plain format without padding as a baseline -- emphasizing the necessity for improved data formatting for efficient emergence -- we incorporate the `\$'-encapsulation in our modified reverse format.  For further details, refer to Appendix~\ref{sec:appdx_zeropad_delimiter}.

In Section~\ref{sec:plain_reverse}, we explore the limitations of the conventional plain-format data and demonstrate how a simple reversal of the output order can lead to substantial performance improvements and enhanced sample efficiency. We introduce two Lemmas to support and explain these findings. Additionally, in Section~\ref{sec:scratchpad-and-ar}, we present results on the simplified and detailed scratchpad formats, highlighting significant enhancements in sample efficiency for learning addition. We also emphasize the importance of carefully designing the intermediate steps in the detailed scratchpad method.

\paragraph{Structured Data Sampling.}\label{sec:data_sampling}

While data formatting plays a crucial role, we also discover that selecting the appropriate samples for inclusion in the training data is also essential. When sampling operands for $n$-digit addition uniformly at random between $1$ to $10^n-1$, the dataset becomes highly skewed in terms of the number of samples with (i) operands containing a specific number of digits and (ii) operands resulting in a certain number of \emph{carry-on}\footnote{In this paper, we adopt the definition that a carry-on operation involves transferring information from one digit position to another position of higher significance. Therefore, we refer to the ``borrow'' operation in subtraction as a carry operation.} operations. For instance, in the case of 3-digit addition, random sampling results in a meager $0.01\%$ probability of selecting a $1$-digit number. Additionally, $1$ or $2$ carry-on operations are more likely to occur than $0$ or $3$. To address this \emph{imbalance}, we employ a structured sampling approach. Specifically, we aim to \textbf{(i) balance digits} by assigning higher weights to lower-digit numbers during the sampling process as in~\citet{nogueira2021investigating} and \textbf{(ii) balance carry-ons} by ensuring an equal distribution of examples with $0, 1, \ldots, n$ carry-on operations.

\begin{wrapfigure}{r}{0.46\textwidth} %
\vspace{-4mm}
\centering
\includegraphics[width=0.45\textwidth]{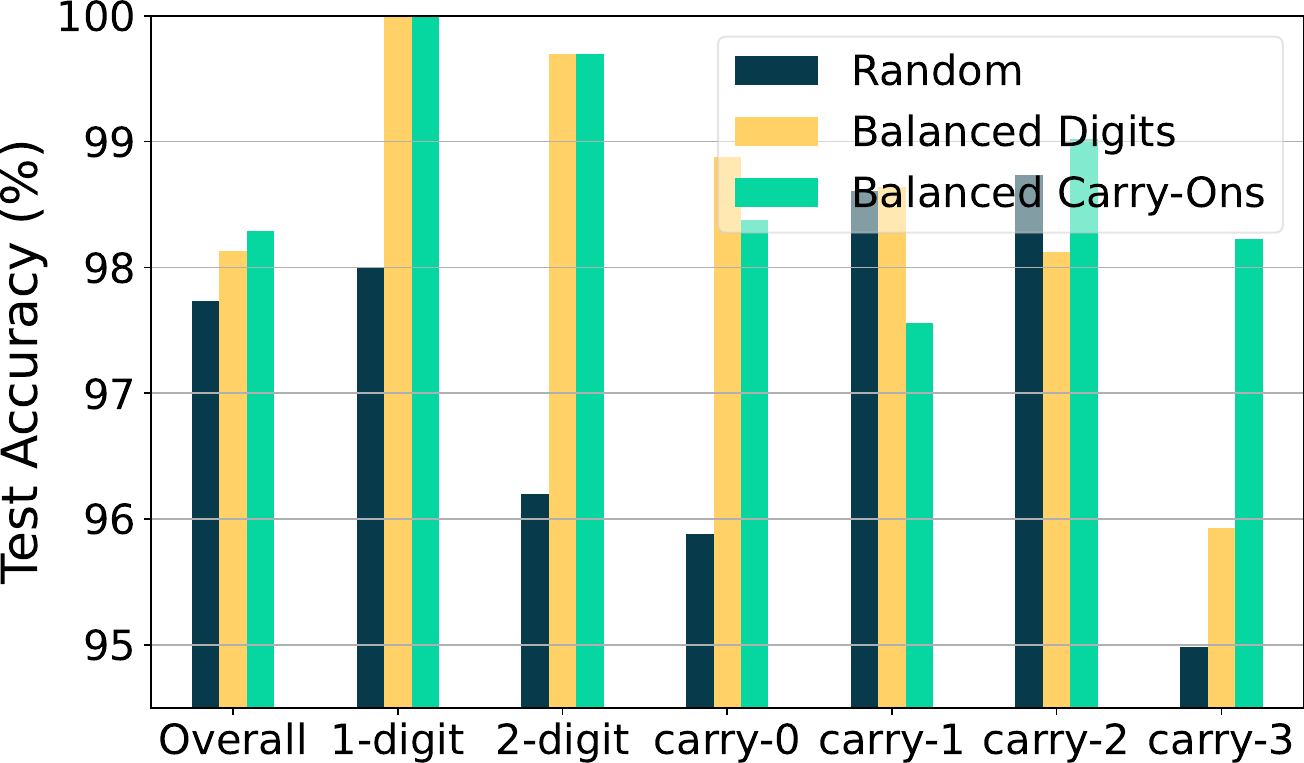}
    \caption{Performance of 3-digit addition on various data sampling methods used: \textbf{(i) Random}: uniform sampling of operands; \textbf{(ii) Balanced digits}: assigning higher sampling weights to operations involving $1$ and $2$-digit numbers; \textbf{(iii) Balanced carry}: balancing the dataset to contain an equal number of carry-on operations. Experiments on addition with zero-padding both operands and output to have 3 and 4 digits respectively.}
\label{fig:data_sampling}
\vspace{-4mm}
\end{wrapfigure}

When sampling $10,000$ examples of 3-digit addition, we include all possible $100$ $1$-digit additions, $900$ $2$-digit samples and $9000$ $3$-digit samples. Note that while the number of samples increase, the fraction of all possible $k-$digit additions that we sample for $k=2, 3$ decreases due to the inherent skew.
The split was chosen heuristically to ensure we saw a ``reasonable'' fraction of all possible $k-$digit samples for all $k$.
Similarly, we ensure that the number of samples with $0, 1, 2$, or $3$ carry-on operations are all approximately $2500$.

Figure~\ref{fig:data_sampling} reveals the importance of ``balancing''. We observe improvements in accuracy across the board while using \emph{Balanced} data when compared to random sampling. Further, random sampling performs relatively poorly even for the simple task of $2-$digit addition. We conjecture that this is due to the fact that the model has not seen enough of these examples. For the remaining experiments, we set the default dataset for addition to be one that has both balanced digits and carry-ons.

\section{Learning Addition in Small Models}\label{sec:plain_reverse}

\begin{wrapfigure}{r}{0.42\textwidth}
\vspace{-12mm}
\centering
\includegraphics[width=0.38\textwidth]{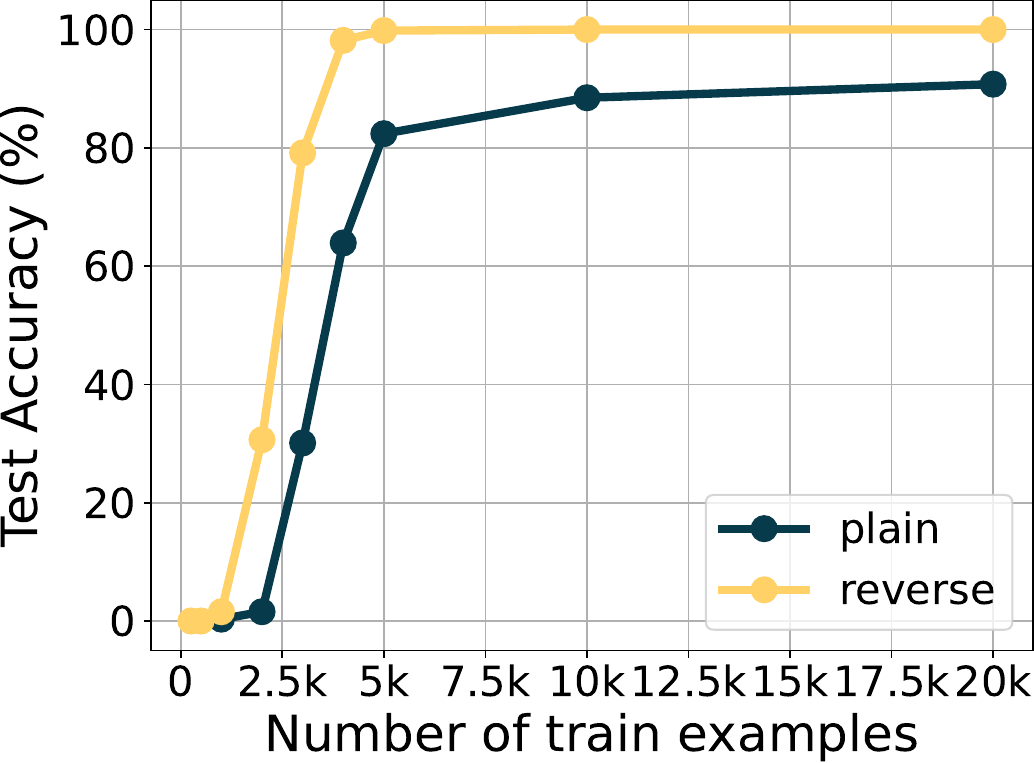}
    \caption{Comparison of NanoGPT model performance on the addition task, trained on plain and reverse formatted data. The conventional plain format exhibits suboptimal performance, even with a larger number of addition examples, whereas a distinct \emph{phase transition} is observed for the reverse format around 2500 train samples where it learns addition perfectly.} 
\label{fig:plain_vs_reverse}
\vspace{-12mm}
\end{wrapfigure}

We start by examining one of the most basic arithmetic tasks: \emph{addition}. Initially, we concentrate on the $3$-digit addition, where the two operands have at most $3$ digits ($999$). Later, in Section~\ref{sec:higher_digits}, we demonstrate that our findings can be applied to larger digits. We assess whether NanoGPT can learn addition from training data of various sizes. As we will soon discover, learning addition may not be as straightforward as one might anticipate.

\subsection{Training on Conventional Data}
We begin by training NanoGPT on conventional addition data in the form of `$\mathsf{A_3A_2A_1+B_3B_1B_1= C_3C_2C_1}$', which we denote as the \emph{plain} data format. However, as shown in Figure~\ref{fig:plain_vs_reverse}, this leads to fairly poor performance. We believe that this is because the next-token prediction objective is not optimized for generating the most significant digit (MSB) first.

The following lemma clarifies the necessity to access all operand digits in order to output the MSB first:

\begin{restatable}{lemma}{lemmaOne}\label{lemma:left-to-right}
Let $A$ and $B$ be two $n$-digit numbers, and let $C=A+B$. Suppose an algorithm $\gA$ outputs the digits of $C$ in decreasing order of significance, then $\gA$ must have access to all digits of $A$ and $B$ starting from the first digit that it outputs.
\end{restatable}

The lemma suggests that to train the model for addition and to output the most significant digit first, it is necessary for the model to learn a \emph{global} algorithm. Unlike the standard algorithm for addition which consists of computing digit-wise sums and carry-ons, approximating a global algorithm would necessitate learning a more complicated function than necessary. The increased complexity results in decreased accuracy, as observed throughout our experiments. \citet{liu2023exposing} refer to this phenomenon as \emph{attention glitches}.

\subsection{Reversing the Output}\label{sec:addition-plain-vs-reverse}
This leads us to ask, \emph{``is it possible to guide the model to learn a simpler algorithm for addition?''} We propose an intuitive approach to improve performance by training the model to generate the least significant digit (LSB) first, following the way humans typically perform addition. By starting with the LSB and progressing towards the most significant digit (MSB) from right to left, the model can learn a simpler algorithm that relies on just three inputs: the corresponding digits from the operands and the carry-on information (0 or 1) carried from the LSB to the MSB. This approach offers an advantage over the plain format, where generating the MSB first would necessitate the model to learn a more complex function involving all digits in the two operands.

We propose that using this \emph{reverse} format (`$\mathsf{\$A_3A_2A_1+B_3B_1B_1= C_1C_2C_3\$}$')
is more suitable for next-word prediction models. The rationale behind this is that when generating the sum by starting with the least significant digit (LSB), the model only needs to learn a local function of three inputs per digit -- the two relevant digits of the operands and the carry-on from the previous digit. This local operation simplifies the learning function. The following lemma substantiates this idea:

\begin{restatable}{lemma}{lemmaTwo}\label{lemma:right-to-left}
There exists an algorithm that computes $C=A+B$ for two n-digit numbers $A$ and $B$ and outputs its digits in increasing order of significance such that, at each position $i$, the algorithm only requires access to the $i^{\text{th}}$ digits of $A$ and $B$, as well as the carry-on from the previous position.
\end{restatable}

Lemma~\ref{lemma:right-to-left} directly follows from the \emph{standard} algorithm for addition, which performs the sum and carry-on operations digit by digit.  The implications of these two lemmas are evident in our experiments when comparing training NanoGPT on \emph{plain} and \emph{reverse} samples. As shown in Figure~\ref{fig:plain_vs_reverse}, the accuracy of \emph{plain} addition plateaus at slightly over $85\%$ even with $10,000$ samples. In contrast, simply training the model on reversed output significantly enhances the performance. Additionally, we observe that the reverse format requires considerably fewer training data to achieve good performance, further reinforcing that the reverse format's associated function has less complexity than the plain format.
What is particularly remarkable is the occurrence of a notable \emph{phase transition} between $1000$ and $4000$ samples for reverse. At this point, the model rapidly transitions from being unable to perform addition to being capable of perfectly adding two $3$-digit numbers. This leads to an important question: 

\begin{quotation}
\begin{center}
    \hspace{-5mm} \noindent \it Why does addition rapidly emerge as the number of training examples increases?
\end{center}
\end{quotation}

\section{Connection to Low-Rank Matrix Completion}\label{sec:lrmc}
Although the rapid phase transition observed in the previous section may initially seem surprising, closer examination reveals a fascinating equivalence: learning an addition map on $n$ digits from random samples can be considered as completing a rank-2 matrix. This equivalence offers a compelling explanation for the phenomenon we observed.
In this section, we delve into the intricate details of this connection and elucidate how learning the addition map can be formulated as low-rank matrix completion (LRMC). Establishing this connection provides meaningful insights into the observed phenomenon. Further, our investigation goes beyond that and highlights the enhanced capabilities of Transformer models. We demonstrate that Transformers possess expanded capabilities that surpass what traditional LRMC algorithms can do.

\subsection{Addition Tables are Rank-2 Matrices}
\begin{figure}[ht]
\vspace{-4mm}
\centering
\subfloat[Matrix Completion of Addition Matrix]{\includegraphics[width=0.45\textwidth]{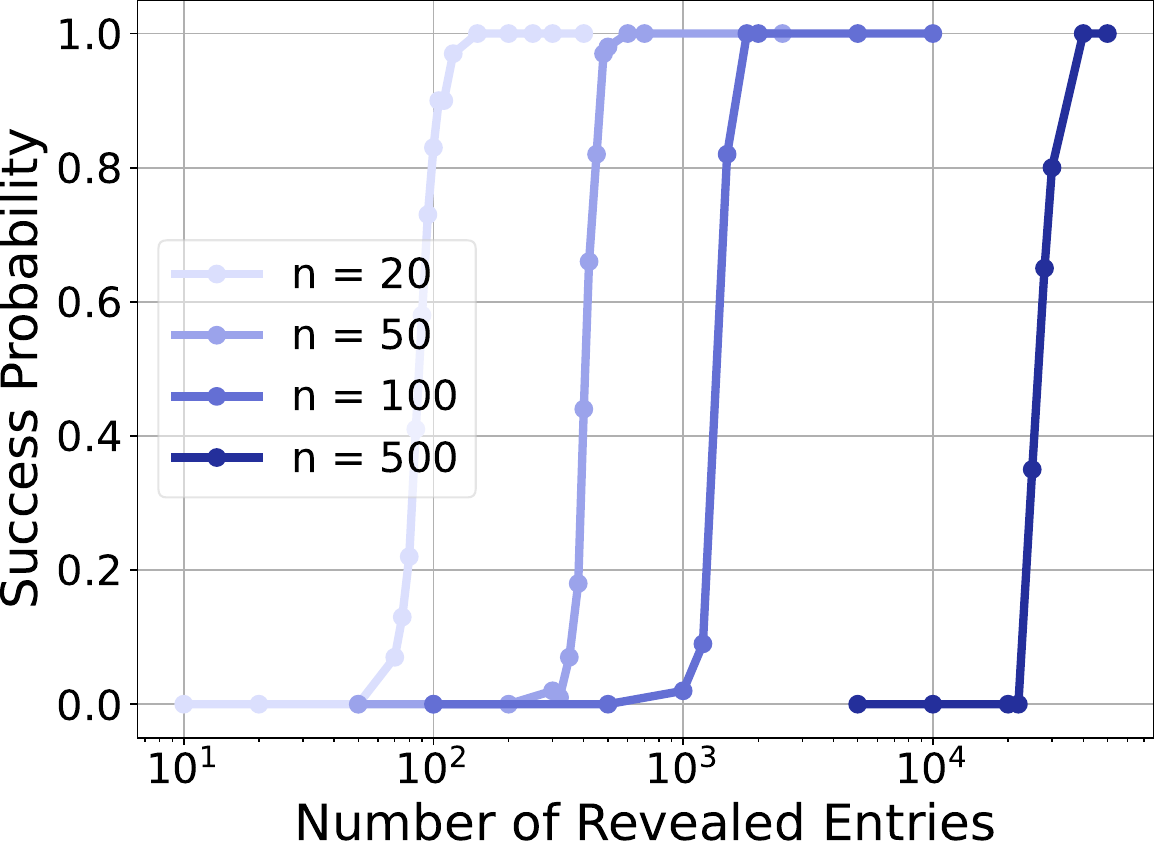}\label{fig:matrix_completion}}
\hspace{0.5cm}
\subfloat[Comparing LRMC and NanoGPT]{\includegraphics[width=0.45\textwidth]{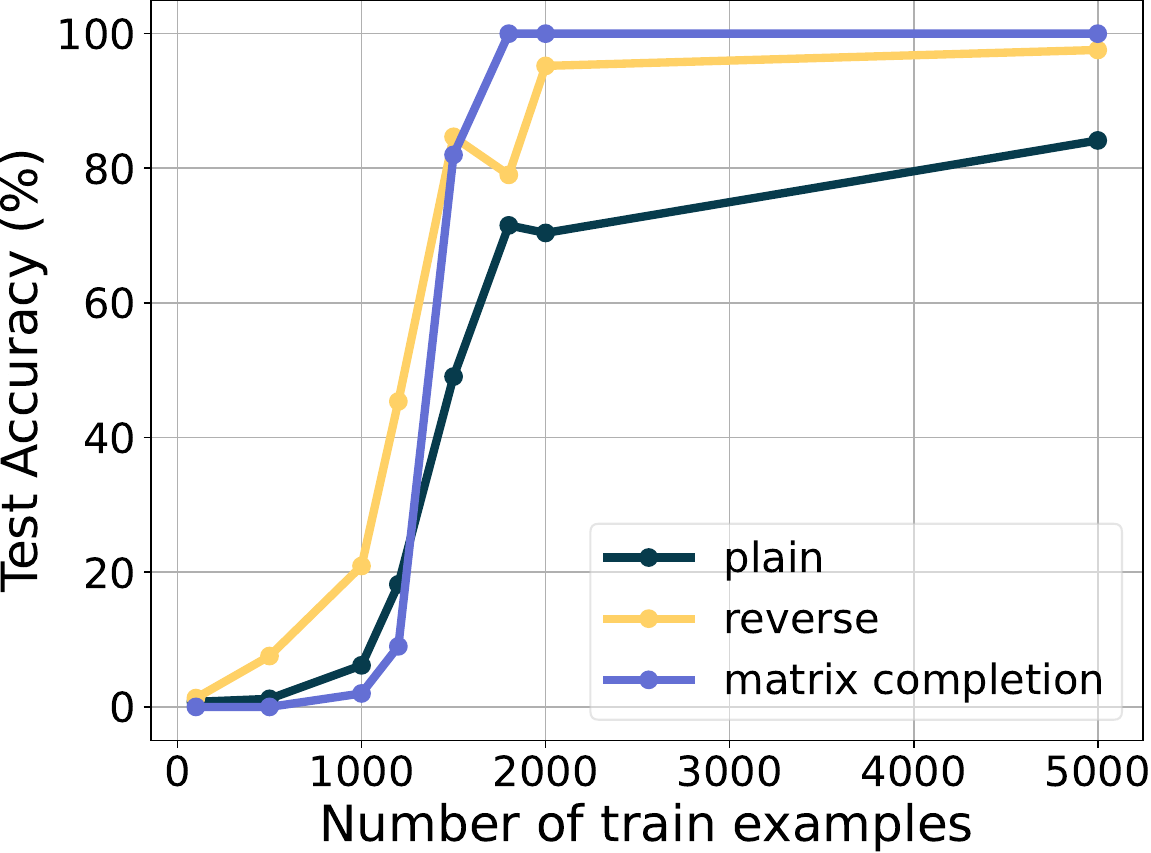}\label{fig:nanogpt_vs_mc}}
    \caption{(a) We run Algorithm~\ref{alg:iterative_2x2_mc}~\citep{kiraly2015algebraic}, a simple iterative algorithm for 2-rank matrix completion for the addition matrix $(n=20,50,100,500)$ and report the success probability over multiple random trials while varying the number of revealed entries. As anticipated, a sharp phase transition occurs when approximately $\gO(n)$ entries are revealed. (b) We compare the performance of a NanoGPT model trained on a dataset containing $n=100$ samples (\ie $2$-digit addition) to that of the corresponding LRMC problem using the same sample set. Notably, the phase transition at around $1500$ samples, where both NanoGPT and Algorithm~\ref{alg:iterative_2x2_mc} begin learning addition almost flawlessly, is remarkably similar.} 
\end{figure}

Learning addition from samples can be formulated as a rank-2 Matrix Completion (MC) problem involving an $n \times n$ matrix $\mM$, where the $(i,j)$-th entry $M_{i,j}$ represents the output of the addition `$i+j$'. Such $\mM$ can be decomposed into the sum of two rank-one matrices, $\mN \vone^T$ + $\vone \mN^T$, where $\mN$ is a column vector with entries $\{1,\dots n\}$, and $\vone$ is a vector of $n$ ones. Thus, learning addition from samples can be viewed as solving the MC problem in which only the entries corresponding to those samples are revealed. When the underlying matrix is noiseless and of rank-2, \citet{kiraly2015algebraic} demonstrates that a simple iterative algorithm (Algorithm~\ref{alg:iterative_2x2_mc} in Appendix~\ref{sec:appdx_lrmc}) is optimal.
As depicted in Figure~\ref{fig:matrix_completion}, a sharp \emph{phase transition} occurs at $\gO(n)$. This aligns with Theorem-2 from \citet{recht2011simpler} which states that the exact convex relaxation to the MC problem has a unique solution as long as $\widetilde{\gO}(n)$ samples are observed.

The sharp phase transition observed in LRMC bears a resemblance to what we notice in NanoGPT. To further investigate this phenomenon, we focus on $2$-digit addition ($n=100$) as shown in Figure~\ref{fig:matrix_completion}. We evaluate the performance of learning addition through NanoGPT in comparison to LRMC by constructing a training dataset consisting of the matrix's revealed entries in either plain or reverse format. It is important to note that the training dataset is no longer ``\emph{balanced}'', as the revealed entries are randomly and uniformly sampled for the LRMC experiments. The comparison between NanoGPT and LRMC results is presented in Figure~\ref{fig:nanogpt_vs_mc}. Remarkably, both NanoGPT and LRMC exhibit a similar phase transition at approximately $1500$ samples, where they both start to learn addition almost perfectly. This observation regarding LRMC offers an explanation for the rapid emergence of addition in NanoGPT.

\subsection{NanoGPT Generalizes better than Matrix Completion solutions}
We noted above that there are some striking similarities between the addition map learned by NanoGPT and LRMC. However, we now delve deeper and find that this map exhibits capabilities beyond LRMC. A well-known limitation of LRMC is its inability to generalize when entire rows or columns are empty. Therefore, we intentionally hide certain numbers in the training dataset or specific digit positions, and examine whether our model can still learn addition.

\paragraph{Generalizing to unseen numbers.} In order to further investigate the connection with LRMC, we exclude an increasing fraction of the numbers from the training data and evaluate the model's ability to learn addition. As shown in Table~\ref{tab:hiding_numbers}, the answer to this question is a resounding \emph{Yes!} The model achieves almost perfect accuracy even when excluding half of all possible $3-$digit numbers. More precisely, we randomly choose $100/200/500$ numbers and exclude them from the training data. We then evaluate the trained models two metrics: (i) \textbf{Overall accuracy:} which measures the accuracy over a random set of $10,000$ examples and (ii) \textbf{Exclusion accuracy:} which measures the accuracy only over the excluded set. 
Remarkably, excluding numbers from the training data sometimes leads to improved performance. We conjecture that this may be due to the effect of regularization, similar to random masking or cropping images in vision tasks. Note that these results indicate that the model is not simply performing LRMC. In the LRMC setting, even a single missing number corresponds to an empty row or column, which cannot be recovered. Hence, the ability of the NanoGPT model to generalize to missing numbers signifies its distinct capabilities beyond LRMC.

\begin{table}[ht!]
\vspace{-2mm}
  \caption{Impact of excluding numbers on addition task: NanoGPT models trained with $100/200/500$ excluded operands show no significant drop in accuracy and in some cases, the performance even improves. Note that models trained with \emph{reverse} data remain consistently at $100\%$ accuracy.}
  \label{tab:hiding_numbers}
  \vspace{1mm}

  \centering
  \small
    \setlength{\tabcolsep}{4pt} %
    \renewcommand{\arraystretch}{1.0}
		 {
\begin{tabular}{ccccccrcr}
\toprule
\multicolumn{1}{l}{} & \multicolumn{2}{c}{No Exclusion}                      & \multicolumn{2}{c}{\shortstack{Excluding\\100 numbers}} & \multicolumn{2}{c}{\shortstack{Excluding\\200 numbers}}                    & \multicolumn{2}{c}{\shortstack{Excluding\\500 numbers}}                    \\ \midrule
                     & Plain          & Rev                           & Plain             & Rev            & Plain                       & \multicolumn{1}{c}{Rev} & Plain                       & \multicolumn{1}{c}{Rev} \\ \midrule
\shortstack{Overall Accuracy}     & 87.18\%        & 99.97\%                           & 87.94\%           & 100.00\%           & \multicolumn{1}{r}{87.24\%} & 99.99\%                     & \multicolumn{1}{r}{88.15\%} & 99.99\%                     \\
\shortstack{Exclusion Accuracy}      & - & - & 92.55\%           & 100.00\%           & \multicolumn{1}{r}{92.15\%} & 99.95\%                     & \multicolumn{1}{r}{90.85\%} & 100\%                       \\ 
\bottomrule
\end{tabular}
}
\vspace{-2mm}
\end{table}

\paragraph{Generalizing to unseen digits.} Building upon the model's robustness to excluded numbers, we further investigate its ability to handle excluded digits.
Intuitively, this should be even more challenging since excluding a digit means the model cannot learn directly how to operate in that position. Instead, it would have to generalize and infer that digits act similarly across all positions. We construct datasets with the number $\mathbf{5}$ excluded in $1$st (LSB), $2$nd, and $3$rd (MSB) positions, and train separate models on each of these datasets. We compare the resulting models by evaluating \textbf{overall accuracy} on a test set of $10,000$ randomly sampled numbers, as well as their accuracy specifically on samples with $\mathbf{5}$ in each position which we call \textbf{exclusion accuracy}.

The results presented in Table~\ref{tab:hiding_5} indicate that the model is not as robust to excluding digits compared to excluding numbers. However, it still achieves more than $66\%$ accuracy on every test and maintains an overall accuracy above $85\%$. Moreover, it appears that excluding a number in the least significant position yields the worst performance. This can be attributed to the fact that learning addition in this position is transferable to other positions since it is unaffected by carry-on operations. Failing to learn addition in this position, however, will have a detrimental impact on other positions as well.

\begin{table}[ht!]
\vspace{-2mm}
\center
  \caption{Impact of excluding digits on addition task: We investigate whether GPT-based models can infer addition on an excluded digit in a specific position from training data on other positions. We compare NanoGPT models trained with and without an excluded digit and find that excluding digits is harder to learn but not entirely impossible, with the worst performance observed when excluding the least significant digit.}
  \label{tab:hiding_5}
  \vspace{1mm}
\centering
\small
\setlength{\tabcolsep}{4pt} %
    \renewcommand{\arraystretch}{1.0}
\begin{tabular}{c|c|cccc}
\toprule
\multicolumn{1}{l|}{Excluded position}                              & \multicolumn{1}{l|}{\shortstack{Input\\format}} & Overall Acc & \shortstack{``5'' in the \\1st (LSB) digit} & \shortstack{``5'' in the \\2nd digit} & \shortstack{``5'' in the \\3rd (MSB) digit} \\ \midrule
\multirow{2}{*}{No exclusion}                         & Plain                 & 87.18\%     & 87.50\%                    & 88.65\%              & 91.80\%                    \\
                                                   & Reverse               & 99.97\%     & 99.90\%                    & 99.95\%              & 100\%                      \\
\multirow{2}{*}{\shortstack{1st (LSB) digit}} & Plain                 & 85.05\%     & \textbf{76.70\%}           & 85.80\%              & 88.35\%                    \\
                                                   & Reverse               & 93.31\%     & \textbf{66\%}              & 94.80\%              & 94.45\%                    \\
\multirow{2}{*}{\shortstack{2nd digit}}       & Plain                 & 85.44\%     & 84.55\%                    & \textbf{78.50\%}     & 90.15\%                    \\
                                                   & Reverse               & 98.85\%     & 98.85\%                    & \textbf{94.20\%}     & 99.50\%                    \\
\multirow{2}{*}{\shortstack{3rd (MSB) digit}} & Plain                 & 85.70\%     & 85.35\%                    & 87.35\%              & \textbf{83.45\%}           \\
                                                   & Reverse               & 97.18\%     & 97.25\%                    & 97.35\%              & \textbf{85.45\%}           \\ 
\bottomrule
\end{tabular}
\vspace{-2mm}
\end{table}

\paragraph{The distinct learning mechanism of NanoGPT.} The phase transition of LRMC offers significant insights into NanoGPT's learning process. Nevertheless, further experiments clearly demonstrate that NanoGPT's mechanism for learning addition is fundamentally different from LRMC. It can successfully learn addition even when numbers or digits are intentionally excluded from the training data, thereby exhibiting generalization capabilities that far exceed that of typical LRMC algorithms.

\section{The power of Chain-of-Thought: Incorporating Intermediate Steps in Training Data}\label{sec:scratchpad-and-ar}

So far, we observed that utilizing the straightforward method of reversing the output can result in remarkable performance, exceeding that of LRMC in learning addition. Nonetheless, it may be possible to expedite the emergence of addition by further enhancing the data format. As addition is a multi-step process, we further explore the idea of incorporating additional information about each step. We adopt a Chain-of-Thought (CoT) style approach, where we guide the model to learn addition step-by-step. In the subsequent sections, we assess the effect of incorporating these intermediate steps on the performance of small models. We demonstrate that this results in a substantial improvement in sample complexity of learning addition and carefully analyze how the level of detail offered for each step impacts the model's performance.

\subsection{Training on Chain-of-Thought Data}

In the following experiments, we evaluate if training on scratchpad data further improves the learning of addition. As described briefly in Section~\ref{sec:exp}, scratchpad data incorporates step-by-step instructions in varying amounts of detail into the samples. This approach aims to help the model learn addition as a compositional function. We explore two levels of detail in the provided instruction steps: \textbf{Simplified Scratchpad} format offers minimal information -- the sum and carry information for each digit/step. \textbf{Detailed Scratchpad} provides comprehensive information on how to execute each step in the addition process in natural language. By comparing the performance of the model trained with these different levels of detail, we can analyze its impact on the model's ability to learn addition effectively.

\begin{wrapfigure}{r}{0.4\textwidth}
\vspace{-4mm}
\centering
\includegraphics[width=0.38\textwidth]{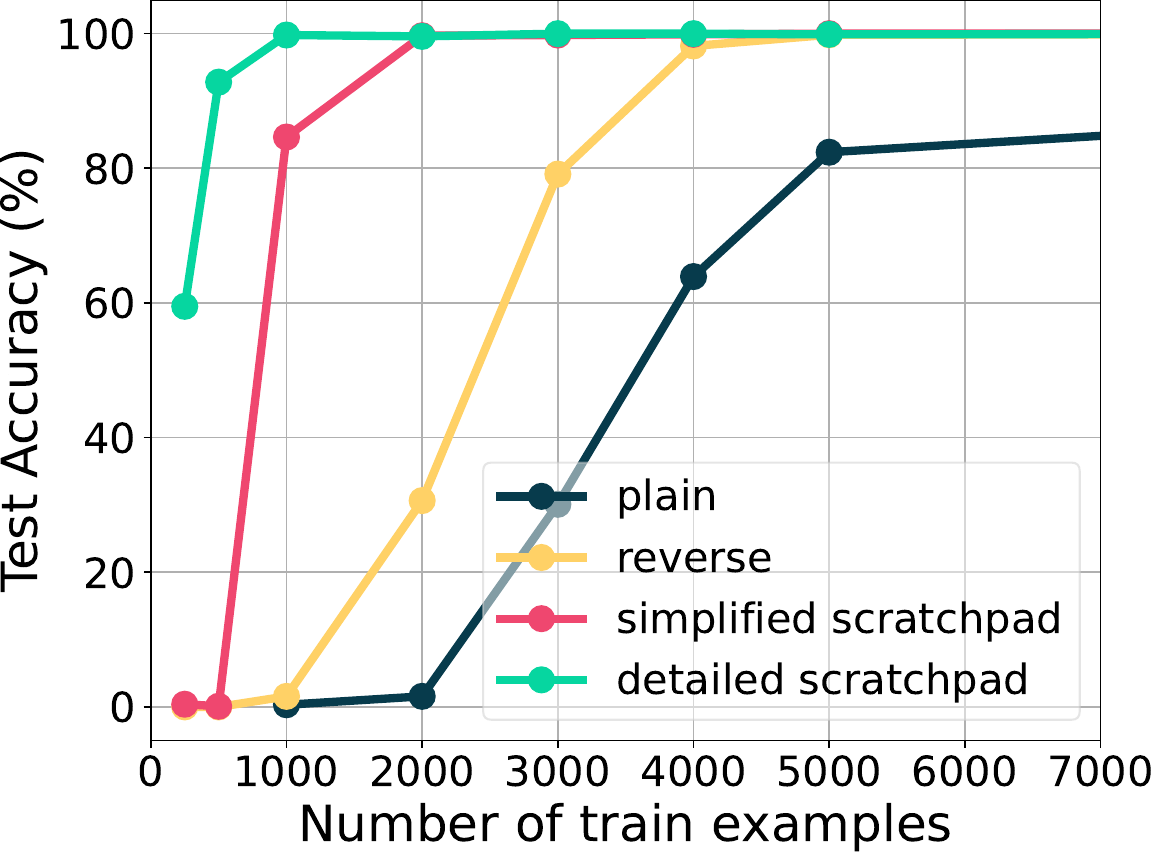}
    \caption{Comparison of sample efficiency: evaluating performance on training datasets with different numbers of addition samples. While all modified methods (reverse, simplified scratchpad, and detailed scratchpad) achieve 100\% test accuracy, they exhibit varying requirements in terms of the number of addition examples in the training dataset to reach optimal performance.} 
\label{fig:nanogpt_sample_efficiency}
\vspace{-8mm}
\end{wrapfigure}

The results presented in Figure~\ref{fig:nanogpt_sample_efficiency} demonstrate the effectiveness of different data formats for training addition.
The model trained on \textbf{Simplified Scratchpad} data achieves $100\%$ accuracy with only $2000$ samples, whereas the \textbf{Reverse} format requires more than double the number of samples. Furthermore, the \textbf{Detailed Scratchpad} format, which provides even more detailed information, achieves perfect addition with just $1000$ samples.
This indicates that incorporating more information enables the model to learn addition more efficiently, requiring fewer examples. 
We conjecture that this is because breaking down the required compositional function to be learned into individual components allows the model to learn a higher-dimensional but easier-to-learn function map.
We note that while CoT-style training enhances sample efficiency, it may not necessarily be the most ``token-efficient'' approach. We delve into this aspect in more detail in Section~\ref{sec:cost_analysis}.
In summary, incorporating scratchpad data and decomposing the addition task into steps offer a promising strategy to improve the performance and efficiency of small models in learning addition from scratch.

\subsection{The Importance of Intermediate Step Design: Subtraction}\label{sec:appendix_subtraction_detailed_scratchpad}

In this section, we underscore the significance of \emph{meticulously designing the intermediate steps in a Chain-of-Thought manner}. Specifically, we focus on the \textbf{subtraction} task and conduct experiments to compare two different versions of the detailed scratchpad for this operation (see examples in Figure~\ref{fig:subtraction_ar}). These trials shed light on the importance of decomposing the subtraction task into \emph{simpler intermediate steps}. Unlike addition, subtraction behaves differently depending on whether the first operand ($a$) is greater than the second operand ($b$) or vice versa.

\begin{AIbox}{\bf{\large Detailed scratchpad formatting for different arithmetic tasks}}
\vspace{5mm}
{\footnotesize Examples of two variations of detailed scratchpad formatting for subtraction, considering the scenario where the first operand $a$ is greater than the second operand $b$, and vice versa. In Version 1, a result processing step is included in the final stage to handle negative outputs. In Version 2, the operands are compared at the beginning, and if $b$ is larger, their order is reversed.\\}

\textbf{Prompt (Case 1. $a-b\geq0$) :}\newline
{\tt \footnotesize Input:\\
367-128\\
Target:}

\begin{minipage}[t]{0.5\linewidth}
\centering
\textbf{Version 1. }
\begin{lstlisting}[language=markdown]
...
<scratch>
[3,6,7] has 3 digits.
[1,2,8] has 3 digits.
[3,6,7] - [1,2,8] , A=[] , C=0 , 7-8-0+10=9 , A->9 , C->-1
[3,6] - [1,2] , A=[9] , C=-1 , 6-2-1=3 , A->3 , C->0
[3] - [1] , A=[3,9] , C=0 , 3-1-0=2 , A->2 , C->0
[] - [] , A=[2,3,9]
200+39=239 , END # result processing
</scratch>
2 3 9
\end{lstlisting}
\end{minipage}
\begin{minipage}[t]{0.5\linewidth}
\centering
\textbf{Version 2.}
\begin{lstlisting}[language=markdown]
...
<scratch>
[3,6,7] has 3 digits.
[1,2,8] has 3 digits.
367>=128 # comparison of two operands
[3,6,7] - [1,2,8] , A=[] , C=0 , 7-8-0+10=9 , A->9 , C->-1
[3,6] - [1,2] , A=[9] , C=-1 , 6-2-1=3 , A->3 , C->0
[3] - [1] , A=[3,9] , C=0 , 3-1-0=2 , A->2 , C->0
[] - [] , A=[2,3,9] , END
</scratch>
2 3 9
\end{lstlisting}
\end{minipage}

\textbf{Prompt (Case 2. $a-b < 0$) :}\newline
{\tt \footnotesize Input:\\
128-367\\
Target:}

\begin{minipage}[t]{0.5\linewidth}
\centering
\textbf{Version 1.}
\begin{lstlisting}[language=markdown]
...
<scratch>
[1,2,8] has 3 digits.
[3,6,7] has 3 digits.
[1,2,8] - [3,6,7] , A=[] , C=0 , 8-7-0=1 , A->1 , C->0
[1,2] - [3,6] , A=[1] , C=0 , 2-6-0+10=6 , A->6 , C->-1
[1] - [3] , A=[6,1] , C=-1 , 1-3-1=-3 , A->-3 , C->-1
[] - [] , A=[-3,6,1]
-300+61=-239 , END # result processing
</scratch>
-2 3 9
\end{lstlisting}
\end{minipage}
\begin{minipage}[t]{0.5\linewidth}
\centering
\textbf{Version 2.}
\begin{lstlisting}[language=markdown]
...
<scratch>
[1,2,8] has 3 digits.
[3,6,7] has 3 digits.
128<367 : 128-367=-(367-128) # comparison
[3,6,7] - [1,2,8] , A=[] , C=0 , 7-8-0+10=9 , A->9 , C->-1
[3,6] - [1,2] , A=[9] , C=-1 , 6-2-1=3 , A->3 , C->0
[3] - [1] , A=[3,9] , C=0 , 3-1-0=2 , A->2 , C->0
[] - [] , A=[2,3,9] , END
</scratch>
-2 3 9
\end{lstlisting}
\end{minipage}
\end{AIbox}
\noindent\begin{minipage}{\textwidth}
\captionsetup{type=figure}
\captionof{figure}{Two versions of detailed scratchpad formatting for subtraction.}\label{fig:subtraction_ar}
\end{minipage}

The first strategy (Version 1 in Figure~\ref{fig:subtraction_ar}) involves performing digit-wise subtraction starting from the least significant bit (LSB) and considering borrows when necessary. However, this strategy produces incorrect results when the first operand is smaller than the second operand. In such cases, we subtract the number in the most significant bit (MSB) position multiplied by 10 to the power of (number of digits in the output - 1) from the remaining digits in the output. An example illustrating this approach is shown in Version 1, Case 2.
Alternatively, we can adopt a more familiar strategy. If the first operand is smaller than the second, we swap the operands and compute the negation of the subtraction of the swapped operands: $a - b = -(b - a)$ (referred to as Version 2).

The results in Figure~\ref{fig:subtraction_ar_result} indicate that Version 2, which involves comparing two operands, performs considerably worse than Version 1. In Version 1, each intermediate step only requires the simpler 1-digit subtraction, along with addition in the final result processing step. Upon analyzing the failure cases of Version 2, we observe that the majority of errors stem from incorrectly identifying which of the two operands is larger, while the intermediate steps are handled correctly. This finding underscores the significance of breaking down arithmetic operations into \emph{simpler intermediate steps}. Unless otherwise specified, we use Version $1$ in all detailed scratchpad experiments. 

\begin{figure}[ht]
\vspace{-6mm}
\centering
\includegraphics[width=0.38\textwidth]{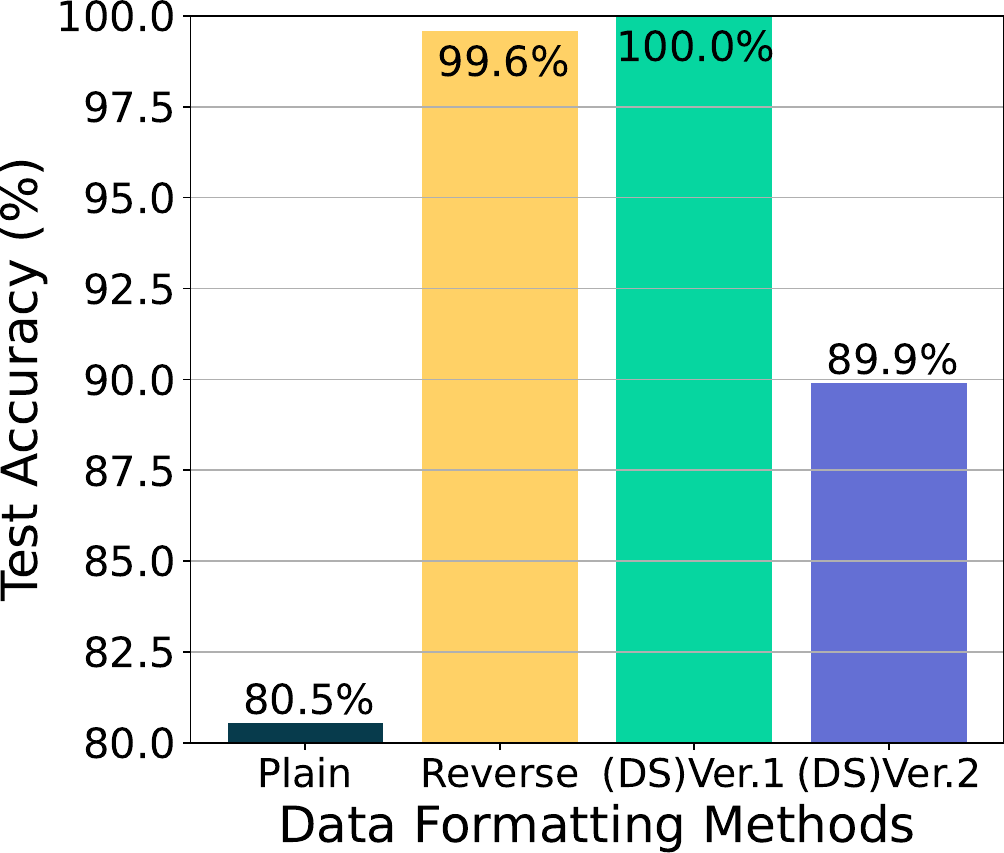}
    \caption{Comparison of performance among various data formatting approaches (plain, reverse, and two versions of detailed scratchpad (DS)) for the subtraction task. The experiments were conducted on a NanoGPT model trained on a dataset of 10,000 examples. Version 2, which incorporates operand comparison, exhibits significantly lower performance compared to Version 1. This observation highlights the substantial impact of the construction of intermediate steps on the model's performance.} 
\label{fig:subtraction_ar_result}
\end{figure}
\vspace{-6mm}

\subsection{The Effect of Noisy Inputs on Accuracy}\label{sec:ablation_noisy}

\paragraph{Noisy intermediate steps in the scratchpad data. }

\begin{figure}[ht]
\vspace{-4mm}
\centering
\subfloat[{Test accuracy on Addition}]{\includegraphics[width=0.45\textwidth]{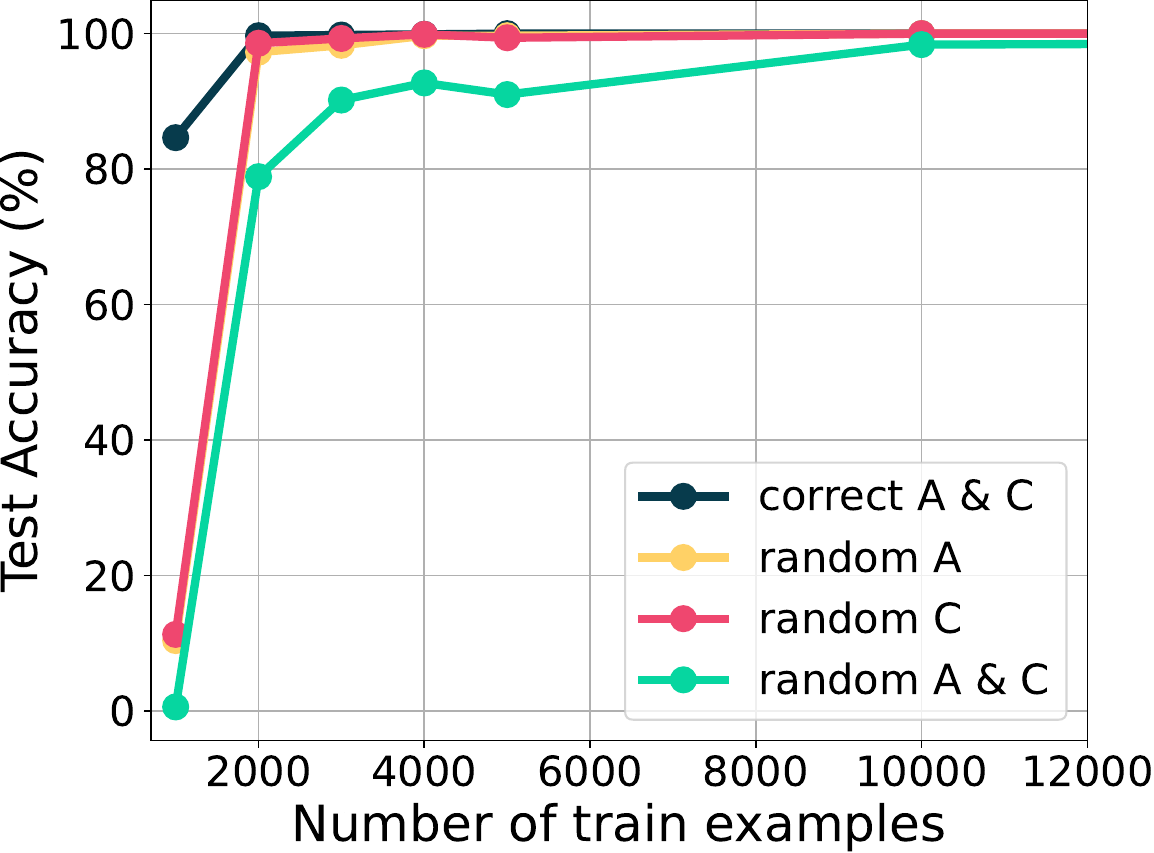}}
\hspace{6mm}
\subfloat[{Test accuracy on Subtraction}]{\includegraphics[width=0.45\textwidth]{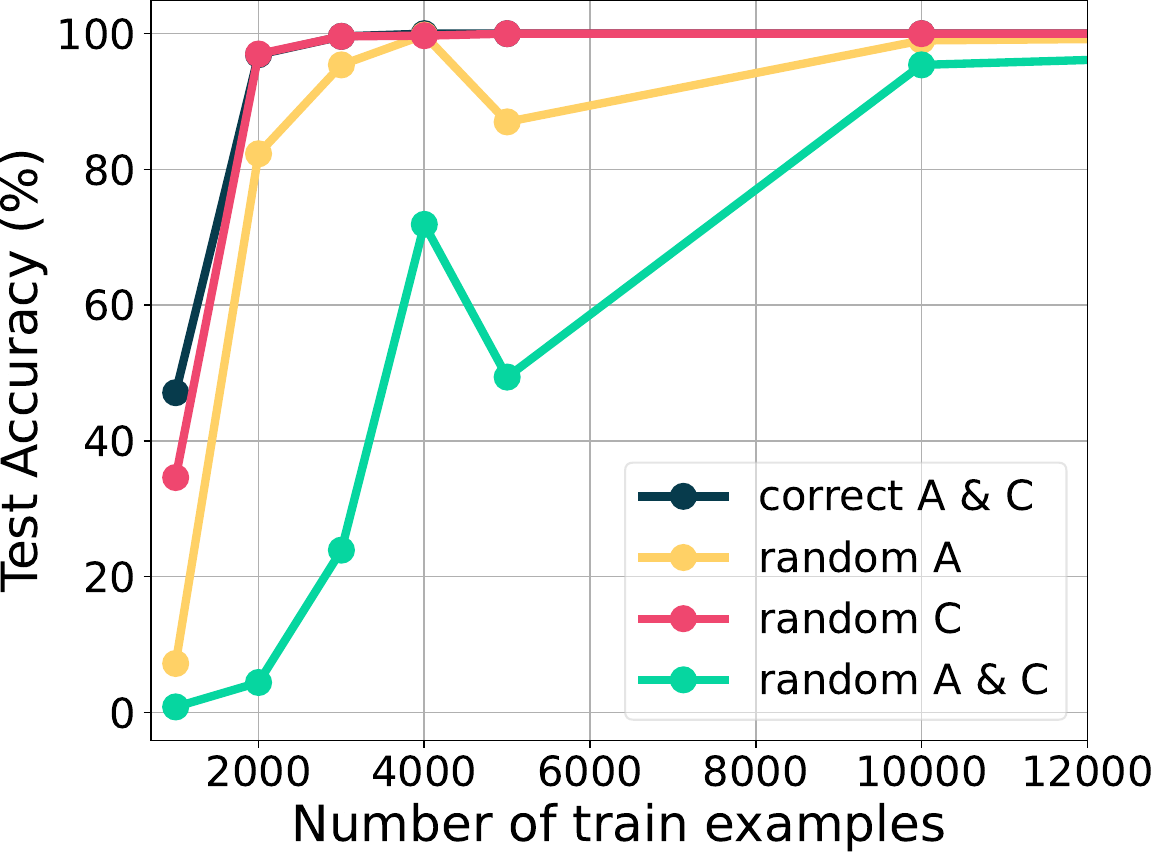}}
    \caption{Comparison of training with simplified scratchpad formatting using correct A and C information with formatting using random A/C and their effect on sample efficiency and accuracy. Results show that noisy labels degrade sample efficiency, but with sufficient training data, the model eventually reaches full accuracy.} 
\label{fig:simplified_scratchpad}
\end{figure}

We further investigate the significance of providing accurate intermediate steps in the scratchpad during the training process. While this was inspired by the findings of \citet{min2022rethinking}, it is inherently different. \citet{min2022rethinking} show that using random labels in ICL demonstrations caused minimal degradation when compared to the gold labels. However, those models were trained on gold labels and then evaluated on multiple downstream tasks. In our setting, the model is trained and evaluated on a single arithmetic task. Further, the final result(or label) is left untouched as the correct answer to the arithmetic operation. We only replace the intermediate steps. The goal of this study is to verify whether the model actually learns to reason using the given intermediate steps or merely uses the scratchpad to improve its expressivity. We compare the performance of training with our simplified scratchpad formatting, which includes accurate $A$ (digit sum) and $C$ (carry) information, with formatting that includes random $A$, random $C$, or random $A$ and $C$ for each intermediate step, as depicted in Figure~\ref{fig:input_formatting}.

The results in Figure~\ref{fig:simplified_scratchpad}, demonstrate that the inclusion of noisy labels can impede sample efficiency. However, with enough samples, the model ultimately achieves full accuracy. This suggests that while the model is capable of leveraging the information contained in the intermediate steps, it can also gradually learn how to perform addition while disregarding the presence of noisy intermediate steps.

\paragraph{Model robustness to noise in the auto-regressive output. }

In this analysis, we explore the robustness of models trained on plain or reverse formatted data (without noise) when exposed to noise during an auto-regressive generation process. In particular, we aim to unravel how much the learned mapping of the $i$-th output relies on the operands and preceding tokens in the addition result, given that transformer models generate tokens sequentially in an autoregressive manner, making them prone to error propagation.

For this experiment, we focus on $3$-digit addition. We train models on either plain or reverse format data and evaluate the accuracy of next-token predictions when the output sequence contains noise. Specifically, in the plain format setting, we expect a well-performing model to generate the correct output tokens $\mathsf{O_3}$, $\mathsf{O_2}$, $\mathsf{O_1}$ sequentially, where $\mathsf{O_3=C_3}$, $\mathsf{O_2=C_2}$, $\mathsf{O_1=C_1}$, and $\mathsf{C_3C_2C_1}$ represents the correct answer. We consider two types of perturbation: (i) \textbf{random} perturbation, where we modify the first two output tokens $\mathsf{O_3O_2}$ to random numbers different from $\mathsf{C_3C_2}$, and (ii) \textbf{precise} perturbation, where we perturb only the second output token $\mathsf{O_2}$ by 1. The second case is particularly relevant since a common error case is where the model misses a digit by $1$.
We provide the model with an expression of the form ``$\mathsf{A_3A_2A_1+B_3B_1B_1= O_3O_2}$'', where $\mathsf{O_3O_2}$ can be either (i) a random incorrect number, \ie $\mathsf{O_3O_2 \neq C_3C_2}$, or (ii) $\mathsf{O_2}=\mathsf{C_2}\pm 1 \mod 10$, and observe the next token generated by the model.
A corresponding process is deployed for the reverse format, introducing a noisy sequence to models trained on reverse format data.

To evaluate the performance, we define two accuracy criteria for $\mathsf{O_1}$: \textbf{exact accuracy}, reckoning $\mathsf{O_1}$ as accurate only when $\mathsf{O_1=C_1}$, and \textbf{relaxed accuracy}, considering $\mathsf{O_1}$ correct if it deviates from the original output $\mathsf{C_1}$ by at most 1. In other words, $\mathsf{C_1=O_1}$, $\mathsf{C_1 = O_1+1 \mod 10}$ or $\mathsf{C_1 = O_1 - 1 \mod 10}$.

\begin{table}[th!]
\vspace{-4mm}
\center
  \caption{
  Prediction accuracy for the third digit output under different types of noise in the preceding output tokens. \textbf{Random} perturbation, applies random flips whereas \textbf{precise} perturbation shifts the preceding output tokens by $1$. \textbf{Relaxed accuracy}, allows for a $\pm1$ deviation from the true output whereas \textbf{Exact accuracy} is strict. Reverse consistently outputs a number that is at most $1$ different from the true output, even in the presence of noise. The plain format has high exact accuracy in the presence of precise perturbation, as the noise in the output token has a lower impact on predicting the next token, which is of lower significance. However, with completely random noise, the plain format shows poor performance, suggesting a strong dependence on all digits. (See Lemma~\ref{lemma:left-to-right} and ~\ref{lemma:right-to-left}).
  }
  \label{tab:noise_output}
  \vspace{2mm}
\centering
\small
\begin{tabular}{l|cc|cc}
\toprule
Perturbation Type  & \multicolumn{2}{c|}{Random} & \multicolumn{2}{c}{Precise} \\ \cline{2-3}\cline{4-5}
            & Plain       & Reverse      & Plain        & Reverse      \\ \midrule
Exact Acc   & 49.88\%     & \textbf{81.26\% }     & \textbf{99.85\%}      & 90.47\%      \\
Relaxed Acc & 61.55\%     & \textbf{100\%}        & \textbf{100\%}        & \textbf{100\% }       \\
\bottomrule
\end{tabular}
\vspace{-2mm}
\end{table}

The results presented in Table~\ref{tab:noise_output} reveal intriguing findings. We observe that the reverse format consistently outputs a result that \emph{deviates by no more than 1 from the true answer}, regardless of whether the preceding outputs $\mathsf{O_3O_2}$ are subjected to random or precise perturbation. This consistency can be explained by Lemma~\ref{lemma:right-to-left}, indicating that the reverse format only requires learning a straightforward function of digit-wise addition for each corresponding position, along with the carry-on (0 or 1). Therefore, even with noise in the preceding tokens, the model accurately performs digit-wise addition, albeit with occasional carry-on prediction errors.
With an exact accuracy of 81.26\% even in the presence of random perturbation, the reverse format demonstrates the model's ability to rely less on the preceding output tokens, indicating a robust learned output mapping.

On the contrary, models using the plain format have to decipher a more intricate function drawing from all digits within the sequence, as described by Lemma~\ref{lemma:left-to-right}. Given that in addition, carry operations transition from right to left (\ie least to most significant digit), the introduction of precise perturbation on preceding output tokens, which possess higher significance, has a minor impact on the output (which has less significance). As a result, models trained using the plain format attain an exact accuracy rate of 99.85\% and a relaxed accuracy of 100\% for cases involving precise perturbation. Interestingly, under purely random perturbation, the plain format struggles, leading to a reduced relaxed accuracy of 61.55\% and exact accuracy of 49.88\%. This suggests that the output mapping learned by the plain format is not merely a function of the two operands but rather enmeshed in complex dependencies on preceding output tokens.

\section{Extending to Longer Digit Addition}\label{sec:higher_digits}

In this section, we extend our experiments beyond 3-digit addition and explore longer-digit settings, ranging up to $10$ digits. Our aim is to investigate whether our previous findings regarding the sample efficiency of reverse and scratchpad formats hold true for larger numbers of digits.

We begin by observing that the phase transition behavior observed in previous sections also applies to longer-digit addition.  Furthermore, we discover that the advantages of using reverse and scratchpad formats become even more pronounced as the number of digits increases.
Next, we examine the number of training samples required to learn $k+1$ digit addition when fine-tuning a pretrained model trained on $k$ digit addition. We find that while the number of samples needed to further learn $k+1$ digit addition remains relatively consistent for reverse and scratchpad formats, the plain format requires an increasing number of samples.

\paragraph{Experimental setup and data generation.}
To explore the performance of the model in higher-digit addition scenarios, we extend the experimental setup described in Section~\ref{sec:exp}. We adopt a balanced sampling approach for training data with $D$ digits, ensuring an equal number $d$ of all combinations of digits for both operands as follows:

We begin by sampling all $100$-digit additions. For the remaining number of digits, ranging from $2$ to $D$, we generate addition examples of the form ``$\textsf{A + B = C}$''. The two operands, $\textsf{A}$ and $\textsf{B}$, are randomly sampled $d=\lfloor (N-100) / (D(D+1)/2 -1 ) \rfloor$ times for every $D$, where $N$ is the total number of training examples. Operand $\textsc{A}$ is sampled between $[10^{k_1-1}, 10^{k_1} -1]$ and operand $\textsc{B}$ is sampled between $[10^{k_2-1}, 10^{k_2} -1]$, for all $1 \leq k_1 \leq k_2 \leq D$, excluding the case where $k_1=k_2=1$. After sampling the two operands, we randomly interchange them to cover cases where $\textsc{A}$ has fewer digits than $\textsc{B}$ and vice versa.

\subsection{Training from Random Initialization}\label{sec:higher_digits_scratch}

\begin{figure}[ht]
\vspace{-4mm}
\centering
\subfloat[{5-digit Addition}]{\includegraphics[width=0.32\textwidth]{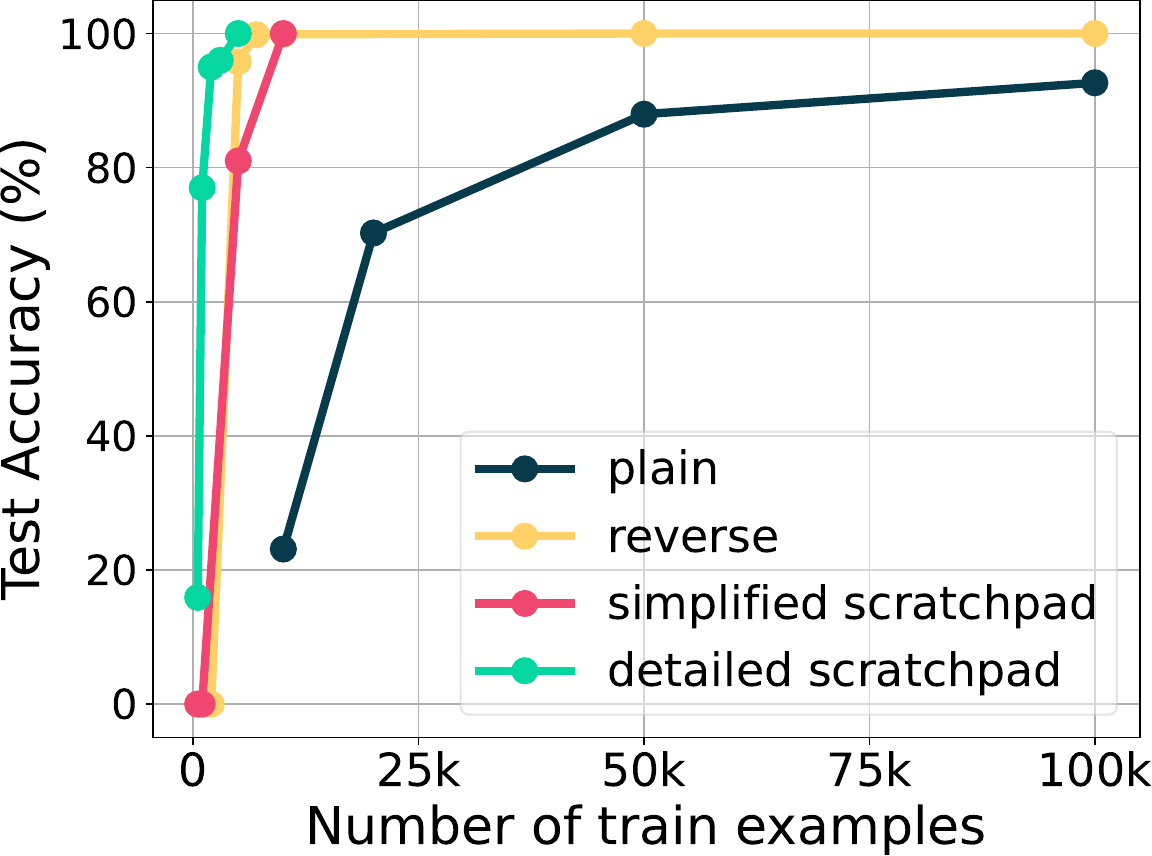}}
\hspace{1mm}
\subfloat[{7-digit Addition}]{\includegraphics[width=0.32\textwidth]
{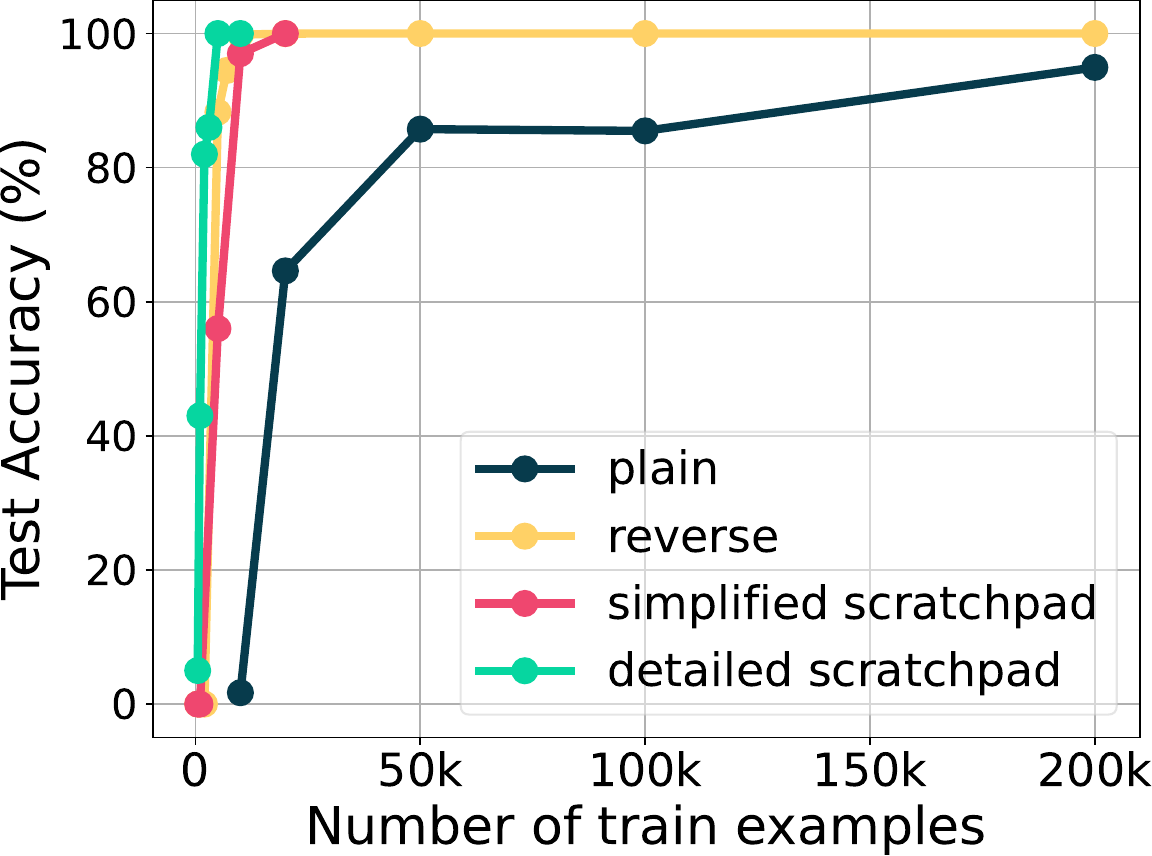}}
\hspace{1mm}
\subfloat[{10-digit Addition}]{\includegraphics[width=0.32\textwidth]{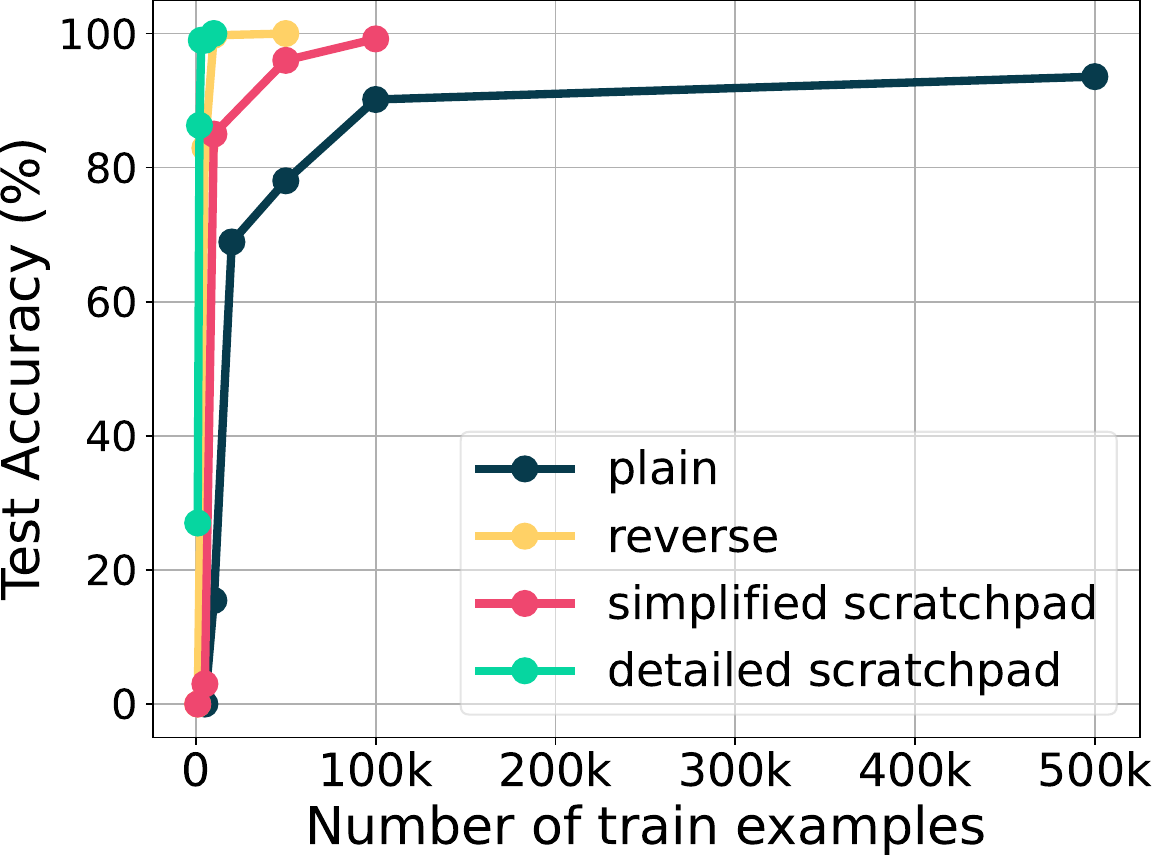}}
    \caption{Comparison of sample efficiency for 5, 7 and 10-digit additions: performance of models trained with varying numbers of addition samples on each data format. The plain format data requires an increasing number of training examples for higher digits, while the number of samples required for other methods remains relatively consistent.}
\label{fig:multidigit_scratch}
\vspace{-2mm}
\end{figure}

We repeat the experiment from Section~\ref{sec:exp} on nanoGPT with longer digits. The results shown in Figure~\ref{fig:multidigit_scratch} demonstrate a similar behavior to the findings observed in Figure~\ref{fig:nanogpt_sample_efficiency} for 3-digit addition. This indicates that our previous observations generalize to longer sequence lengths.
Notably, the performance gap between the modified formats (reverse, simplified scratchpad, and detailed scratchpad) and the plain format becomes even more significant in the context of higher digits. While the plain format requires an increasing number of training examples to learn higher-digit additions, the reverse or scratchpad formats exhibit a more consistent requirement in terms of the number of training examples.

This prompts us to explore the differences between each format in a fine-tuning setting. Specifically, we ask whether a model trained on reverse or scratchpad-formatted $k$ digit addition data would find it easier to learn $k+1$ digit addition compared to a model trained with plain format addition.

\subsection{Fine-Tuning from Pretrained Models}

In this section, we investigate the generalization ability of transformer models, specifically focusing on their capacity to learn higher-digit additions based on their knowledge of lower-digit additions. Additionally, we explore how the choice of data format affects the number of samples required to learn higher-digit additions.

\paragraph{Forgetting of $k$-digit addition when trained on $k+1$-digit addition. }\mbox{}

We begin by fine-tuning a model that was initially trained on 3-digit addition. We fine-tune this model using 4-digit addition training data, with each data format being used separately. To mitigate the ``catastrophic forgetting'' phenomenon, we experiment with different learning rates, gradually reducing the magnitude. We continue this process until the learning rate becomes too small for the model to effectively learn 4-digit addition.

\begin{figure}[ht] 
\vspace{-2mm}
\centering
\includegraphics[width=0.5\textwidth]{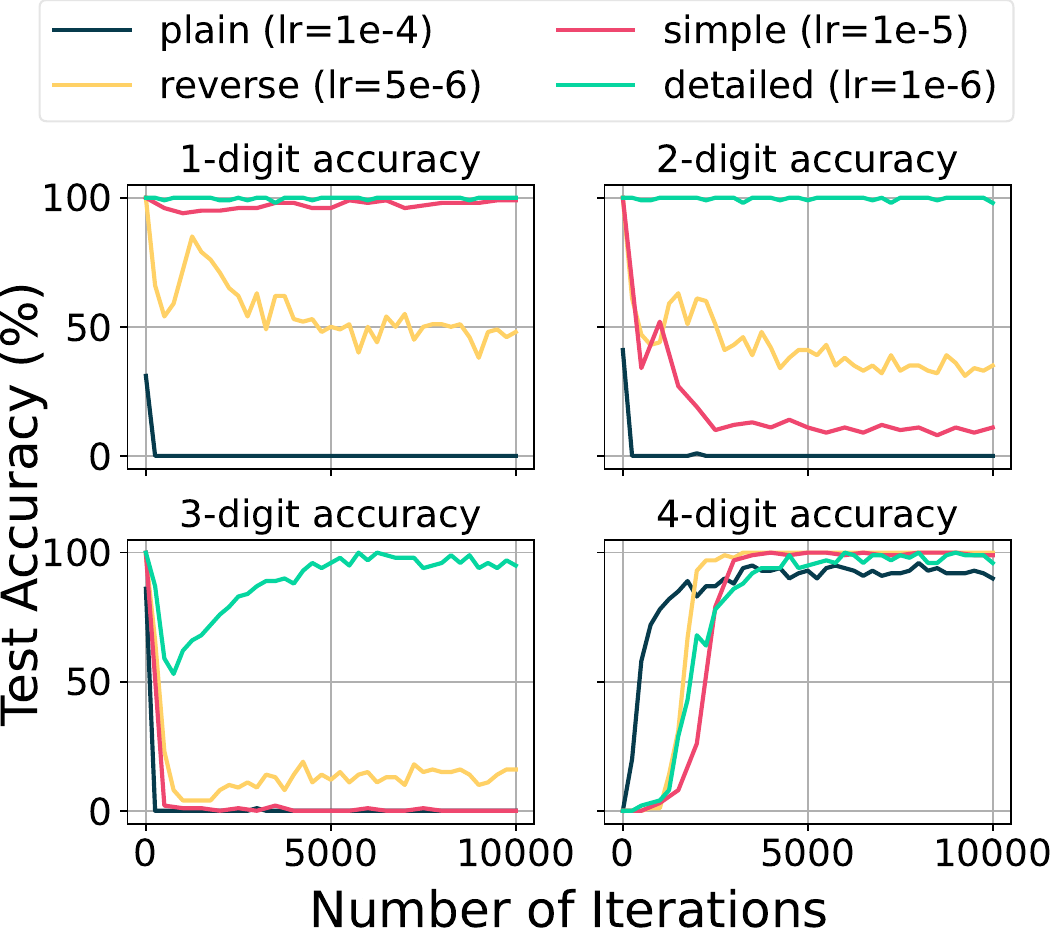}
    \caption{Accuracy of 1 to 4-digit additions during fine-tuning of a pretrained model on 3-digit additions using different data formats. The model is fine-tuned using only 4-digit addition data with corresponding formats. We observe that the plain format `forgets' 1 to 3-digit additions entirely when learning 4-digit addition. In contrast, the detailed scratchpad method successfully learns 4-digit addition while maintaining high performance on 1 to 3-digit additions.} 
\label{fig:ft_3to4digit}
\vspace{-2mm}
\end{figure}

The results depicted in Figure~\ref{fig:ft_3to4digit} reveal interesting insights about the fine-tuning process. When training the model using the plain format with only 4-digit addition data, there is an immediate drop in accuracy for 1 to 3 digit additions. This indicates that the model experiences significant forgetting of previously learned additions.
In contrast, the reverse and scratchpad methods exhibit a more favorable behavior. The model trained with these methods does not completely forget 1 or 2 digit additions while learning 4-digit addition. Remarkably, the detailed scratchpad method stands out by enabling the model to learn 4-digit addition without compromising its performance on 1 to 3 digit additions. Although there is a slight decrease in performance for 3-digit additions initially, the model quickly recovers and picks up the knowledge again as it trains on 4-digit additions.

This result can be explained by the hypothesis that learning a $k+1$ digit addition from a $k$-digit model is an incremental process for the detailed scratchpad method. The model already has a solid foundation in understanding the intermediate steps involved in addition, so it only needs to adapt to longer sequences. In contrast, for the plain format, learning higher-digit additions requires the model to establish new mappings to generate correct outputs, which is a more challenging task.

\paragraph{Sample efficiency of fine-tuning $k$-digit models with $k+1$-digit examples. }

Building upon our previous findings that fine-tuning a model solely on $k+1$-digit addition leads to a loss in performance for $k$-digit addition, we modify our approach to prevent the loss of performance in the $k$-digit addition task. Instead of training solely on $k+1$-digit examples, we construct a dataset that includes all addition tasks from 1-digit to $k+1$-digit, with the method described in the previous section. By doing so, we aim to maintain the performance of 1 to $k$-digit addition while enabling the model to learn $k+1$-digit addition during fine-tuning.

In this experiment, we investigate the number of $k+1$-digit training examples required for the model to effectively learn $k+1$-digit addition when fine-tuning a pretrained model on $k$-digit addition. It is important to note that this setting differs from the previous section (Section~\ref{sec:higher_digits_scratch}), where we focused on training models from random initialization. Here, we specifically focus on the fine-tuning process. We fine-tune individual models pretrained on each data format (using $k$-digit addition) and further train them using the same data format on a new dataset that includes all addition examples from 1-digit to $k+1$-digit.

\begin{figure}[ht]
\vspace{-4mm}
\centering
\subfloat[{Plain}]{\includegraphics[width=0.4\textwidth]{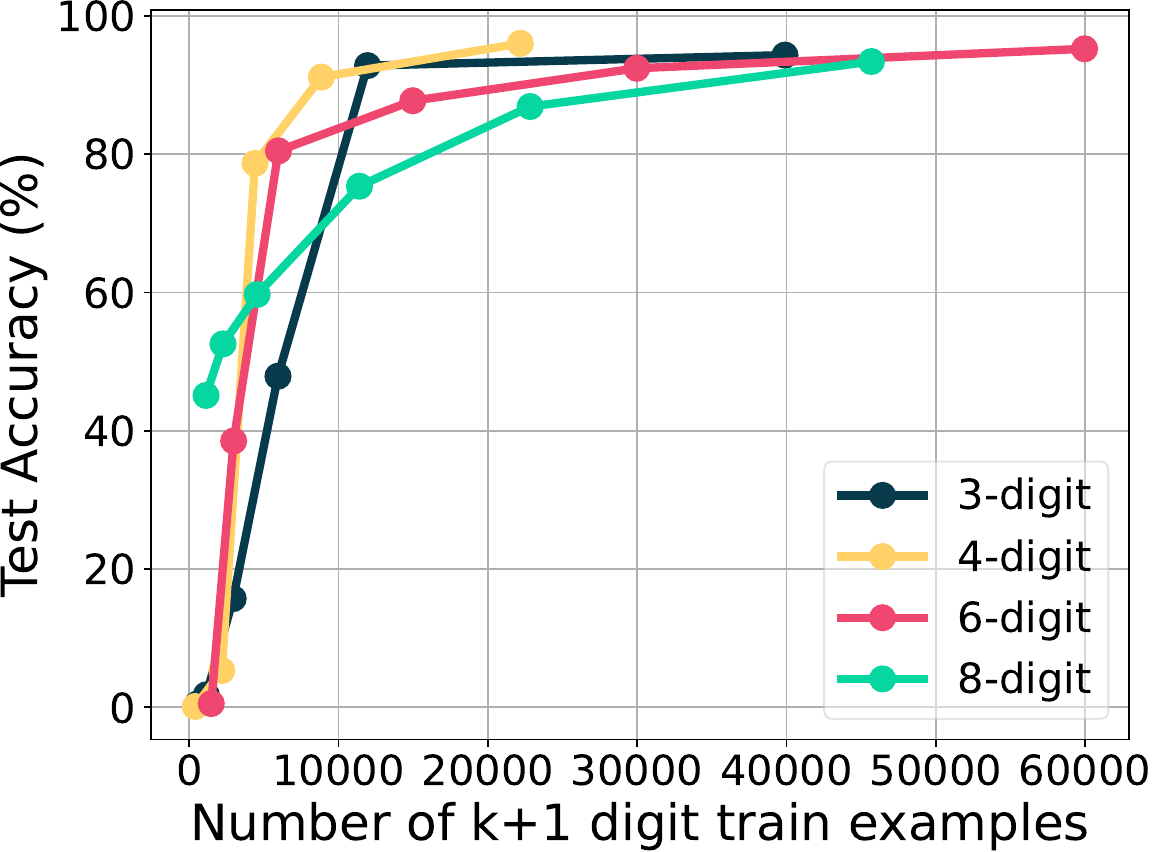}}
\hspace{1mm}
\subfloat[{Reverse}]{\includegraphics[width=0.4\textwidth]{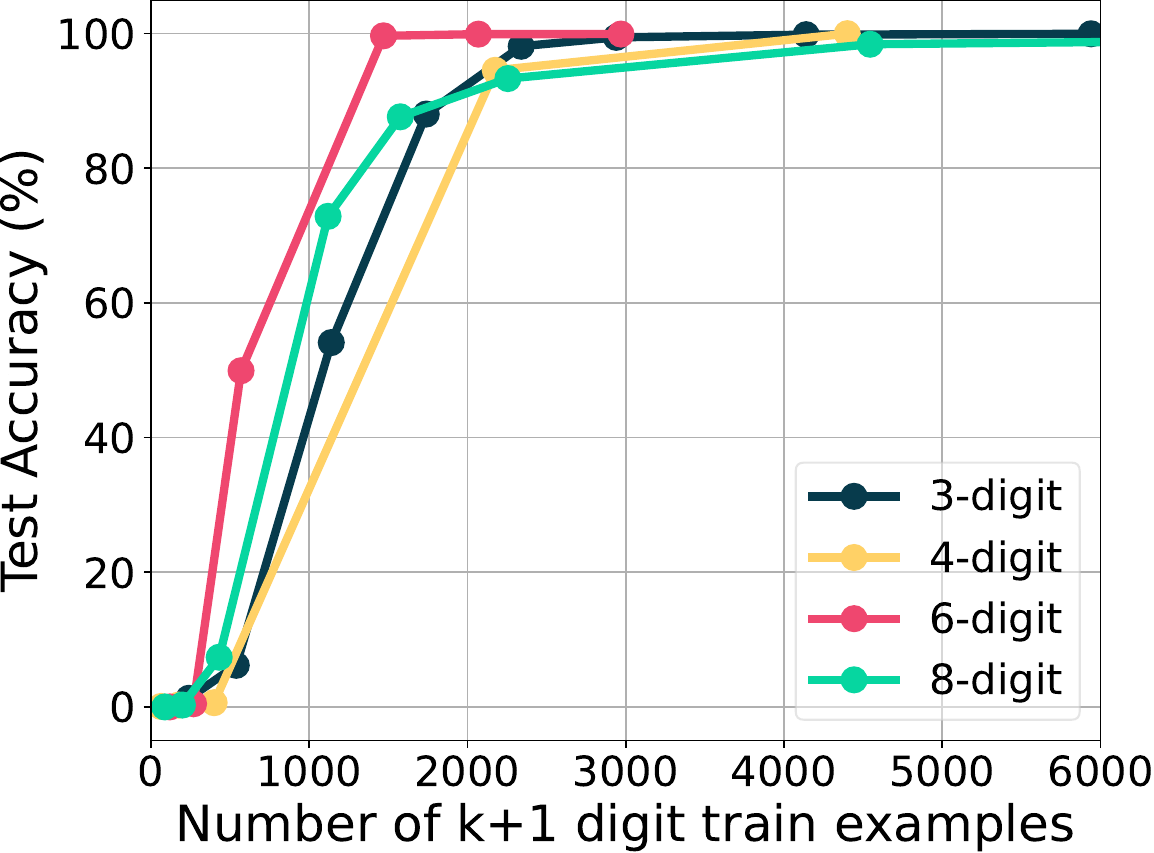}}
\vspace{0.1mm}
\subfloat[{Simplified Scratchpad}]{\includegraphics[width=0.4\textwidth]{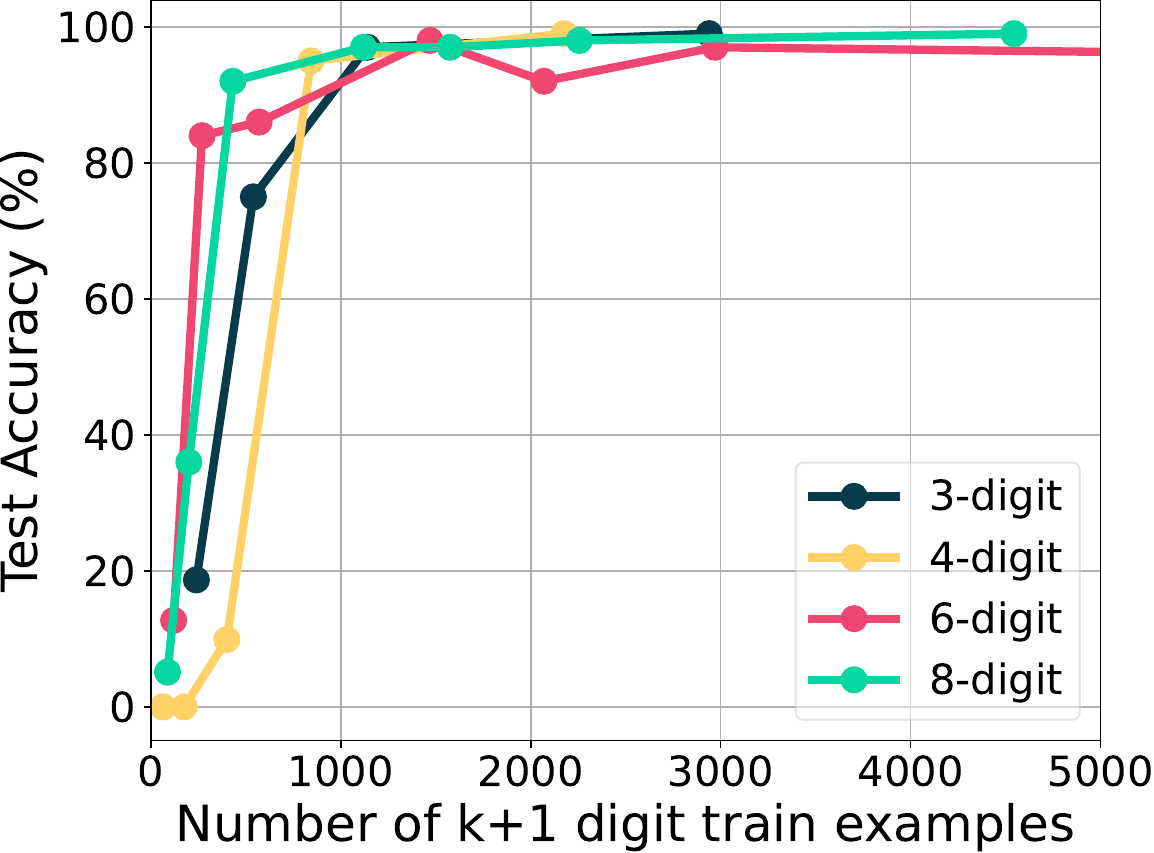}}
\hspace{1mm}
\subfloat[{Detailed Scratchpad}]{\includegraphics[width=0.4\textwidth]{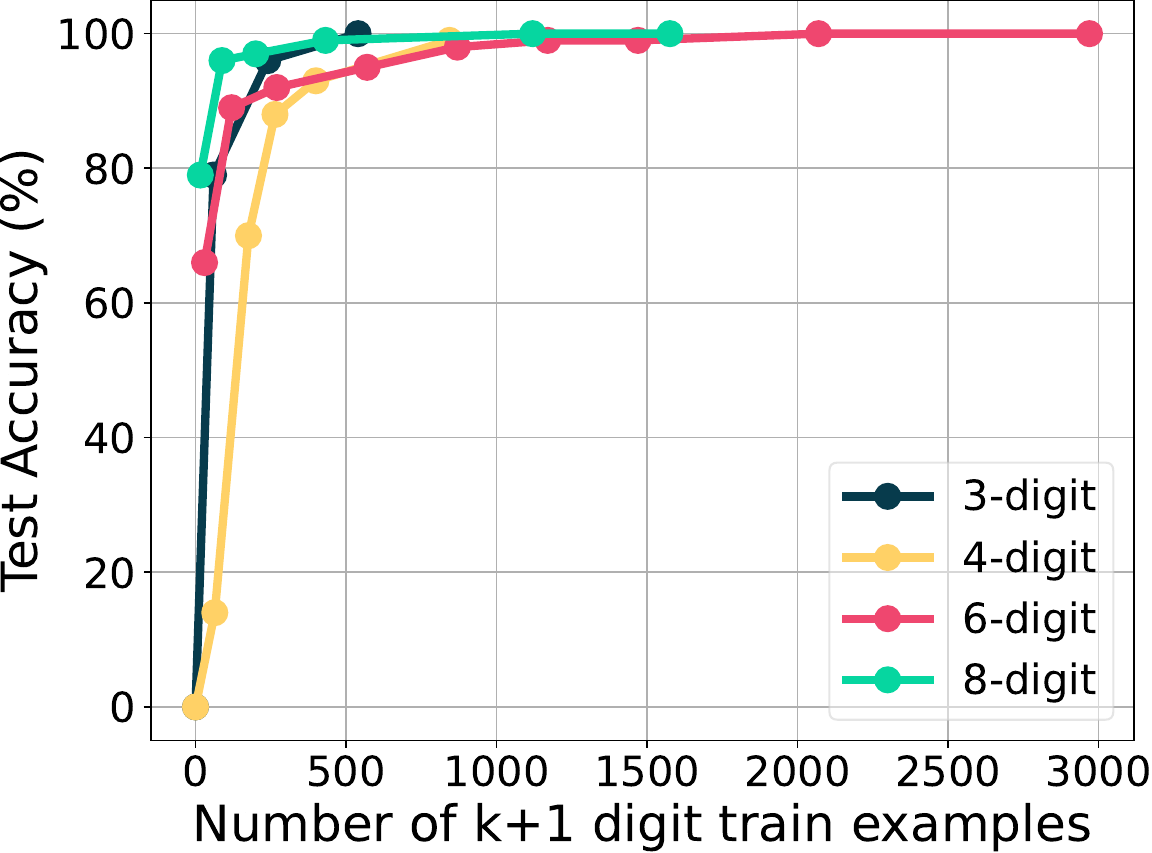}}
    \caption{
    Fine-tuning performance of pretrained $k$-digit models using varying numbers of $k+1$-digit examples, with corresponding data formats. The plain format requires an increasing number of $k+1$-digit examples as the number of digits ($k+1$) increases. In contrast, the modified formats (reverse, scratchpad) exhibit consistent performance across different numbers of digits, requiring a relatively consistent number of examples to learn the additional digit.
    } 
\label{fig:multidigit_ft}
\vspace{-4mm}
\end{figure}

The results in Figure~\ref{fig:multidigit_ft} demonstrate the number of $k+1$-digit addition samples required for a pretrained model capable of performing $k$-digit addition to learn the addition of $k+1$ digits. The findings reveal that modified formats (reverse, scratchpad) require a relatively small number of samples (between 1000 and 5000) to learn the addition of an extra digit. In contrast, the plain format necessitates a significantly larger number of training examples, with the requirement increasing as the number of digits grows.

This observation aligns with our previously established Lemma~\ref{lemma:right-to-left} and Lemma~\ref{lemma:left-to-right}, which suggest that learning higher-digit addition in the reverse format involves processing the $i$-th digit of the operands and carrying from the previous position. This operation \emph{remains consistent regardless of the number of digits being added}. As a result, the model primarily needs to learn how to handle longer digits to perform addition effectively.

In contrast, the plain addition format requires the model to learn a more complex function that incorporates all digits from both operands. As the number of digits increases, the complexity of this function grows as well. This highlights the greater difficulty faced by the plain format in accommodating additions with a larger number of digits.

\subsection{Impact of Formats on Fine-Tuning}
We delve deeper into the impact of different formats on the fine-tuning process. Specifically, we investigate whether training a model in one format helps in learning addition in another format, and vice versa. To conduct this analysis, we begin with a model trained on each data format using 3-digit addition examples. We then individually fine-tune these pretrained models using different data formats, on 4-digit addition examples.

\begin{wrapfigure}{r}{0.5\textwidth}
\vspace{-8mm}
\centering
\includegraphics[width=0.5\textwidth]{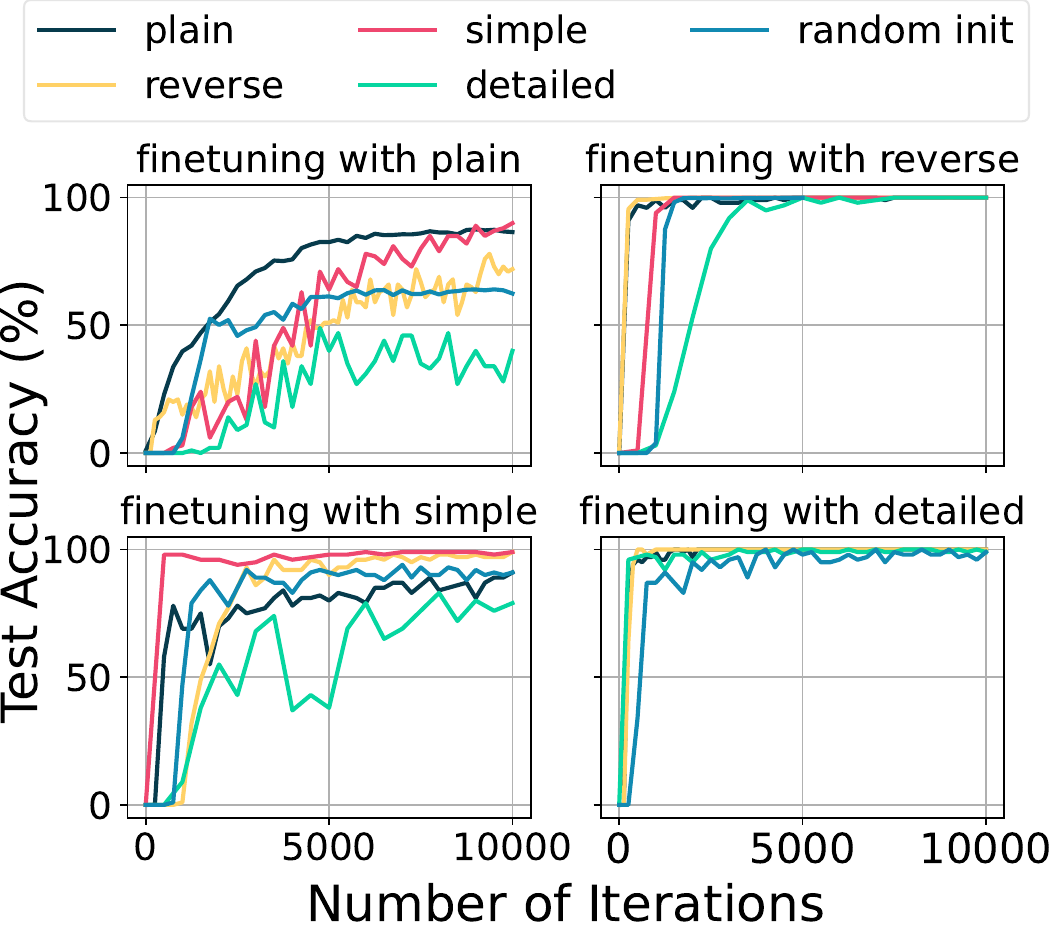}
    \caption{Performance of fine-tuning a 3-digit model trained on different data formats (plain, reverse, simple scratchpad, detailed scratchpad, and random initialization) individually with different data formats of 4-digit addition. The results demonstrate that fine-tuning yields the best performance when the pretrained model and the fine-tuning format are consistent. Notably, fine-tuning a detailed scratchpad format model shows suboptimal performance. We hypothesize that this is due to the need for the model to ``unlearn'' the rigid and verbose format and adapt to the new format.} 
\label{fig:ft_formats}
\vspace{-6mm}
\end{wrapfigure}

The results depicted in Figure~\ref{fig:ft_formats} highlight some interesting findings. Firstly, we observe that a model trained with the same format as the fine-tuning format exhibits faster learning in terms of the number of iterations. For instance, training a model with the plain format outperforms training a model pretrained with scratchpad formats. This suggests that the model benefits from the consistency and familiarity provided by the same format throughout the training process.

Additionally, we notice that fine-tuning a detailed scratchpad pretrained model on other formats proves to be more challenging. This observation can be attributed to the need for the model to ``unlearn'' the intricacies of the verbose detailed scratchpad format and adapt to the new format.  For example, the plain format does not involve the use of alphabet characters in the data, so a model pretrained with the plain format would have a low probability of generating alphabetic outputs. In contrast, a detailed scratchpad pretrained model would have encountered various alphabets and may have a tendency to output them. Therefore, adjusting to a new format requires additional effort for the model to ``unlearn'' the patterns specific to the previous format and effectively learn the new format it is being trained on.

These findings highlight the importance of considering format consistency during the fine-tuning process, as it can impact the efficiency and effectiveness of the learning process. We will delve further into this topic in the upcoming section~\ref{sec:finetuning_pretrained_models}, where we fine-tune pretrained GPT-3 models. Notably, we observe that fine-tuning with reverse or simplified scratchpad formats actually yields worse results compared to fine-tuning with plain formats. For a detailed exploration of these observations, please refer to the forthcoming section.

\section{Teaching Arithmetic Operations Beyond Addition}\label{sec:beyond_addition}
While this study has a primary focus on the \emph{addition} operation and aims to comprehend the significance of data sampling and formatting, its findings are applicable beyond the realm of addition alone. In this section, we expand our examination to include other arithmetic operations, thus demonstrating the broader applicability of our insights.
We consider a mix of arithmetic tasks, including binary operations like \emph{subtraction} and \emph{multiplication}, and unary operations such as \emph{sine} and \emph{square root}. Each operation entails its unique challenges and intricacies. For instance, subtraction introduces the concept of negative numbers, multiplication can generate significantly longer outputs, and sine and square root functions entail computations involving floating-point numbers, which are considered up to four digits of precision in our work.

We acknowledge that while our examination is detailed, it does not encompass all the fundamental arithmetic operations or the entire scope of floating-point arithmetic. Specifically, our focus is primarily on integer arithmetic for binary operations, considering a limited length of digits. Additionally, for unary operations, we confine ourselves to a restricted number of digits below the decimal point. 

In Section~\ref{sec:extended_arithmetic}, we delve into each arithmetic operation individually, exploring the impact of data formatting and determining the relevancy of our insights across disparate tasks. Further, in Section~\ref{sec:joint_arithmetic}, we perform an analysis of joint training across all five tasks, investigating the potential performance implications for each individual task.

\subsection{Extended Arithmetic Operations}\label{sec:extended_arithmetic}

In order to extend our analysis to arithmetic operations beyond addition, we consider the following tasks:
\paragraph{Subtraction ($-$).} We consider subtraction of positive numbers up to $3$ digits, written as $\mathsf{A_3A_2A_1 - B_3B_2B_1 = C_3C_2C_1}$ in (i) plain formatting, and $\mathsf{\$A_3A_2A_1 - B_3B_1B_1 = C_1C_2C_3\$}$ in (ii) reverse formatting. 
As with addition, scratchpad-based methods (iii, iv), present the intermediate steps of digit-wise subtraction and handling of carry-ons%
. These steps proceed from the least significant bit (LSB) to the most significant bit (MSB). If the final result after computing all the digit-wise subtractions is negative, we subtract the number in the most significant bit (MSB) position multiplied by 10 to the power of (number of digits in the output - 1) from the remaining digits in the output. In Section~\ref{sec:appendix_subtraction_detailed_scratchpad}, we present an alternative version of the detailed scratchpad formatting for subtraction. 

\paragraph{Multiplication ($\times$).} We consider multiplication of positive numbers up to 2-digits. (i) Plain formatting examples are formatted as $\mathsf{A_2A_1 * B_2B_1 = C_4C_3C_2C_1}$, while (ii) reverse formatting is formatted as $\mathsf{\$A_2A_1 * B_2B_1 = C_1C_2C_3C_4\$}$. The (iv) detailed scratchpad method simplifies each intermediate step by conducting a series of multiplications between the first operand and each digit of the second operand, starting from the least significant bit (LSB) and moving toward the most significant bit (MSB). For each step, we multiply the result by an exponentiation of 10 corresponding to the relative digit position.

\paragraph{Sine ($\sin{}$).} We consider decimal numbers within the range $[-\pi/2, \pi/2]$, truncated to 4-digit precision. (i) Plain formatting examples are formatted as $\sin(\mathsf{A_0.A_1A_2A_3A_4})=\mathsf{B_0.B_1B_2B_3B_4}$. For (iv) detailed scratchpad method, we include the Taylor series expansion steps for sine, which is represented as $\sin(x) = x - \frac{1}{3!}x^3 + \frac{1}{5!}x^5 - \frac{1}{7!}x^7 + \cdots $. These intermediate steps involve exponentiation, which may not be any easier to compute than the sine operation itself.

\paragraph{Square Root ($\sqrt{}$).}  We consider decimal numbers within $[1, 10)$, truncated to 4-digits of precision with the format, written as $\text{sqrt}(\mathsf{A_0.A_1A_2A_3A_4})=\mathsf{B_0.B_1B_2B_3B_4}$ for (i) plain formatting. For (iv) detailed scratchpad method, we enumerate each step of Newton's method to compute the square root function. The iterative formula is given by $x_n = \frac{1}{2} (x_{n-1} + \frac{x}{x_{n-1}})$, where $x_0$ is initialized as the floor of the square root value of the operand $x$. These intermediate steps involve a division operation, which can be as complex as the square root operation itself.

For evaluation of sine and square root, we classify the result $\hat{y_i}$ as correct if the absolute difference between $\hat{y_i}$ and the ground truth value $y_i$ is less than or equal to a predefined threshold $\epsilon \geq 0$. %

For each arithmetic task, we explore both the plain format and the detailed scratchpad format. The detailed scratchpad formatting for each task is illustrated in Figure~\ref{fig:ar_others} and Appendix~\ref{sec:prompt_examples}. For subtraction, the process involves breaking down the operation into intermediate steps of digit-wise subtraction, including carry-ons when necessary. Unlike addition, subtraction requires an additional step to handle cases where the first operand is smaller than the second. Further details on the detailed scratchpad for subtraction can be found in Section~\ref{sec:appendix_subtraction_detailed_scratchpad}.
For multiplication, each intermediate step carries out a $2$-digit $\times$ $1$-digit multiplication between the first operand and each separate digit of the second operand.
For sine and square root, we utilize a sequence of \emph{iterative approximations} instead of algorithmic explanations. Specifically, Taylor's series expansion steps for sine and Newton's method steps for square root are used.  
It is important to note that while addition, subtraction, and multiplication are broken down into simpler operations at each step, CoT for sine and square root functions requires intermediate steps involving operations like exponentiation or division, which might not be inherently simpler. 

\begin{AIbox}{\bf{\large Detailed scratchpad formatting for different arithmetic tasks}}
\vspace{5mm}
{\tt \footnotesize Examples of detailed scratchpad formatting for different arithmetic tasks:\\(1) Subtraction - includes borrows for intermediate steps, (2) Multiplication - decomposes the second operand for 2-digit $\times$ 1-digit multiplication at each step, (3) Sine - utilizes Taylor series expansion, and (4) Square root - employs Newton's method.
\\}

\begin{minipage}[t]{0.48\linewidth}
\centering
\textbf{Subtraction}
\begin{lstlisting}[language=markdown]
Input:
128-367
Target:
<scratch>
[1,2,8] has 3 digits.
[3,6,7] has 3 digits.
[1,2,8] - [3,6,7] , A=[] , C=0 , 8-7-0=1 , A->1 , C->0
[1,2] - [3,6] , A=[1] , C=0 , 2-6-0+10=6 , A->6 , C->-1
[1] - [3] , A=[6,1] , C=-1 , 1-3-1=-3 , A->-3 , C->-1
[] - [] , A=[-3,6,1]
-300+61=-239 , END
</scratch>
-2 3 9
\end{lstlisting}

\textbf{Multiplication}
\begin{lstlisting}[language=markdown, showlines=true,]
Input:
12*36
Target:
<scratch>
[1,2] has 2 digits.
[3,6] has 2 digits.
[1,2] * 6 , A=[7,2] , k=1 , B=[7,2] , C=0+72=72
[1,2] * 3 , A=[3,6] , k=10 , B=[3,6,0] , C=72+360=432 , END
</scratch>
4 3 2

\end{lstlisting}
\end{minipage}
\begin{minipage}[t]{0.52\linewidth}
\centering
\textbf{Sine}
\begin{lstlisting}[language=markdown, showlines=true,]
Input:
sin(1.5707) 
Target:
<scratch> 
x_0=1.5707 
x_1: x_0 - 1/3! * (x^3) , x_1=0.9247 
x_2: x_1 + 1/5! * (x^5) , x_2=1.0043 
x_3: x_2 - 1/7! * (x^7) , x_3=0.9996 
x_4: x_3 + 1/9! * (x^9) , x_4=0.9997 , END 
</scratch>
0.9997

\end{lstlisting}
\textbf{Sqrt}
\begin{lstlisting}[language=markdown]
Input: 
sqrt(2.7174)  
Target:
<scratch> 
x_0=1
x_1: 1/2*(1+2.7175/1)=1.8587, x_1=1.8587
x_2: 1/2*(1.8587+2.7175/1.8587)=1.6603, x_2=1.6603 
x_3: 1/2*(1.6603+2.7175/1.6603)=1.6485, x_3=1.6485
x_4: 1/2*(1.6485+2.7175/1.6485)=1.6484, x_4=1.6484 , END
</scratch> 
0.6484
\end{lstlisting}
\end{minipage}
\end{AIbox}
\noindent\begin{minipage}{\textwidth}
\captionsetup{type=figure}
\captionof{figure}{Examples of the detailed scratchpad format for different arithmetic tasks such as subtraction, sine, multiplication, and square root.}\label{fig:ar_others}
\end{minipage}

\begin{figure}[ht]
\vspace{-4mm}
\centering
\subfloat[{Subtraction}]{\includegraphics[width=0.4\textwidth]{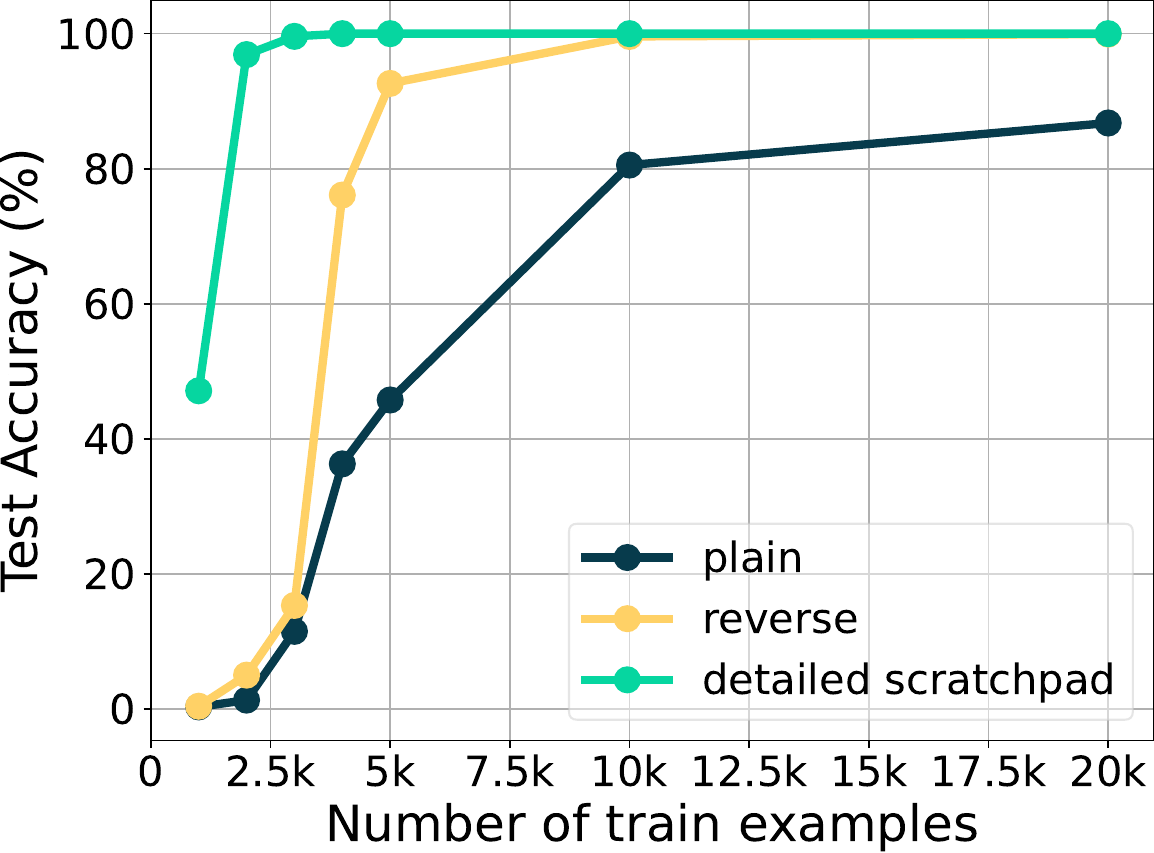}}
\hspace{6mm}
\subfloat[{Multiplication}]{\includegraphics[width=0.4\textwidth]{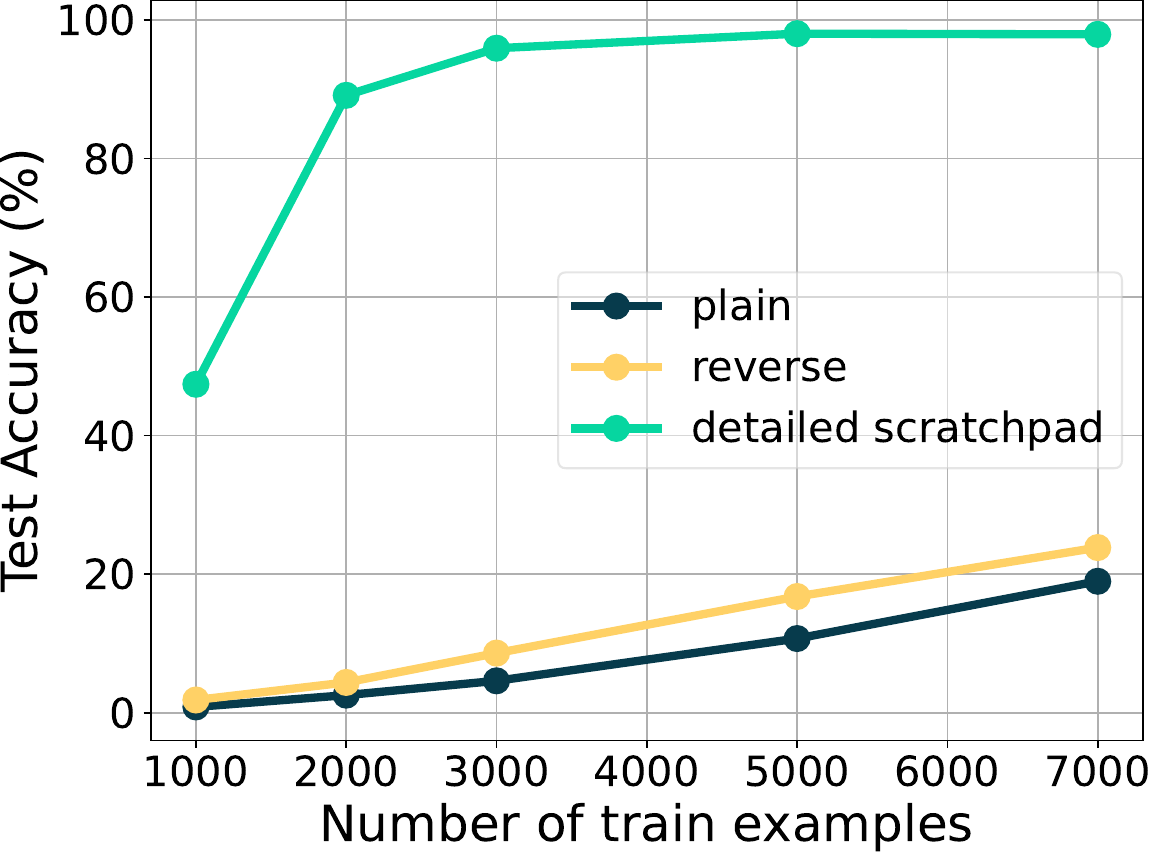}}
\vspace{0.1mm}
\subfloat[{Sine}]{\includegraphics[width=0.4\textwidth]{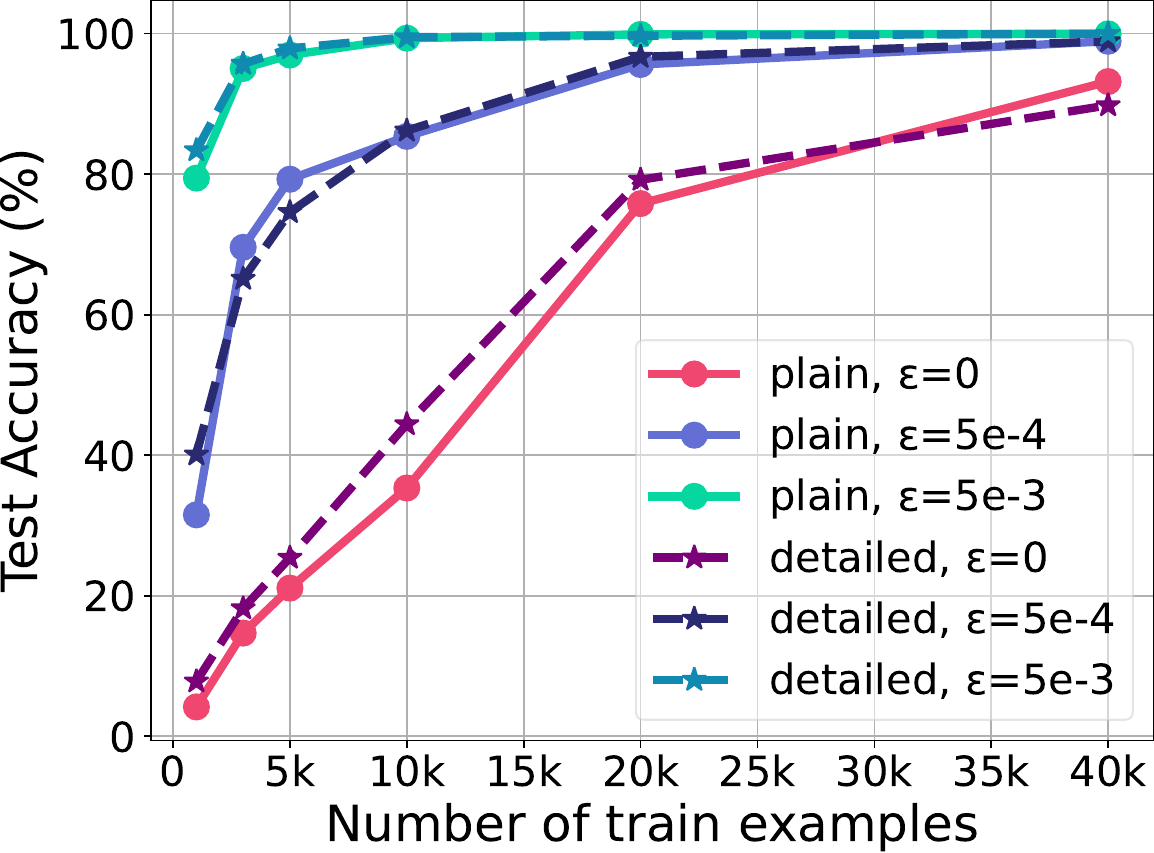}}
\hspace{6mm}
\subfloat[{Square Root}]{\includegraphics[width=0.4\textwidth]{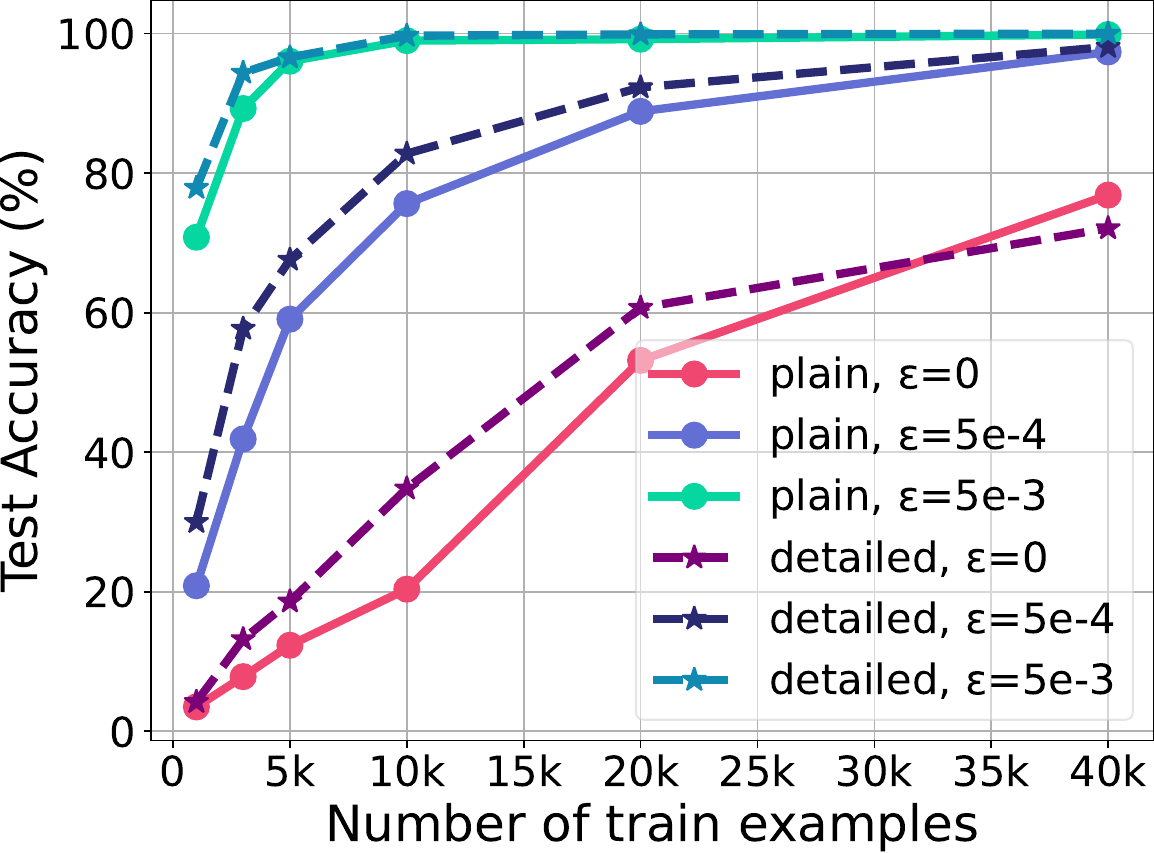}}
    \caption{Performance of $3-$digit subtraction, $2-$digit multiplication, $4-$digit precision sine and square root with varying data formats. As with addition, reverse always produces improved sample complexity and performance for all operations. For sine and square root, scratchpad formatting provides limited improvement. This discrepancy can be attributed to the complexity of the intermediate steps involved in the detailed scratchpad. 
    } 
\label{fig:sample_efficiency_sub_mul}
\vspace{-2mm}
\end{figure}

The results depicted in Figure~\ref{fig:sample_efficiency_sub_mul} indicate that similar to the findings of addition, the detailed scratchpad format significantly improves performance over \emph{plain} or \emph{reverse} formats and yields efficient results even with few samples for subtraction and multiplication tasks. Interestingly, we find \emph{reverse} is not particularly effective in multiplication. 
On the other hand, the detailed scratchpad format exhibits reduced efficiency for $\sin{}$ and $\sqrt{}$ compared to other operations ($+,-,\times$).  This discrepancy can be traced back to the complexity of the intermediate steps involved in the detailed scratchpad. While addition, subtraction, and multiplication are decomposed into simpler functions, sine and square root operations involve more intricate operations. For a broader analysis of the error profile, see Appendix~\ref{sec:sin_sqrt_analysis}.

\subsection{Jointly Training on All Five Arithmetic Tasks}\label{sec:joint_arithmetic}

So far, we only considered the problem of learning different arithmetic operations individually. In this section, we study the effect of jointly training on all five arithmetic tasks - addition, subtraction, multiplication, sine, and square root. We construct a single train dataset incorporating all task $\gD_\text{train}=\{\gD_\text{train}^{+},\gD_\text{train}^{-},\gD_\text{train}^{\times},\gD_\text{train}^{\sin{}},\gD_\text{train}^{\sqrt{}}\}$, and randomize the sequence of tasks in our train samples. For example, a randomly chosen segment of the training data may exhibit a task order such as $(+,-,\sin{}.-,\times,\times,\sqrt{},...)$. We consider $10,000$ training examples for each task of addition, subtraction, sine, and square root and $3,000$ for multiplication.

The model's performance, after training on our joint dataset $\gD_\text{train}$, is evaluated in both zero-shot and few-shot settings. These results are also compared with the performance of models that were trained separately on each dataset $(\gD_\text{train}^{+},\gD_\text{train}^{-},\gD_\text{train}^{\times},\gD_\text{train}^{\sin{}},\gD_\text{train}^{\sqrt{}})$, identical to those used to construct $\gD_\text{train}$. In the few-shot setting, each task is given examples from any of the five arithmetic tasks (not necessarily related to the test task under consideration) or prompt texts, followed by test queries specific to the task of interest. For further details on the few-shot prompting methods used, please refer to Section~\ref{sec:exp3}.

Table~\ref{tab:multiple_tasks} shows that joint training significantly enhances the zero-shot performance for multiplication and square root tasks, yet it slightly reduces the performance for subtraction.  Generally, few-shot prompting exhibits improved performance. Notably, the performance of few-shot prompting remains consistent regardless of whether the exemplars provided are from unrelated tasks or are task-specific. We propose that this consistency is due to our randomized task sequence during training, which presents the model with numerous instances where one task directly follows another, thus simulating few-shot prompting with different tasks.  Furthermore, we observe that text prompting performs similar to zero-shot. We conjecture that this is because the training data does not include text data and the model has never encountered text and therefore, text prompting serves as a random prefix attached to our test query.

\begin{table}[ht]
\center
\caption{Performance of models trained individually and jointly on five arithmetic tasks. The threshold $\epsilon$ for $\sin{}$ and $\sqrt{}$ functions is set to 0. For the models trained jointly on all five tasks, we evaluate their performance in both a zero-shot setting and a few-shot setting. In the few-shot setting, each task is presented with exemplars from one of the five arithmetic tasks or prompted with text, followed by task-specific test queries. The results show that few-shot prompting with any arithmetic operators (even unrelated to the test task) generally improves performance. However, text prompting shows performance similar to the zero-shot setting. }
\label{tab:multiple_tasks}
\vspace{1mm}
\centering
\setlength{\tabcolsep}{4pt} %

\begin{tabular}{c|c|ccccccc}
\toprule
         & \multirow{3}{*}{\shortstack{Trained on\\individual task}} & \multicolumn{7}{c}{Trained jointly on all 5 tasks}                 \\ \cline{3-9}
         &            & \multicolumn{1}{c}{\multirow{2}{*}{Zero-shot}} & \multicolumn{6}{c}{Few-shot exemplar format}            \\ \cline{4-9}
         &            &  \multicolumn{1}{c}{}          & +     & --    & $\times$ & sin   & sqrt  &  text     \\ \midrule
+        & 84.06      & 87.96     & 96.45 & 96.90 & 96.92    & 97.06 & 97.01 & 88.71 \\
--       & 79.97      & 72.83     & 81.28 & 79.59 & 81.39    & 81.84 & 81.74 & 68.91 \\
$\times$ & 4.58       & 14.28     & 18.86 & 18.96 & 15.43    & 19.20 & 19.59 & 15.48 \\
sin      & 35.03      & 34.74     & 34.35 & 34.31 & 34.34    & 32.64 & 33.42 & 33.96 \\
sqrt     & 19.85      & 27.37     & 26.65 & 26.74 & 26.70    & 25.60 & 25.61 & 26.02 \\ 
\bottomrule
\end{tabular}
\end{table}

\section{Mixing Shakespeare with Arithmetic Data}\label{sec:exp3}
Until now, our focus was primarily on models trained exclusively on arithmetic tasks. However, in practice, large language models (LLMs) utilize a combination of arithmetic and \emph{text} data for training. In this section, we broaden our scope by incorporating both addition samples and text into our pretraining data. We then evaluate the trained models with various few-shot prompts to analyze if the model is able to effectively identify the correct context.

\textbf{Experimental Setup. } We mix addition and text data in our experiment using the Shakespeare dataset~\citep{shakespeare} that includes $1,115,394$ tokens of text, $10,000$ plain addition examples ($120,027$ tokens), and $3,000$ detailed scratchpad formatted addition examples ($813,510$ tokens). We fix the number of detailed scratchpad examples and plain addition examples ($3,000$ and $10,000$ respectively) while varying the number of each example type in the training process. The Shakespeare text is segmented into dialogue chunks, with a random number of addition data inserted between them.
We use a character-level tokenizer with a vocabulary size of $80$, containing all characters present in the dataset, including alphabets, digits, and certain symbols like $+,=$ and \textit{\textbackslash n}.

\textbf{Few-shot prompting. } 
Given the mixed nature (arithmetic and text) of our dataset, introducing relevant examples seems an effective strategy to prime the model to generate the desired type of output. To assess the performance of such few-shot $(1/2/3-$shot) prompting, we provide task-specific exemplars as illustrated in Figure~\ref{fig:few-shot-prompting}. Plain addition formatted exemplars are used for testing plain addition inputs, while detailed scratchpad formatted exemplars are utilized for assessing performance on detailed scratchpad formatted inputs. Additionally, we experiment with demonstrating text (see Appendix~\ref{sec:appendix_text_prmopt}. for details) before querying addition (which we denote, Text-prompt). For each 1/2/3-shot and text prompting, average performance is reported over a fixed set of exemplars. Standard deviations of these prompts are denoted by shaded areas in the plots. The term ``few-shot'' refers to the reported mean of all 1/2/3-shot prompting results.

\begin{figure}[h] 
\centering
\includegraphics[width=0.95\textwidth]{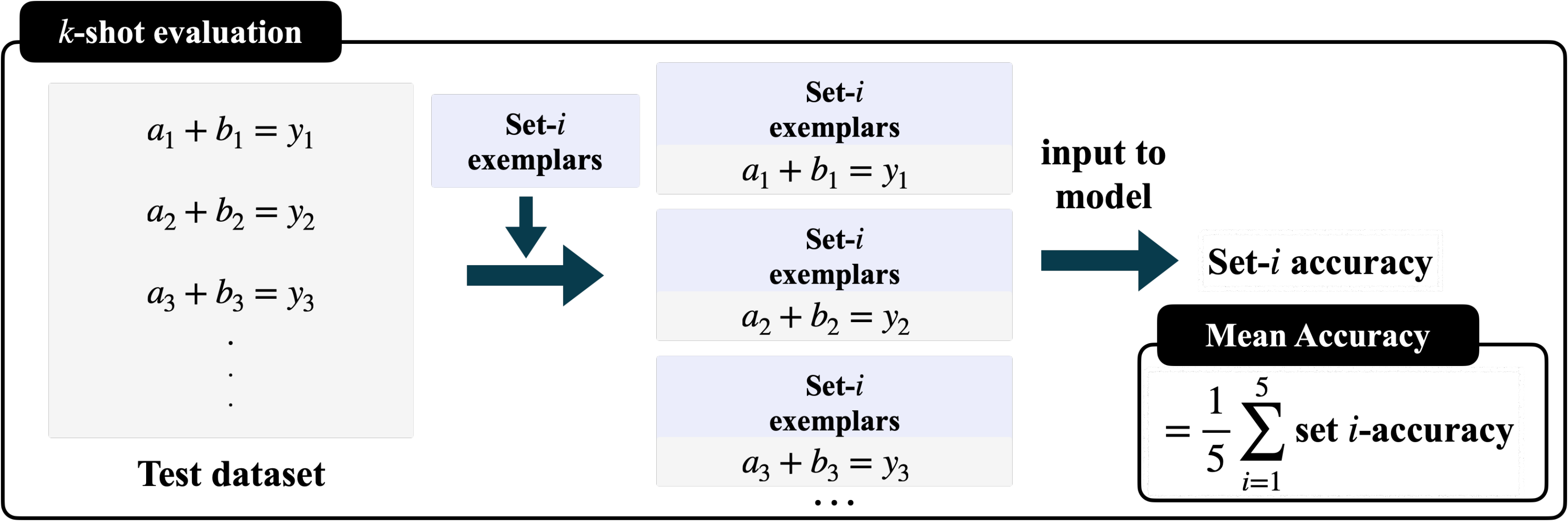}
    \caption{Few-shot prompting method. Few-shot prompting performance is evaluated by presenting relevant exemplars of addition and detailed scratchpad formatted inputs. Each 1/2/3-shot prompting is tested on a fixed five set of exemplars, and the accuracy is averaged over these evaluations.} 
\label{fig:few-shot-prompting}
\end{figure}

\begin{figure}[ht] 
\centering
\subfloat[{Test accuracy on plain addition}]{\includegraphics[width=0.40\textwidth]{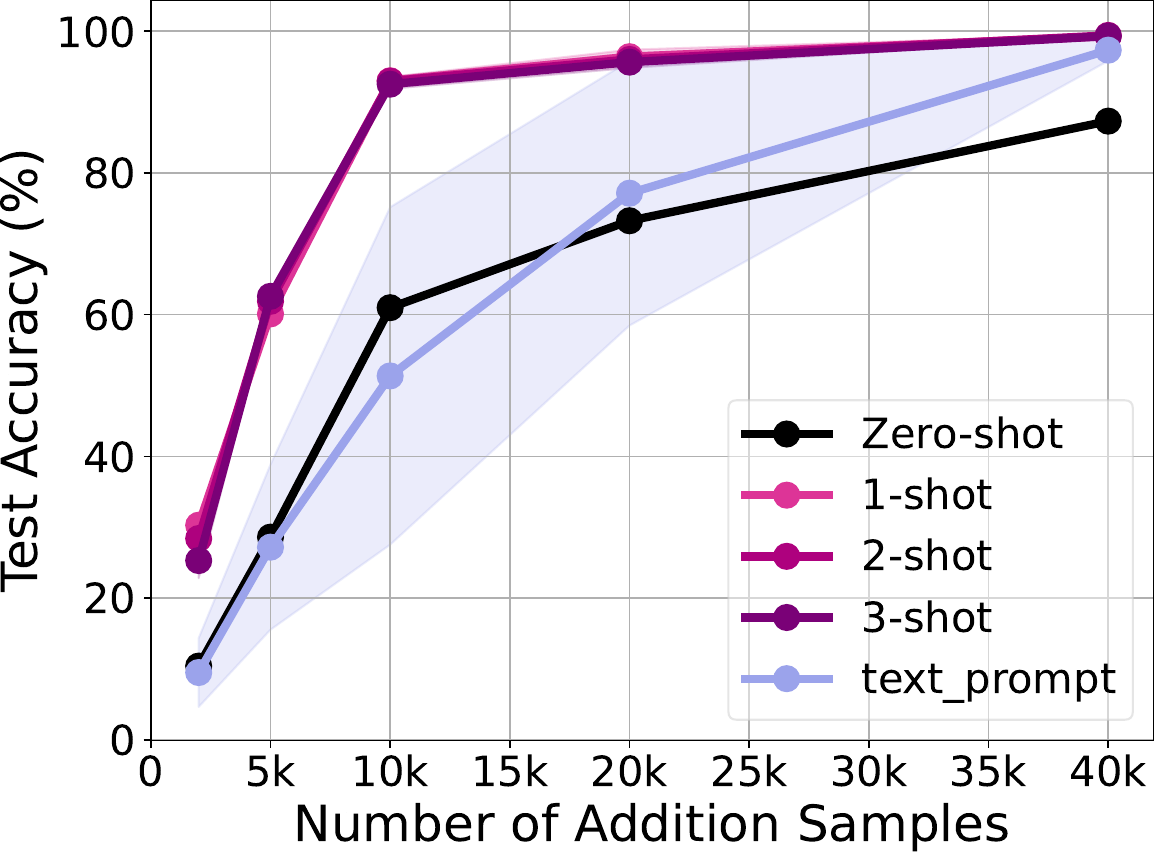}}
\hspace{0.1mm}
\hspace{6mm}
\subfloat[{Test accuracy on detailed scratchpad}]{\includegraphics[width=0.40\textwidth]{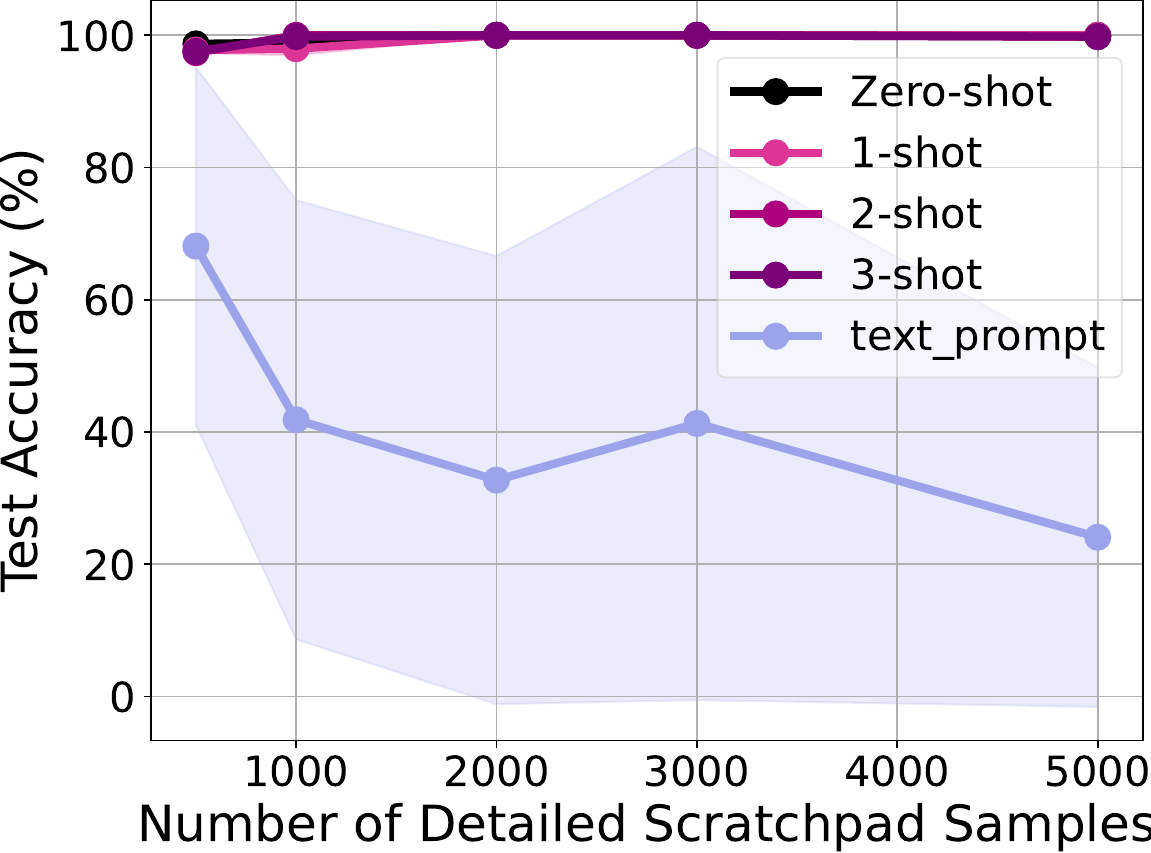}}
    \caption{Performance of NanoGPT model trained with the Shakespeare dataset, addition dataset in plain, and detailed scratchpad format. The number of plain (left) and detailed scratchpad (right) formatted addition samples are varied. Performance is evaluated on zero-shot, few-shot, and text prompts, with the shaded area representing the standard deviation across various prompt exemplar sets. The results indicate a consistent enhancement in model performance using few-shot prompting.} 
\label{fig:NanoGPT_mixed_performance}
\vspace{-2mm}
\end{figure}

Figure~\ref{fig:NanoGPT_mixed_performance} shows that few-shot prompting directs the enhancement of performance, thereby allowing plain addition to perform almost perfectly with 40,000 train samples. Intriguingly, performance remains high on plain addition even with the inclusion of a text prompt, given a substantial number of addition examples. We hypothesize that this is due to the structure of our mixed dataset where addition examples are interspersed within Shakespeare data. With the incorporation of more addition examples, instances where addition examples directly follow Shakespeare text increase, leading to a decrease in potential inconsistencies when text content is present during addition test queries.

\begin{wrapfigure}{r}{0.46\textwidth}
\centering
\includegraphics[width=0.40\textwidth]{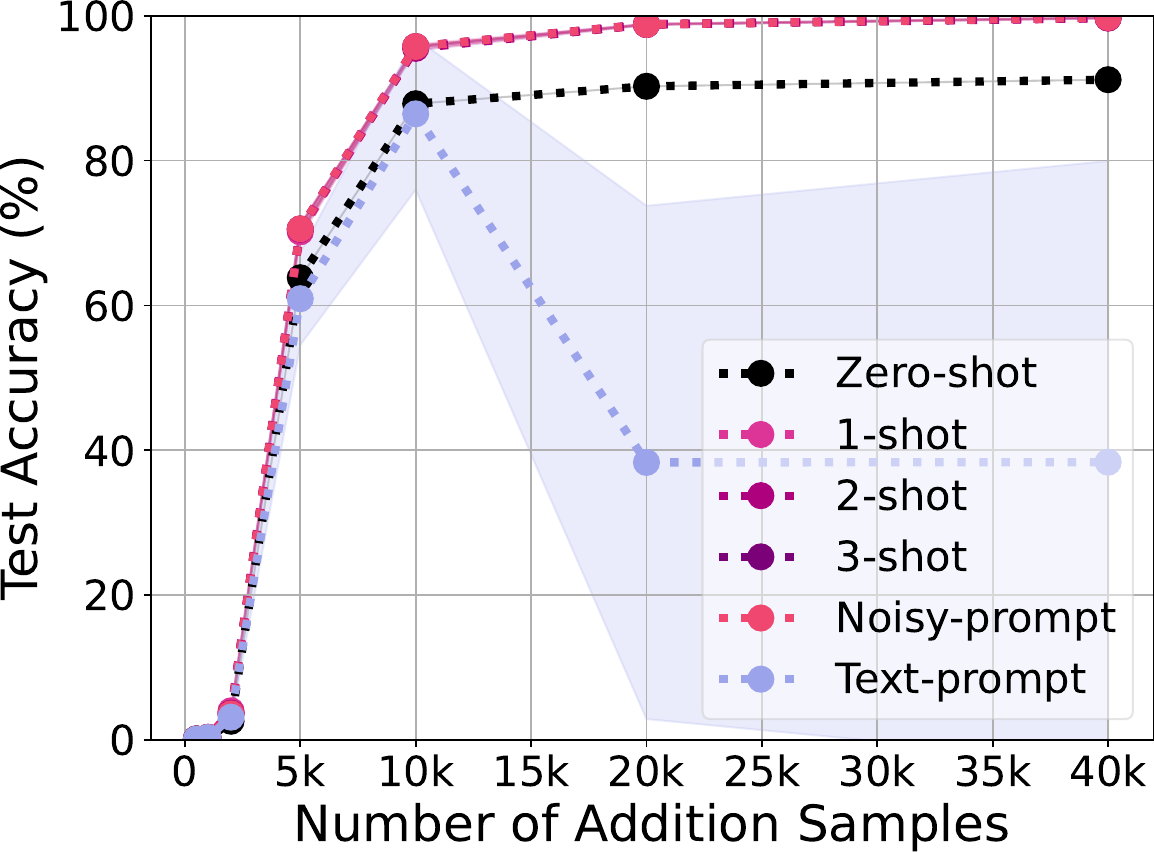}
\caption{Performance of NanoGPT model trained exclusively on plain addition, but with an extended vocabulary including both addition and alphabets (vocabulary size = 80). Few-shot prompting, using both correct addition examples (1, 2, 3-shot) and incorrect addition examples (noisy-prompt) leads to enhanced performance, while the use of text prompts results in a degradation of performance when the model is trained solely on addition.} 
\label{fig:nanogpt_only_addition}
\end{wrapfigure}

To disentangle the effects of the textual content in the training data, we train a model strictly on plain addition, utilizing an enlarged vocabulary that also includes alphabet characters, thereby enabling text prompting. (Note that previous experimental settings on plain formatted additions used a vocabulary size of 13, which only includes 10 numerals and 3 symbols - ``+'',``='',``\textbackslash n''). We introduce a variant of few-shot prompting, termed as \textbf{noisy-prompt}, which prompts the model with erroneous addition exemplars, \ie, $\mathsf{A + B = C}$, where $\mathsf{C \neq A+B}$.

Figure~\ref{fig:nanogpt_only_addition} shows that few-shot prompting contributes to performance enhancement even when the model is confined to training on a single plain addition task. Even in the presence of noisy prompting, simply providing the model with the $\textsf{A + B = C}$ format yields performance nearly identical to few-shot prompting, aligning with the result observed by \citet{min2022rethinking}. Conversely, we notice that text prompts negatively influence performance when the model is trained only on addition. This finding reinforces our earlier observation in Figure~\ref{fig:NanoGPT_mixed_performance} that the advantageous impact of text prompts originates from the combined text and addition data.

\section{Fine-tuning, Scaling, and Pretraining in Larger Models}\label{sec:finetuning_pretrained_models}

This section focuses on bridging the gap between our experiments on NanoGPT and the more realistic setting of larger language models like GPT-2 and GPT-3. We begin by comparing the performance of NanoGPT and GPT-2 models when trained from random initialization. This comparison highlights the improved performance achieved with the larger model scale, especially in the zero-shot setting. Subsequently, we delve into the impact of tokenization methods and model pretraining in GPT-2 models. Our exploration reveals the crucial role of pretrained models and the consistent tokenization of numbers (achieved by introducing spaces) during the training phase for arithmetic tasks. Building on these findings, we proceed to fine-tune a pretrained GPT-3 model on various arithmetic tasks, employing different data formats.

\paragraph{Comparing NanoGPT and GPT-2.}
To examine the impact of scale on arithmetic performance, we explore a larger GPT-2 model with $85$ million parameters, featuring twice as many self-attention layers, heads, and embedding size compared to the previously used NanoGPT model. We train the GPT-2 model from scratch using character-level tokenization, jointly on text and addition tasks, adopting both plain and detailed scratchpad formats; an approach mirroring the setting in Section~\ref{sec:exp3}. The results depicted in Figure~\ref{fig:nanogpt_vs_gpt2} demonstrate that the larger model outperforms in both plain and detailed scratchpad evaluations. For a comprehensive analysis of GPT-2, including few-shot learning and the influence of text prompts, refer to Figure~\ref{fig:mixed_text_prompt_performance} and Figure~\ref{fig:mixed_performance}.

\begin{figure}[ht] 
\vspace{-2mm}
\centering
\subfloat[{Test accuracy on plain addition}]{\includegraphics[width=0.4\textwidth]{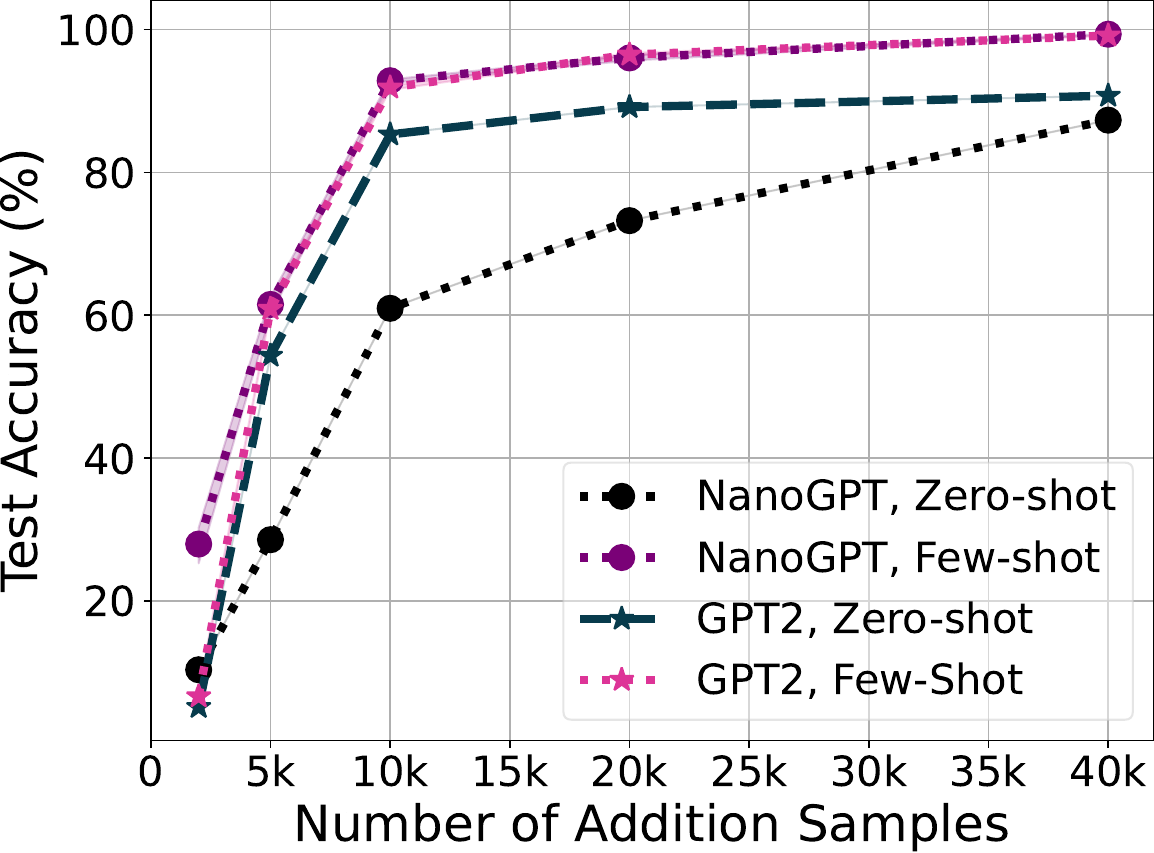}}
\hspace{6mm}
\subfloat[{Test accuracy on detailed scratchpad}]{\includegraphics[width=0.4\textwidth]{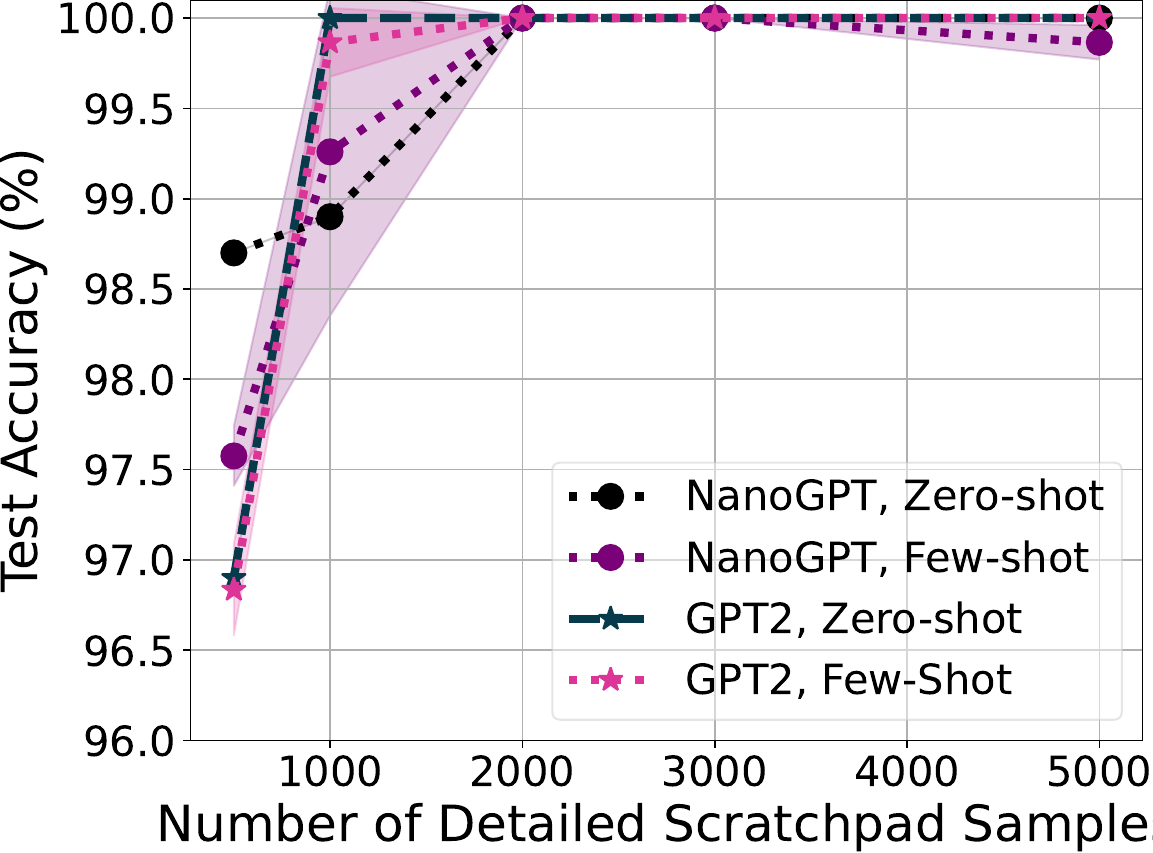}
}
    \caption{Comparing NanoGPT and GPT-2 on addition task.  We compare the performance of NanoGPT and GPT-2 models trained jointly on the Shakespeare dataset and addition tasks using plain and algorithmic reasoning formatting. The results indicate that larger models exhibit improved performance, and using few-shot prompting enhances performance as well. The left side shows results for plain data formatting, while the right side presents results for algorithmic reasoning data formatting.} 
\label{fig:nanogpt_vs_gpt2}
\vspace{-4mm}
\end{figure}

\paragraph{Going from character-level tokenization to BPE.}
The transition to a GPT-2 setup necessitates several modifications. Firstly, we shift to OpenAI's Tiktoken BPE tokenizer, which is the default tokenizer for the pretrained GPT-2 model, featuring a vocabulary size of 50,257. We also examined two different training approaches: training the model from random initialization (scratch) and fine-tuning the pretrained model sourced from Huggingface. To ensure uniform digit tokenization, alterations were made in data formatting to include spaces between numbers. This change aims to circumvent potential inconsistent tokenization of numbers while utilizing the Tiktoken tokenizer.  

Figure~\ref{fig:gpt2_tokenization} shows that GPT-2 demonstrates high performance in addition tasks with both character-level tokenization and Tiktoken with spaces between digits.  This aligns with the results by ~\citet{wallace2019nlp}, suggesting that character-level tokenization exhibits stronger numeracy capabilities compared to a word or sub-word methods. Furthermore, comparing the models trained from scratch and the models trained from the pretrained model, we observe that fine-tuning a pretrained model results in better performance compared to training a model from scratch.

\paragraph{GPT-3 experiments: Supervised fine-tuning.}
We extend our experiments to verify if our observations hold while fine-tuning larger pre-trained models. In the following, we consider three GPT-3 variants: Ada, Curie, and Davinci. Note that since we perform fine-tuning using the OpenAI APIs, by default only the completions are loss generating tokens. Therefore, these experiments are slightly different when compared to the previous settings. We fine-tune these models using the same four data formatting methods as our NanoGPT experiments: (i) \emph{plain} formatting, (ii) \emph{reverse} formatting, (iii) \emph{simplified scratchpad}, and (iv) \emph{detailed scratchpad}. These formats are identical to those from our NanoGPT experiments except for one aspect. We introduce spaces between numbers in \emph{plain} and \emph{reverse} formatting to ensure consistent tokenization.

Due to budget constraints, all experiments were conducted using a fine-tuning dataset of $1,000$ examples, and models were trained for 4 epochs. Performance evaluation was carried out on $1,000$ examples that were disjoint from the training dataset. Note that this training scale is significantly smaller than our experiments on NanoGPT, which employed $10,000$ training examples for $5,000$ iterations, with evaluations conducted on $10,000$ test examples. However, given these models' extensive pretraining on large data corpora, this scale can be deemed rational.

\begin{figure}[ht] 
\centering
\includegraphics[width=0.50\textwidth]{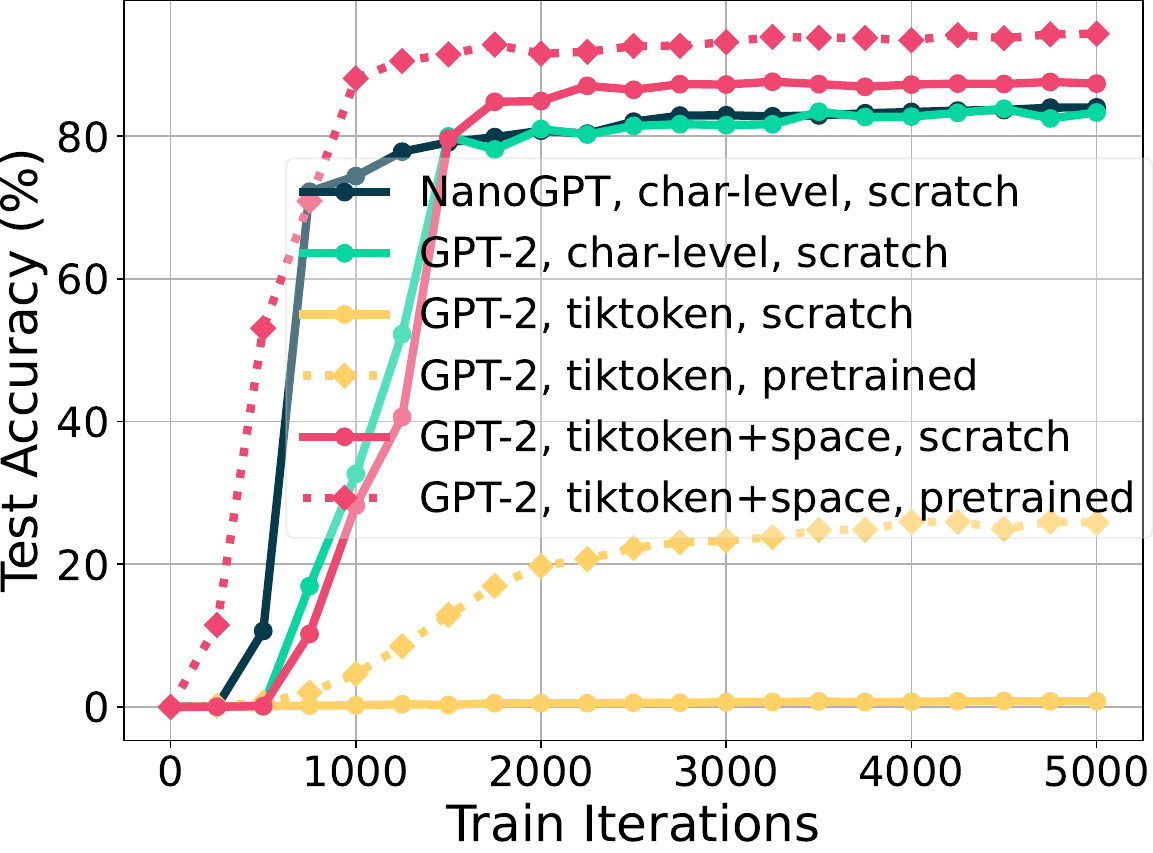}
    \caption{Performance of various configurations of the GPT-2 model on the addition task. We compare the effects of tokenization methods, specifically character-level tokenization versus Tiktoken (OpenAI's BPE tokenizer), training initialization (training from scratch versus training from a pretrained GPT-2 model), and the inclusion or exclusion of spaces between numbers. The results highlight the significance of utilizing pretrained models and incorporating spaces for consistent tokenization of numbers when training a model for arithmetic tasks.} 
\label{fig:gpt2_tokenization}
\vspace{-2mm}
\end{figure}

The results for addition and subtraction tasks are presented in Table~\ref{tab:gpt3_addition} and Table~\ref{tab:gpt3_subtraction}, respectively. We observed that initiating with a pretrained GPT-3 model significantly improves performance compared to training NanoGPT or GPT-2 models from random initialization with only 1000 samples. This indicates the utility of leveraging pretrained models for improved arithmetic performance. Interestingly, while reverse formatting and simplified scratchpad formats improve addition performance, they adversely affect subtraction performance. This observation is consistent with our earlier finding depicted in Figure~\ref{fig:ft_formats}, wherein transitioning from one data format to another often results in lower performance compared to initiating training from random initialization. We postulate that this discrepancy may be due to the pretrained GPT-3 model's requirement to adapt to the reversed approach and ``unlearn'' its knowledge of plain formatting arithmetic, thereby introducing additional complexity.  On the other hand, the detailed scratchpad method achieves excellent performance, albeit with increased training and inference costs due to higher token requirements.

\begin{table}[ht!]
\caption{Evaluation of addition performance for fine-tuned GPT-3 models: Davinci, Curie, and Ada.}
\label{tab:gpt3_addition}
\vspace{2mm}
\centering
\setlength{\tabcolsep}{4pt}

\begin{tabular}{c|ccccc}
\toprule
Addition & \multirow{2}{*}{\begin{tabular}[c]{@{}c@{}}pretrained\\ GPT-3\end{tabular}} & \multicolumn{4}{c}{Finetuned with 1000 samples}                 \\ \cline{3-6} 
         &                                                                             & Plain & Reverse & Simplified Scratchpad & Detailed Scratchpad \\ \midrule
Davinci  & 2\%                                                                         & 34\%  & 80.9\%  & 88.7\%                & \textbf{99.5\% }               \\
Curie    & 0.0\%                                                                       & 1.4\% & 12.3\%  & 10.7\%                & \textbf{99.7\%}                \\
Ada      & 0.0\%                                                                       & 0.3\% & 6.3\%   & 0.6\%                 & \textbf{99.8\%}                 \\ \bottomrule
\end{tabular}

\end{table}

\begin{table}[ht!]
\caption{Evaluation of subtraction performance for fine-tuned GPT-3 models: Davinci, Curie, and Ada.}
\label{tab:gpt3_subtraction}
\vspace{2mm}
\centering
\setlength{\tabcolsep}{4pt}

\begin{tabular}{c|ccccc}
\toprule
Subtraction & \multirow{2}{*}{\begin{tabular}[c]{@{}c@{}}pretrained\\ GPT-3\end{tabular}} & \multicolumn{4}{c}{Finetuned with 1000 samples}                  \\ \cline{3-6} 
            &                                                                             & Plain  & Reverse & Simplified Scratchpad & Detailed Scratchpad \\ \midrule
Davinci     & 0.1\%                                                                      & 84.8\% & 66.0\%  & 15.4\%                & \textbf{99.5\% }               \\
Curie       & 0.1\%                                                                       & 24.1\% & 6\%     & 3.8\%                 & \textbf{92.5\% }               \\
Ada         & 0.0\%                                                                       & 3.7\%  & 2.6\%   & 3.4\%                 & \textbf{81.5\%}                \\ \bottomrule
\end{tabular}
\end{table}

\begin{table}[ht!]
\caption{Evaluation of sine and square root performance for fine-tuned GPT-3 models: Davinci, Curie, and Ada.}
\label{tab:gpt3_sine_sqrt}
\vspace{0.2cm}
\centering
\small
\setlength{\tabcolsep}{4pt}
\begin{tabular}{cc|ccc|ccc}
\toprule
\multicolumn{1}{l}{}     & \multicolumn{1}{l|}{} & \multicolumn{3}{c|}{Sine}                                                                                                      & \multicolumn{3}{c}{Square Root}                                                                                               \\ \midrule
                         & \multirow{2}{*}{eps}  & \multirow{2}{*}{\begin{tabular}[c]{@{}c@{}}pretrained\\ GPT-3\end{tabular}} & \multicolumn{2}{c|}{Finetuned with 1000 samples} & \multirow{2}{*}{\begin{tabular}[c]{@{}c@{}}pretrained\\ GPT-3\end{tabular}} & \multicolumn{2}{c}{Finetuned with 1000 samples} \\ \cline{4-5} \cline{7-8} 
                         &                       &                                                                             & Plain           & Detailed Scratchpad          &                                                                             & Plain           & Detailed Scratchpad         \\ \midrule
\multirow{3}{*}{Davinci} & 0                     & 0\%                                                                         & \textbf{11.0\% }         & 10.3\%                         & 0\%                                                                         & 0.7\%           & \textbf{4.6\% }                        \\
                         & 5e-4                  & 0\%                                                                         & \textbf{35.9\%}          & 29.7\%                         & 0\%                                                                         & 7.5\%           & \textbf{17.2\% }                       \\
                         & 5e-3                  & 0.4\%                                                                       & \textbf{85.5\% }         & 72.8\%                         & 0\%                                                                         & 59\%            & \textbf{60.5\%  }                      \\
\multirow{3}{*}{Curie}   & 0                     & 0.0\%                                                                       & \textbf{8.6\% }          & 1.2\%                          & 0.0\%                                                                       & 0.7\%           & \textbf{2.1\% }                        \\
                         & 5e-4                  & 0.4\%                                                                       & \textbf{32.7\% }         & 5.4\%                          & 0.1\%                                                                       & \textbf{6.5\% }          & 6.0\%                         \\
                         & 5e-3                  & 0.9\%                                                                       & \textbf{80.8\% }         & 15\%                           & 0\%                                                                         & \textbf{52.7\%  }        & 30.2\%                        \\
\multirow{3}{*}{Ada}     & 0                     & 0.0\%                                                                       & \textbf{5.8\% }          & 4.3\%                          & 0.0\%                                                                       & 0.3\%           & \textbf{2.7\% }                        \\
                         & 5e-4                  & 0.0\%                                                                       & \textbf{21.4\% }         & 9.1\%                          & 0.0\%                                                                       & 3.8\%           & \textbf{11.9\% }                       \\
                         & 5e-3                  & 0.3\%                                                                       & \textbf{67.8\% }         & 25.2\%                         & 0.0\%                                                                       & 32.2\%          & \textbf{45.8\% }                       \\ \bottomrule
\end{tabular}
\end{table}

For the more complex sine and square root tasks as shown in Table~\ref{tab:gpt3_sine_sqrt}, we found that training with only 1000 samples is insufficient to generate exact answers (eps=0). The GPT-3 model, fine-tuned with 1,000 samples, performs worse than the NanoGPT model trained with 10,000 samples. Further experiments with larger training datasets are necessary for deeper insights and improved performance on these tasks.

It is worth mentioning that while few-shot prompting notably improves the performance of all three GPT-3 models, their zero-shot performance is quite poor (as shown in the leftmost column of the tables). However, post-training, few-shot prompting becomes less effective as OpenAI's fine-tuning process trains the model on individual prompts and desired completions serially, rather than in concatenation with multiple examples like in our NanoGPT experiments. Consequently, our comparisons primarily focus on the \textbf{zero-shot performances} of each task.

\section{Token Efficiency Across Data Formats}\label{sec:cost_analysis}
\begin{wrapfigure}{r}{0.425\textwidth}
\vspace{-6mm}
\centering
\includegraphics[width=0.42\textwidth]{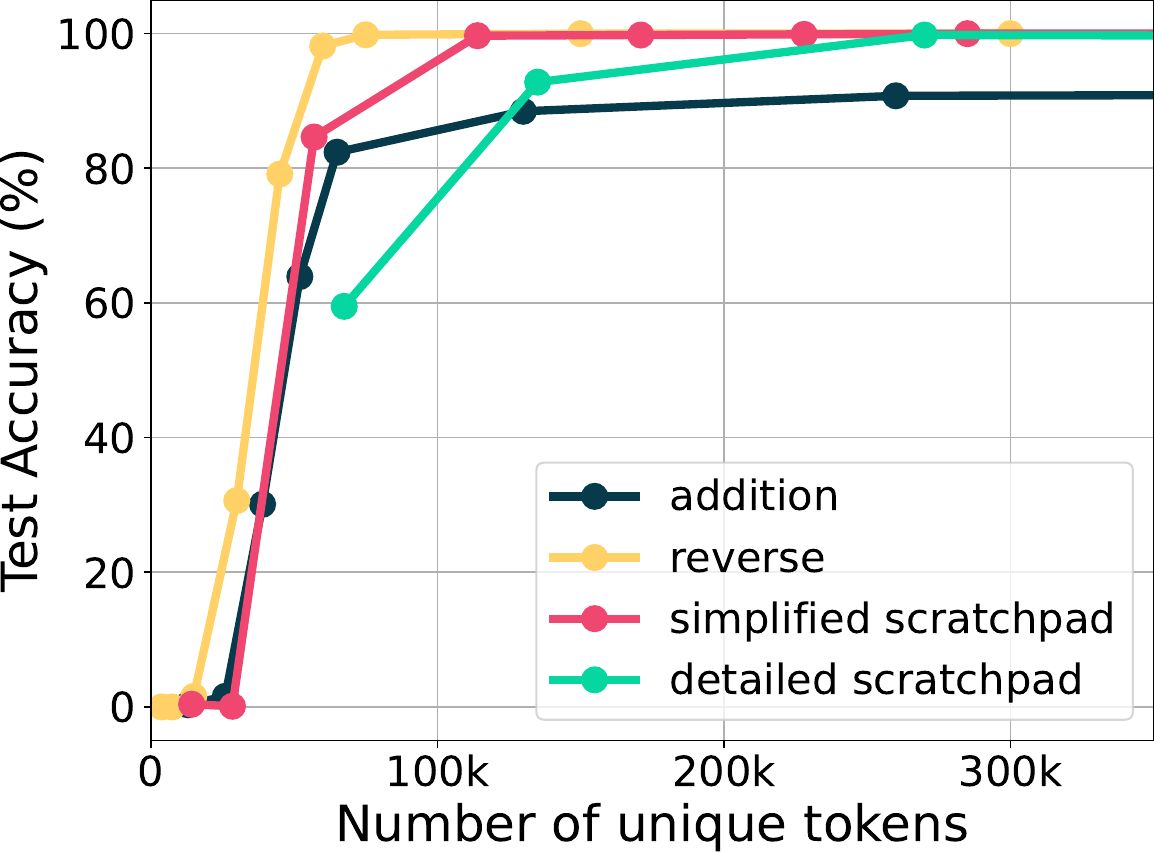}
    \caption{Number of unique tokens required for training addition on NanoGPT using different data formatting methods. The number of unique tokens is calculated by multiplying the number of training samples by the number of tokens per sample. The results demonstrate that the reverse format is the most efficient in terms of token usage for model training, as the scratchpad methods, although more sample-efficient, require more tokens per sample.} 
\label{fig:sample_efficiency_num_tokens}
\end{wrapfigure}

Figure~\ref{fig:nanogpt_sample_efficiency} demonstrates that more detailed training data leads to improved sample efficiency. However, this comparison does not account for the cost associated with training and inference. To address this, we conduct a cost analysis based on the number of ``unique'' tokens encountered during training. Each data sample is treated as a set of unique tokens, and the number of unique tokens is derived by \emph{multiplying the number of samples with the tokens per sample}. For instance, the mean token count for a single training example in a $3$-digit addition task is $13$ for plain format, $15$ for reverse format, $64$ for simplified scratchpad format, and $281$ for detailed scratchpad format. Note that this calculation does not evaluate uniqueness of tokens across samples \ie if the first sample is $\textsc{``112 + 129 = 241''}$ and the second sample is $\textsc{``112 + 128 = 240''}$, we will still consider that the model has seen $26$ unique tokens even though only two tokens differ across samples. This approach ensures our cost calculation accounts for a vanilla implementation of attention with no additional optimizations~\citep{pope2023efficiently}. Table~\ref{tab:num_tokens_per_format} presents the number of tokens required for prompting and completion in each data format, per example. Evidently, the detailed scratchpad method uses considerably more tokens compared to other techniques.

The result in Figure~\ref{fig:sample_efficiency_num_tokens} indicates that reverse formatting is the most token-efficient approach. While detailed scratchpad training is more sample efficient, it necessitates a larger number of tokens per sample, both during training and inference. Given that the inference cost for commercial models is determined by the number of tokens utilized per inference call (sum of prompting and completion tokens), abundant use of models trained on detailed scratchpad formats may escalate overall costs. Furthermore, since the cost of a single forward pass is cubic in the number of tokens, this is important to consider. Therefore, for practical usage, it is crucial to evaluate both the number of samples needed for achieving the desired performance and the actual token demands during training and inference.

\begin{table}[th!]
\vspace{-3mm}
\center
  \caption{
  Token requirements for prompting and completion per single example of 3-digit addition.
  }
  \label{tab:num_tokens_per_format}
  \vspace{1mm}
\centering
\small
\begin{tabular}{lcccc}
\toprule
           & Plain & Reverse & Simplified Scratchpad & Detailed Scratchpad \\ \midrule
Prompt     & 8     & 9       & 23                    & 23                  \\
Completion & 5     & 6       & 41                    & 258                 \\ 
\textbf{Total}      & \textbf{13}    & \textbf{15}      & \textbf{64}                    & \textbf{281}                 \\ \bottomrule
\end{tabular}
\vspace{-3mm}
\end{table}

\section{Length Generalization}\label{sec:appendix_lengthgeneralization}

In this section, we present results from experiments conducted to assess the model's ability to generalize across different digit lengths. Initially, we exclude training examples featuring 2-digit operands from the 10,000-sample addition dataset, yielding a reduced dataset of 7,655 samples, consisting solely of 1 or 3-digit operands. The model is trained with reverse format and its performance is evaluated on test dataset containing 100 random samples of 1-digit, 2-digit, 3-digit, and 4-digit additions. The results in Figure~\ref{fig:length_generalization} demonstrate that the NanoGPT model is incapable of performing 2-digit and 4-digit additions. This suggests an inherent necessity for exposure to all digit combinations to perform accurate calculations and lacks generalization capabilities for unseen digit lengths.

Additionally, we investigate the model's ability to extrapolate over larger digit lengths. The model is trained on 7-digit plain-formatted additions (each digit addition comprises 16650 samples, except 1-digit addition, which is trained on 100 samples). Its ability to add add 8-digit numbers is then put to test. The results in Figure~\ref{fig:length_generalization} show that the model is unable to generalize to a greater number of digits beyond what it has been trained on. Similarly, when training the model on 10-digit binary numbers, it fails to generalize to 11-digit binary additions, further confirming its limited ability to handle unseen digit combinations.

\begin{figure}[ht!] 
\centering
\subfloat[{Trained on 1 and 3 digit addition}]{\includegraphics[width=0.45\textwidth]{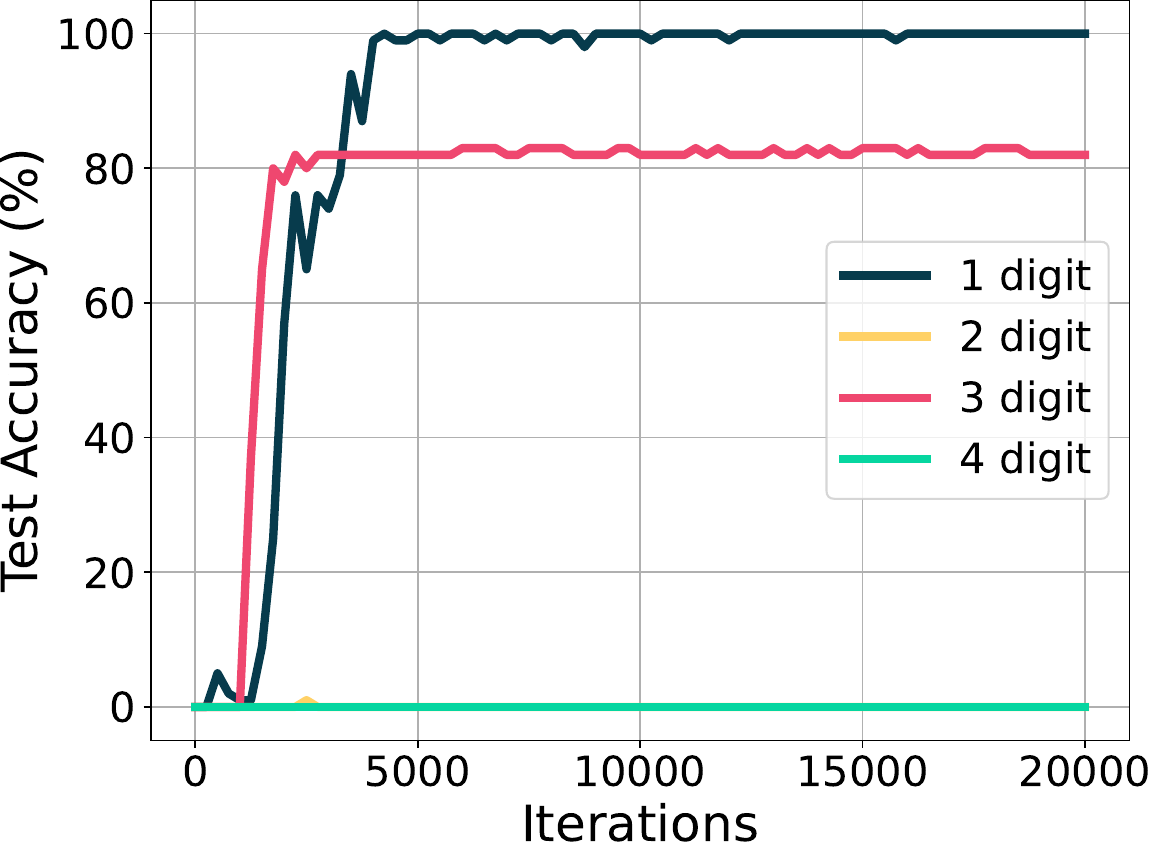}}
\hspace{0.5cm}
\subfloat[{Trained on 1 -- 7 digit addition}]{\includegraphics[width=0.45\textwidth]{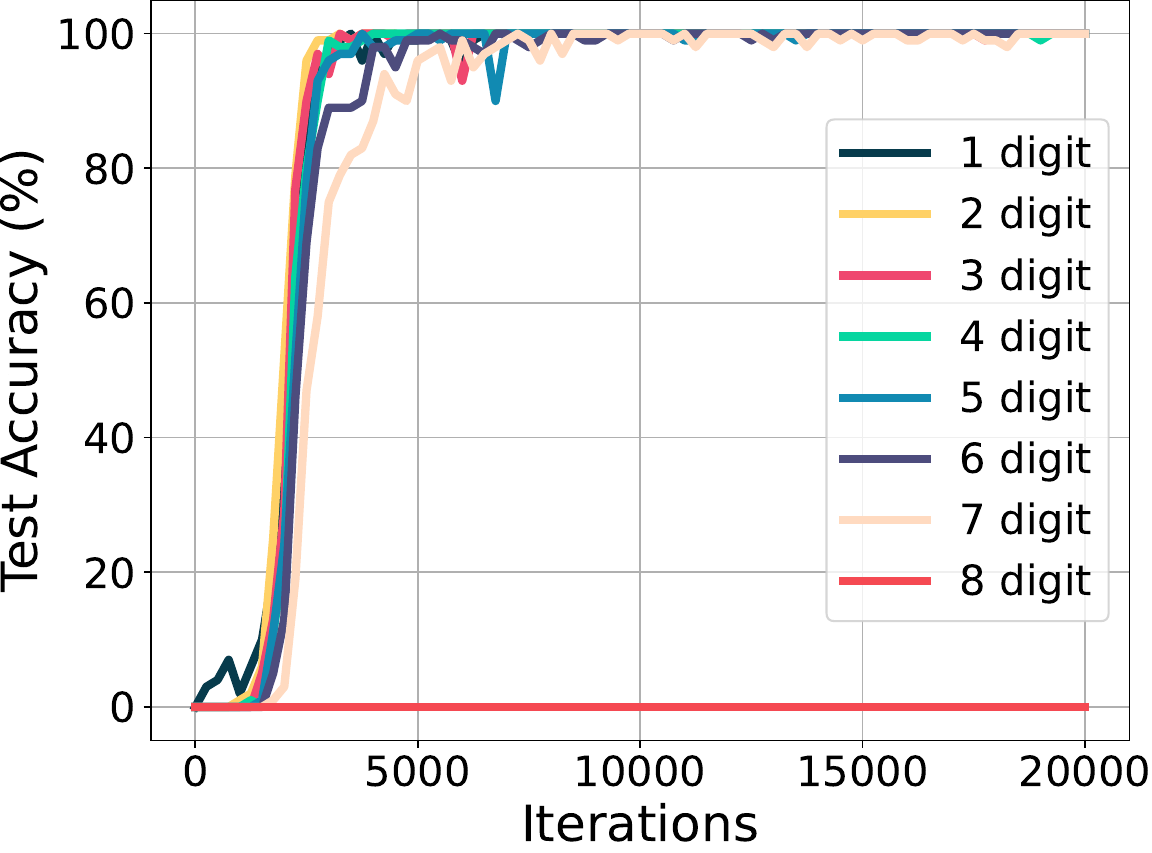}}
    \caption{Generalization experiments testing NanoGPT's performance on unseen numbers of digits in addition tasks. (Left): NanoGPT trained on reverse formatted addition with 1 and 3 digits, and tested on additions ranging from 1 to 4 digits. (Right): NanoGPT trained on up to 7-digit plain formatted addition and tested on additions ranging from 1 to 8 digits. In both cases, NanoGPT exhibits an inability to perform addition on digits it has not been exposed to.} 
\label{fig:length_generalization}
\end{figure}

We further explore the impact of detailed scratchpad formatting. The model trained on additions of up to 3 digits, struggles to generalize to 4-digit additions. Notably, it randomly drops a single digit from the 4-digit number, erroneously perceiving it as a 3-digit number. We illustrate this difficulty in Figure~\ref{fig:length_generalization_examples} through multiple detailed error cases, ranging from instances in which only the test query is provided (Case 1) to scenarios where all intermediate steps are provided except only the final answer (Case 5). The prompts are highlighted in light grey and the responses generated by our trained NanoGPT model are highlighted in light green. These cases emphasize the model's shortcomings in accurately managing larger digit lengths.

\begin{AIbox}[breakable]{\bf{\large Examples for length generalization prompts}}
\vspace{5mm}
{\tt \footnotesize Results obtained by prompting the NanoGPT model with larger digits than those it was trained on. The model is trained using detailed scratchpad formats with 3-digit numbers. We evaluate its performance on 4-digit numbers, with varying levels of provided information. The prompt input is highlighted in a light blue box, while the model's output is highlighted in a light green box.
\\}

\textbf{Case 1: Just asking the question}
\begin{lstlisting}[language=markdown]
Input:
8465+3541
Target:
(*@\highlighttext{<scratch>}@*)
(*@\highlighttext{[8,4,6] has 3 digits.}@*)
(*@\highlighttext{[3,5,1] has 3 digits.}@*)
(*@\highlighttext{[8,4,6] + [3,5,1] , A=[] , C=0 , 6+1+0=7 , A->7 , C->0}@*)
(*@\highlighttext{[8,4] + [3,5] , A=[7] , C=0 , 4+5+0=9 , A->9 , C->0}@*)
(*@\highlighttext{[8] + [3] , A=[9,7] , C=0 , 8+3+0=11 , A->1 , C->1}@*)
(*@\highlighttext{[] + [] , A=[1,9,7] C=1 , END}@*)
(*@\highlighttext{</scratch>}@*)
(*@\highlighttext{1 1 9 7}@*)
\end{lstlisting}

\textbf{Case 2: Giving number of digits}
\begin{lstlisting}[language=markdown]
Input:
8465+3541
Target:
<scratch>
[8,4,6,5] has 4 digits.
[3,5,4,1] has 4 digits.
(*@\highlighttext{[8,4,6] + [3,5,4] , A=[] , C=0 , 6+4+0=10 , A->0 , C->1}@*)
(*@\highlighttext{[8,4] + [3,5] , A=[0] , C=1 , 4+5+1=10 , A->0 , C->1}@*)
(*@\highlighttext{[8] + [3] , A=[0,0] , C=1 , 8+3+1=12 , A->2 , C->1}@*)
(*@\highlighttext{[] + [] , A=[2,0,0] C=1 , END}@*)
(*@\highlighttext{</scratch>}@*)
(*@\highlighttext{1 2 0 0}@*)
\end{lstlisting}

\textbf{Case 3: Giving 1 intermediate steps}
\begin{lstlisting}[language=markdown]
Input:
8465+3541
Target:
<scratch>
[8,4,6,5] has 4 digits.
[3,5,4,1] has 4 digits.
[8,4,6,5] + [3,5,4,1] , A=[] , C=0 , 5+1+0=6 , A->6 , C->0
(*@\highlighttext{[8,4] + [3,5] , A=[6] , C=0 , 4+5+0=9 , A->9 , C->0}@*)
(*@\highlighttext{[8] + [3] , A=[9,6] , C=0 , 8+3+0=11 , A->1 , C->1}@*)
(*@\highlighttext{[] + [] , A=[1,9,6] C=1 , END}@*)
(*@\highlighttext{</scratch>}@*)
(*@\highlighttext{1 1 9 6}@*)
\end{lstlisting}

\textbf{Case 4: Giving all but one intermediate steps}
\begin{lstlisting}[language=markdown]
Input:
8465+3541
Target:
<scratch>
[8,4,6,5] has 4 digits.
[3,5,4,1] has 4 digits.
[8,4,6,5] + [3,5,4,1] , A=[] , C=0 , 5+1+0=6 , A->6 , C->0
[8,4,6] + [3,5,4] , A=[6] , C=0 , 6+4+0=10 , A->0 , C->1
[8,4] + [3,5] , A=[0,6] , C=1 , 4+5+1=10 , A->0 , C->1
[8] + [3] , A=[0,0,6] , C=1 , 8+3+1=12 , A->2 , C->1
(*@\highlighttext{[] + [] , A=[2,0,6] C=1 END}@*)
(*@\highlighttext{</scratch>}@*)
(*@\highlighttext{1 0 0 6}@*)
\end{lstlisting}

\textbf{Case 5: Giving all intermediate steps, all it has to do is aggregate the answer}
\begin{lstlisting}[language=markdown]
Input:
8465+3541
Target:
<scratch>
[8,4,6,5] has 4 digits.
[3,5,4,1] has 4 digits.
[8,4,6,5] + [3,5,4,1] , A=[] , C=0 , 5+1+0=6 , A->6 , C->0
[8,4,6] + [3,5,4] , A=[6] , C=0 , 6+4+0=10 , A->0 , C->1
[8,4] + [3,5] , A=[0,6] , C=1 , 4+5+1=10 , A->0 , C->1
[8] + [3] , A=[0,0,6] , C=1 , 8+3+1=12 , A->2 , C->1
[] + [] , A=[2,0,0,6] C=1 END
</scratch>
(*@\highlighttext{1 2 0 6}@*)
\end{lstlisting}

\textbf{Case 6: Giving 1 in-context example}
\begin{lstlisting}[language=markdown]
Input:
8465+3541
Target:
<scratch>
[8,4,6,5] has 4 digits.
[3,5,4,1] has 4 digits.
[8,4,6,5] + [3,5,4,1] , A=[] , C=0 , 5+1+0=6 , A->6 , C->0
[8,4,6] + [3,5,4] , A=[6] , C=0 , 6+4+0=10 , A->0 , C->1
[8,4] + [3,5] , A=[0,6] , C=1 , 4+5+1=10 , A->0 , C->1
[8] + [3] , A=[0,0,6] , C=1 , 8+3+1=12 , A->2 , C->1
[] + [] , A=[2,0,0,6] C=1 , END
</scratch>
1 2 0 0 6
Input:
1946+3598
Target:
(*@\highlighttext{<scratch>}@*)
(*@\highlighttext{[1,9,4] has 3 digits.}@*)
(*@\highlighttext{[3,5,8] has 3 digits.}@*)
(*@\highlighttext{[1,9,4] + [3,5,8] , A=[] , C=0 , 4+8+0=12 , A->2 , C->1}@*)
(*@\highlighttext{[1,9] + [3,5] , A=[2] , C=1 , 9+5+1=15 , A->5 , C->1}@*)
(*@\highlighttext{[1] + [3] , A=[5,2] , C=1 , 1+3+1=5 , A->5 , C->0}@*)
(*@\highlighttext{[] + [] , A=[5,5,2] C=0 , END}@*)
(*@\highlighttext{</scratch>}@*)
(*@\highlighttext{5 5 2}@*)
\end{lstlisting}

\textbf{Case 7: Giving 1 In-context example, and all intermediate steps}
\begin{lstlisting}[language=markdown]
Input:
8465+3541
Target:
<scratch>
[8,4,6,5] has 4 digits.
[3,5,4,1] has 4 digits.
[8,4,6,5] + [3,5,4,1] , A=[] , C=0 , 5+1+0=6 , A->6 , C->0
[8,4,6] + [3,5,4] , A=[6] , C=0 , 6+4+0=10 , A->0 , C->1
[8,4] + [3,5] , A=[0,6] , C=1 , 4+5+1=10 , A->0 , C->1
[8] + [3] , A=[0,0,6] , C=1 , 8+3+1=12 , A->2 , C->1
[] + [] , A=[2,0,0,6] C=1 , END
</scratch>
1 2 0 0 6
Input:
1946+3598
Target:
<scratch>
[1,9,4,6] has 4 digits.
[3,5,9,8] has 4 digits.
[1,9,4,6] + [3,5,9,8] , A=[] , C=0 , 6+8+0=14 , A->4 , C->1
[1,9,4] + [3,5,9] , A=[4] , C=1 , 4+9+1=14 , A->4 , C->1
[1,9] + [3,5] , A=[4,4] , C=1 , 9+5+1=15 , A->5 , C->1
[1] + [3] , A=[5,4,4] , C=1 , 1+3+1=5 , A->5 , C->0
[] + [] , A=[5,5,4,4] C=0 , END
</scratch>
(*@\highlighttext{5 5 4}@*)
\end{lstlisting}

\end{AIbox}
\noindent\begin{minipage}{\textwidth}
\captionsetup{type=figure}
\captionof{figure}{Example results on the model's output when prompted with a larger number of digits than those it was trained on.}\label{fig:length_generalization_examples}
\end{minipage}

%% file: _Discussion.tex
\section{Limitations}\label{sec:discussion}

\textbf{Length generalization. } In our experiments, we did not observe any instances where the model could predict beyond the number of digits it had been trained on (see Section~\ref{sec:appendix_lengthgeneralization}). This finding is consistent with previous literature that suggests length generalization is a challenging task. For instance, \citet{shaw2018self, sun2022length} reported similar difficulties and proposed approaches such as relative positional encodings. \citet{anil2022exploring} suggests that models can only perform out-of-distribution tasks by combining fine-tuning, prompting, and scratchpad techniques. Nonetheless, there have been cases where length generalization was observed. \citet{nye2021show} demonstrated length generalization but only for models with more than $10^8$ parameters.

\textbf{Model/Data scale. } Due to the smaller scale of our experiments, we were able to thoroughly examine the impact of individual components on the model's arithmetic learning capabilities. Our model was limited to a GPT-type decoder-only architecture, primarily focusing on character-level tokenization. Although we have obtained some preliminary results on scaling up and incorporating BPE-based tokenization, it remains uncertain if all our findings can be generalized to the scale of LLMs being used in practice today.

\textbf{Beyond elementary arithmetic. } We choose to analyze simple arithmetic operations in order to carefully isolate factors that contribute to emergence. While the existing literature has already demonstrated the emergence of complicated abilities in practice, our work seeks to provide a better understanding of this behavior.

%% file: _Conclusion.tex
\section{Conclusion}\label{sec:conclusion}

In this work, we examine the problem of teaching small randomly initialized transformers arithmetic operations and elementary mathematical functions using the next-token prediction objective. We carefully ablate different aspects of the training data so as to isolate the factors that contribute to the emergence of arithmetic capabilities. Our results reveal that traditional training data is sub-optimal for learning arithmetic, and training on detailed, instructive data with intermediate steps or even simply reversing the output improves accuracy and sample complexity. We consider both scenarios with only arithmetic data as well as those with text data, and comprehensively analyze the effect of few-shot prompting, pretraining, and model scale. We find that while detailed, chain-of-thought style data improves sample complexity, it may not be efficient in terms of training and inference costs since it requires training with much more tokens. Furthermore, we find that while the model generalizes to unseen examples of the same number of digits, the problem of length generalization is quite difficult. We attribute this to the model's inability to truly ``learn'' the underlying arithmetic operation in all generality. It remains an open problem how to curate the training data to ensure that the model learns a particular algorithm as opposed to just learning an approximate function map. It is also unclear what the correct way to learn multiple operations is. It seems plausible that learning them in increasing order of complexity is beneficial if one can circumvent the problem of catastrophic forgetting.
Our findings emphasize the significance of high-quality, instructive data for the emergence of arithmetic capabilities in transformers. We anticipate this research will contribute to a more nuanced understanding of the mechanisms by which transformers acquire arithmetic operations. %

%% file: _Appendix_New_Exp.tex
\newpage
\section{Proofs}
Here, we present the proofs of Lemma~\ref{lemma:left-to-right} and \ref{lemma:right-to-left}.

\lemmaOne*

\begin{proof}
We begin by assuming for contradiction that there does exist an algorithm $\texttt{Algo}$ that does not have access to all digits of $A$ and $B$ and still outputs $C=A+B$ correctly for all $n-$ digit numbers $A, B$. Without loss of generality, say $\texttt{Algo}$ does not have access to the $k-$th digit of $A$ where $k \in [n]$ represents the position counting from the least significant digit. Then consider the example $B=(10^n - 1)$ and $(A=000\dots A_k 00 \dots 0)$ where $B$ is just the integer with $n$ $9$'s and $A$ is just $0$'s with $A_k$ in the $k$th position. If $A_k=0$, then $C_{n+1}=0$, but if $A_k=1$, then $C_{n+1}=1$. Therefore, without access to the $k-$th digit of $A$, there exist examples where the algorithm will surely make a mistake. Therefore, by contradiction such an \texttt{Algo} cannot exist.
\end{proof}

\lemmaTwo*
\begin{proof}
First note that the trivial algorithm for addition is exactly the proof of this Lemma. However, we present a more formal argument below for completeness. Let $A$, $B$, and $C$ be $n-$digit numbers such that $C = A + B$. Define the digits of $A, B$, and $C$ as $A_i$, $B_i$, and $C_i$, respectively, for $i \in [n]$ counting from the least significant digit once again. Then, the addition can be performed using the following steps. First, $C_i = (A_i + B_i + carry_i)\mod{10}$ where $carry_i$ is the carry-on from the addition of digits at position $i-1$. If there is no carry from the previous position, then $carry_i = 0$. The carry for the next position is then calculated as $carry_{i+1} = \left\lfloor \frac{A_i + B_i + carry_i}{10} \right\rfloor$.

Putting this together, the algorithm for addition can be described as follows:

\textbf{Step 1:} Set $carry_1 = 0$. \textbf{Repeat} for $i=\{1, \dots, n\}$: \{\textbf{Step 2:} Compute $C_i = (A_i + B_i + carry_i) \mod 10$ and $carry_{i+1} = \left\lfloor \frac{A_i + B_i + carry_i}{10} \right\rfloor$, \textbf{Step 3:} Output $C_i$\}.

It is easy to see that this algorithm computes the digits of the sum $C$ correctly and requires only the individual digits at position $i$ and the carry from the previous position. Therefore, this algorithm satisfies the conditions of the lemma.
\end{proof}

\newpage
\section{Additional Experiments}

\subsection{Zero-Padding and Symbol Wrapping}\label{sec:appdx_zeropad_delimiter}

As discussed briefly in Section~\ref{sec:exp}, we found a significant benefit to using padding for multi-digit addition. Throughout our experiments, we use the plain format without any such padding (denoted as ``vanilla'' below) as the default baseline representing the conventional data format used in training.
Nonetheless, we explore modifications to this plain format to enhance performance; zero-padding, and wrapping with a single symbol. Zero-padding ensures a fixed length for operands and the output. In the case of $3$-digit addition, this means $3$-digit operands and a $4$-digit output.  For example, `$\mathsf{112+29=141}$' becomes `$\mathsf{112+029=0141}$'. As shown in Table~\ref{tab:zeropad_dollar}. this modification significantly improves model performance.
Next, we wrap each sample using the `\$' symbol as in '$\mathsf{\$112+29=141 \$}$'. We found this performs on par with zero-padding.

As a result, we adopt the `\$' symbol for efficient data delimiter, extending its use to the reverse format.  Figure~\ref{fig:dollar_pad} shows `\$'-wrapping also enhances the performance of the reverse format.  Despite the plain format being improved with the `\$' delimiter, it remains short of the reverse format's accuracy and sample efficiency.
We continue to maintain the original plain format as a baseline since it not only exemplifies conventional data but further emphasizes the need for improved data formatting to ensure efficient training. As such, for the reverse format, we have incorporated the `\$' delimiter in our formatting modifications. 

\begin{table}[th!]
\center
  \caption{
Test accuracy of NanoGPT model on 3-digit addition trained on $10,000$ samples of plain format data, comparing (i) vanilla format without modifications, (ii) Zero-padding format, and (iii) `\$'-wrapped format. The results show significant performance enhancement through zero-padding for fixed length and similar improvements when deploying a single-symbol wrapping.
}
  \label{tab:zeropad_dollar}
  \vspace{1mm}
\centering
\begin{tabular}{cccc}
\toprule
 Vanilla & Zero-pad  & `\$'-Wrapped \\ \midrule
 88.17\%                       &  97.74\%  &   97.76\%      \\ \bottomrule
\end{tabular}
\end{table}

\begin{figure}[ht]
\centering
\includegraphics[width=0.45\textwidth]{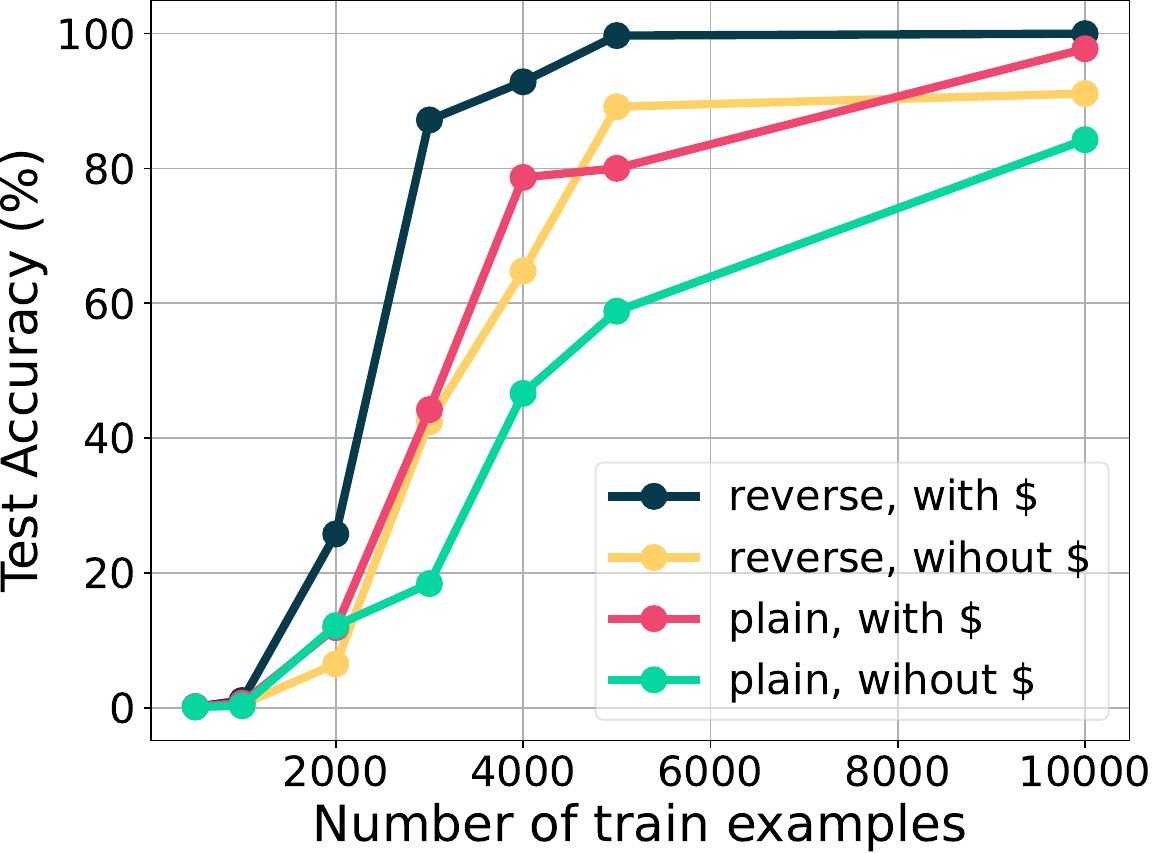}
    \caption{Performance of NanoGPT model on $3$-digit addition using plain and reverse format, both with and without `\$' delimiter.  The addition of the `\$' symbol noticeably enhances performance in both formats. Nevertheless, the plain format underperforms compared to the reverse format, particularly in terms of sample efficiency. While we maintain the original plain format as a baseline -- emphasizing the necessity for improved data formatting for efficient emergence -- we incorporate the `\$' wrapping in our modified reverse format.  
}
\label{fig:dollar_pad}
\end{figure}

\newpage
\subsection{Low-Rank Matrix Completion}\label{sec:appdx_lrmc}
In our Low-Rank Matrix Completion experiment for the addition matrix (which is of rank-2), we employ an iterative algorithm proposed by \citet{kiraly2015algebraic}. This algorithm systematically searches for a $2 \times 2$ submatrix in which three entries are known and one entry is unknown. It then fills the unknown entry to ensure that the determinant of the $2 \times 2$ submatrix becomes zero, where the solution is known to be optimal. We present the full pseudo-code in Algorithm~\ref{alg:iterative_2x2_mc}.

To assess the performance of the algorithm, we generate $n \times n$ addition matrices for various values of $n$ (e.g., 20, 50, 100, 500). We vary the number of revealed entries, randomly sampling a sparse matrix where only a specified number of entries between $n$ and $n \times n$ are known, while the remaining entries are set to zero. We repeat this process 100 times for each number of revealed entries, tracking the algorithm's success or failure in finding the solution. We calculate the average success rate across the trials and present the success probabilities in Figure~\ref{fig:matrix_completion}, where we observe a sharp phase transition when $\gO(n)$ entries are observed, as expected.

\SetKwComment{Comment}{/* }{ */}
\SetKw{Continue}{continue}

\begin{algorithm}
\footnotesize
\caption{Iterative $2\times2$ Matrix Completion Algorithm}\label{alg:iterative_2x2_mc}
\DontPrintSemicolon
\KwData{Data Matrix $\mM \in \mathbb{R}^{n\times n}$ with partially revealed entries. Assumed to be of Rank $2$.}
\KwResult{$\widehat{\mM} \in \mathbb{R}^{n\times n}$, \texttt{Success/Fail}.}
$n_1 \gets 1$ represents number of resolved submatrices.\;
$n_2 \gets 0$ represents number of unresolved submatrices.\;
$\widehat{\mM} \gets \mM$\;
\While{$n_1 \ge 1$}{ 
\Comment*[l]{As long as we resolved at least one submatrix in the previous iteration}
$n_1 \gets 0$\;
$n_2 \gets 0$\;
\For{$i=1$ \KwTo $n$}{
    \For{$j=1$ \KwTo $n$}{
        \Comment*[l]{do something}\;
        \If{$\widehat{\mM}_{i,j}$ is not revealed \textbf{and} all its neighbors are revealed}{
            $\widehat{\mM}_{i, j} = \frac{\widehat{\mM}_{i+1, j} \times \widehat{\mM}_{i, j+1}}{\widehat{\mM}_{i+1, j+1}}$\;
            $n_1 \gets n_1 + 1$\;
            }
        \If{$\widehat{\mM}_{i+1,j}$ is not revealed \textbf{and} all its neighbors are revealed}{
         $\widehat{\mM}_{i+1, j} = \frac{\widehat{\mM}_{i, j} \times \widehat{\mM}_{i+1, j+1}}{\widehat{\mM}_{i+1, j}}$\;
         $n_1 \gets n_1 + 1$\;
        }
        \If{$\widehat{\mM}_{i+1,j+1}$ is not revealed \textbf{and} all its neighbors are revealed}{
         $\widehat{\mM}_{i+1, j+1} = \frac{\widehat{\mM}_{i+1, j} \times \widehat{\mM}_{i, j+1}}{\widehat{\mM}_{i, j}}$\;
         $n_1 \gets n_1 + 1$\;
        }
        \If{$\widehat{\mM}_{i,j}, \widehat{\mM}_{i+1,j}, \widehat{\mM}_{i,j+1}, \widehat{\mM}_{i+1,j+1}$ are all revealed}{
         \Continue
        }
        \Else{
        $n_2 \gets n_2 + 1$\;
        }
        }
    }
}
\If{$n_2 > 0$}{
    \Return $\widehat{\mM}$, \texttt{Fail} \;
}
\Else{
    \Return $\widehat{\mM}$, \texttt{Success} \;
}
\end{algorithm}

\newpage
\subsection{Prompting with Text}\label{sec:appendix_text_prmopt}
To extend on the few-shot prompting experiments from Section~\ref{sec:joint_arithmetic}, we also evaluate the effect of prompting the model with pure-text prompts. If few-shot prompting with addition samples improves accuracy through in-context learning, we expect few-shot prompting with text to hurt accuracy since the text exemplars are \emph{out-of-context}.
We use five different types of text exemplars: (i) \textbf{Prompt1:} a short text prompt that is not present in the Shakespeare dataset, (ii) \textbf{Prompt2:} a short text prompt extracted from within Shakespeare dataset, (iii) \textbf{Prompt3:} a longer form text prompt extracted from within the Shakespeare dataset, (iv) \textbf{Prompt4:} a prompt that includes numbers, and (v) \textbf{Prompt5:} a long text prompt that is not present in the Shakespeare dataset. More details on the text prompts can be found in Figure~\ref{fig:text_exemplars}.

\begin{AIbox}[breakable]{\bf{\large Text prompts for few-shot experiments}}
\vspace{5mm}
{\tt \footnotesize Examples of the different text prompts used in the few-shot experiment. Each exemplar is separated by `-{}-{}-'. %
\\}

\begin{minipage}[t]{0.45\linewidth}
\centering
\textbf{Prompt 1. Short, $\notin$ Shakespeare}
\begin{lstlisting}[language=markdown]
et tu brute
---
hello, world
---
how are you doing?
---
agi is coming
---
boom! stability
\end{lstlisting}

\textbf{Prompt 2. Short, $\in$ Shakespeare}
\begin{lstlisting}[language=markdown]
JULIET:
Romeo!
---
All:
Resolved. resolved.
---
VOLUMNIA:
Why, I pray you?
---
CORIOLANUS:
Nay! prithee, woman,--
---
MENENIUS:
I mean, thy general.
\end{lstlisting}
\end{minipage}
\begin{minipage}[t]{0.55\linewidth}
\centering
\textbf{Prompt 3. Long, $\in$ Shakespeare}
\begin{lstlisting}[language=markdown]
JULIET:
Romeo!
ROMEO:
My dear?
---
MENENIUS:
This is good news:
I will go meet the ladies. This Volumnia
Is worth of consuls, senators, patricians,
---
LADY ANNE:
Foul devil, for God's sake, hence, and trouble us not;
For thou hast made the happy earth thy hell,
Fill'd it with cursing cries and deep exclaims.
---
BUCKINGHAM:
I fear he will.
How now, Catesby, what says your lord?
---
CATESBY:
Bad news, my lord: Ely is fled to Richmond;
And Buckingham, back'd with the hardy Welshmen,
Is in the field, and still his power increaseth.
\end{lstlisting}

\end{minipage}

\centering
\textbf{Prompt 4. Has number, $\notin$ Shakespeare}
\begin{lstlisting}[language=markdown]
I go 16-12
That's the code to my heart, ah
I go 1-6-1-2
Star
---
Like a river flows 17-23
Surely to the sea 15-22
Darling, so it goes 46-92
Some things are meant to be
---
I got my first real 6-string
Bought it at the five and dime
Played it 'til my fingers bled
Was the summer of '69
---
I think someday I might just 5-3-2-1 get a real job
I spent half of my life 1-2-3 in a bus or on a flight
I'm getting off 17-36-8-2 the road and in a real job
---
Every time that 27-67-29 I look in the mirror
All these lines on my 1-3-92-5 face getting clearer
The past 45-5-3 is gone
\end{lstlisting}

\textbf{Prompt 5. Long, $\notin$ Shakespeare}
\begin{lstlisting}[language=markdown]
Is this the real life? Is this just fantasy? Caught in a landside, no escape from reality.
Open your eyes, look up to the skies and see.
I'm just a poor boy, I need no sympathy. Because I'm easy come, easy go,
Little high, little low,
Any way the wind blows doesn't really matter to me, to me.
---
It's my life
And it's now or never
I ain't gonna live forever
I just want to live while I'm alive
My heart is like an open highway
Like Frankie said, I did it my way
---
Destruction leads to a very rough road but it also breeds creation
And earthquakes are to a girl's guitar, they're just another good vibration
And tidal waves couldn't save the world from Californication
---
I want to stay
But I need to go
I want to be the best for you
But I just don't know what to do
'Cause baby, say I've cried for you
The time we have spent together
Riding through this English whether
---
Lorem ipsum dolor sit amet, consectetur adipiscing elit. Vestibulum mattis in leo vel gravida.
Pellentesque libero elit, scelerisque varius vehicula a, hendrerit et tellus.
Proin convallis neque nisl, nec lobortis est scelerisque tincidunt. 
Nunc venenatis auctor urna.
Class aptent taciti sociosqu ad litora torquent per conubia nostra.
\end{lstlisting}

\end{AIbox}

\noindent\begin{minipage}{\textwidth}
\captionsetup{type=figure}
\captionof{figure}{Text prompt exemplars for few-shot experiments.}\label{fig:text_exemplars}
\end{minipage}

\begin{figure}[ht] 
\centering
\vspace{-4mm}
\subfloat[\scriptsize{NanoGPT, Test accuracy on plain addition}]{\includegraphics[width=0.38\textwidth]{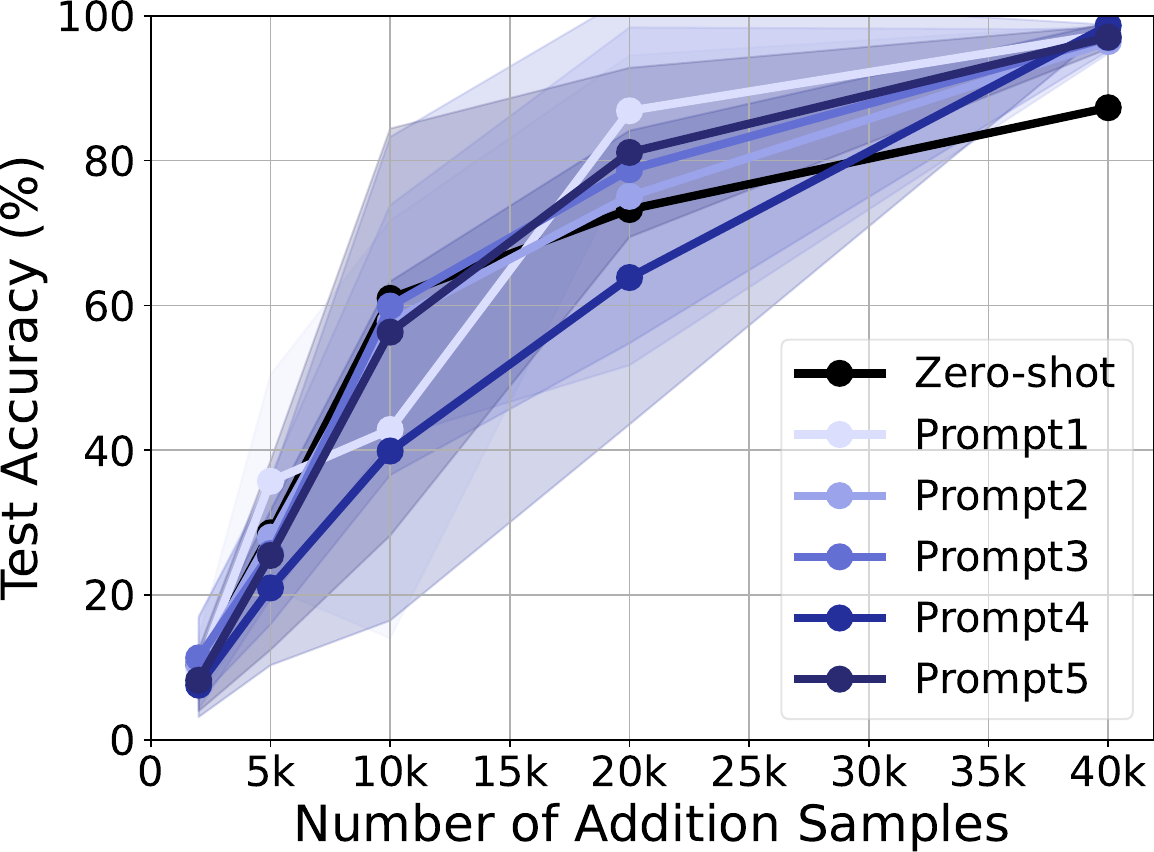}}
\hspace{0.5cm}
\subfloat[\scriptsize{GPT-2, Test accuracy on plain addition}]{\includegraphics[width=0.38\textwidth]{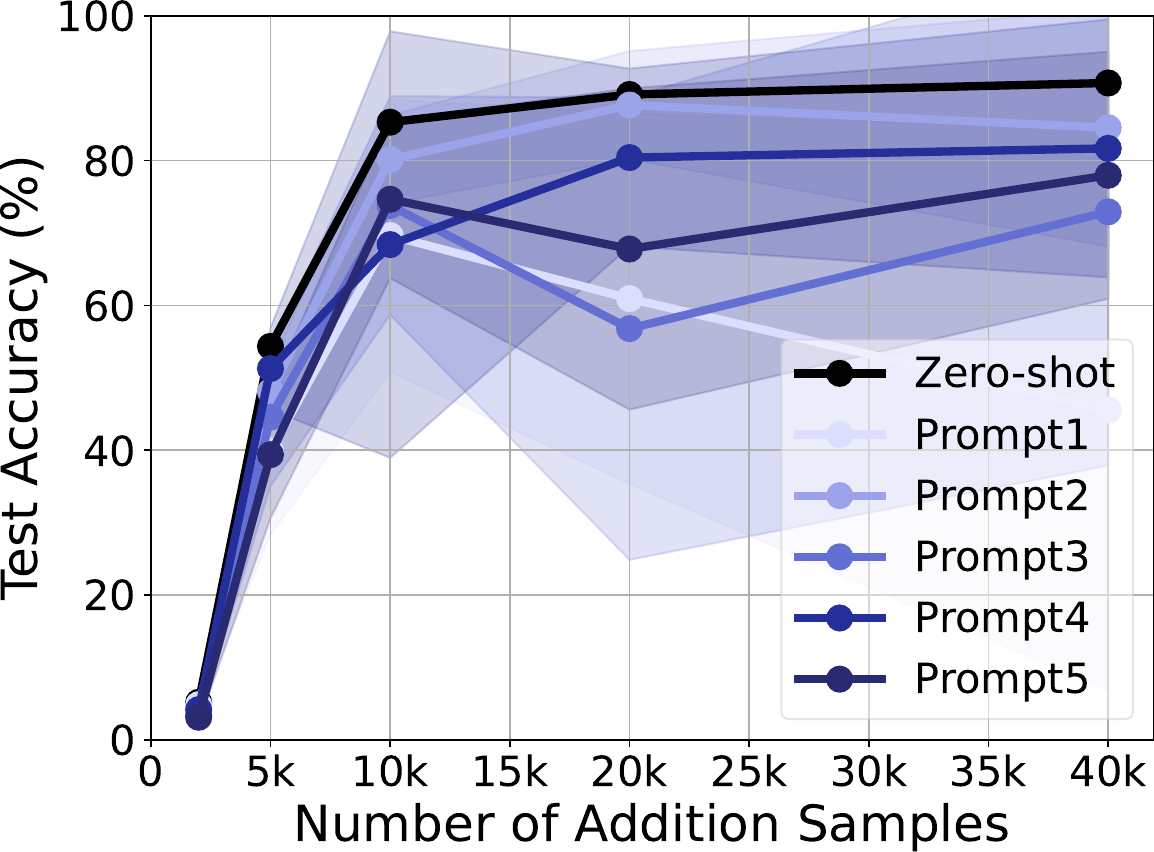}}\\
\vspace{-0.1mm}
\subfloat[\scriptsize{NanoGPT, Test accuracy on detailed scratchpad}]{\includegraphics[width=0.38\textwidth]{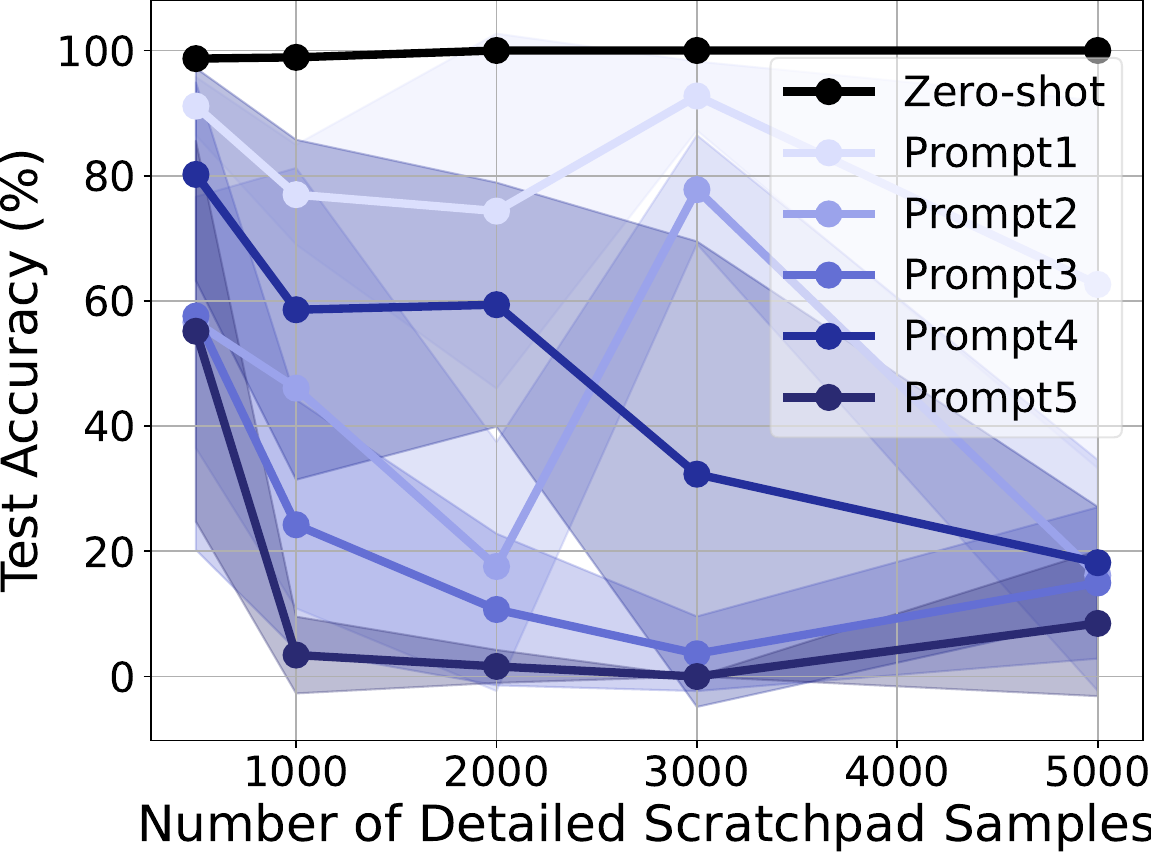}}
\hspace{0.5cm}
\subfloat[\scriptsize{GPT-2, Test accuracy on detailed scratchpad}]{\includegraphics[width=0.38\textwidth]{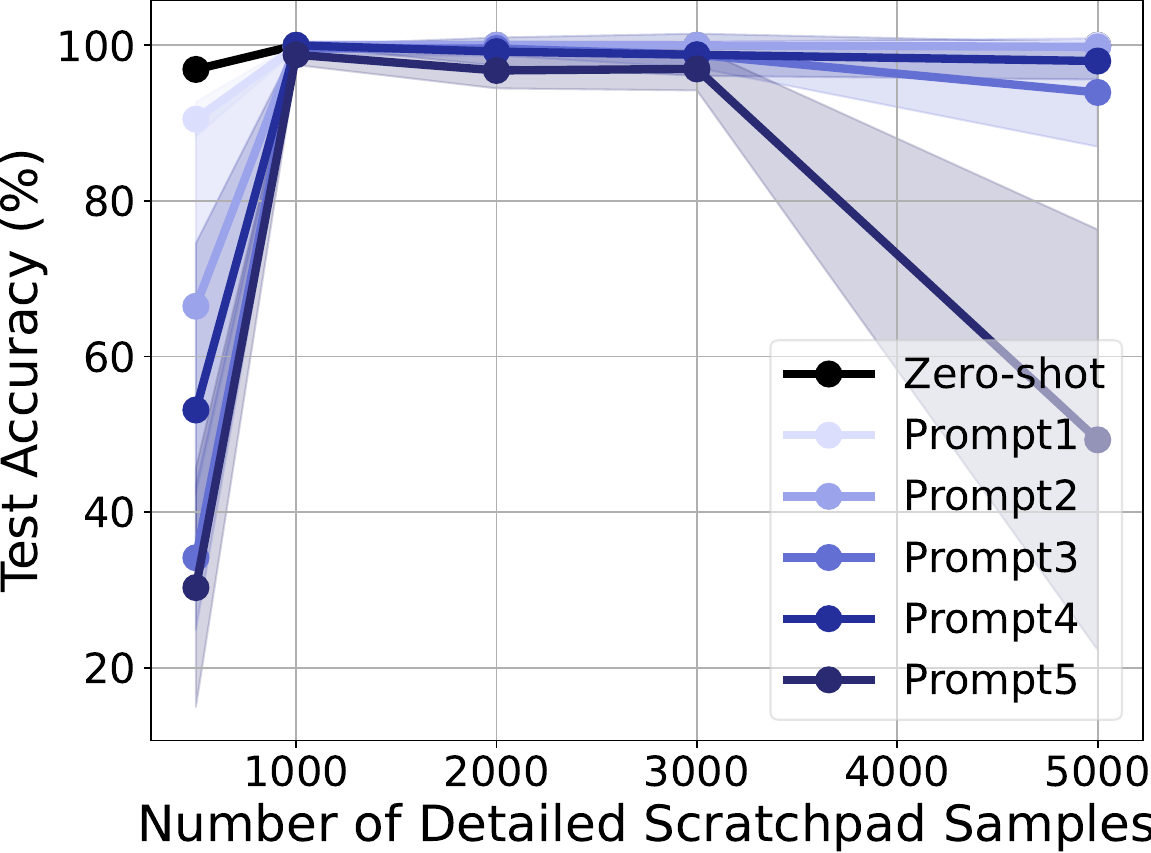}}
    \caption{Experiments on few-shot prompting with different text prompts: (i) Prompt1: short text not in Shakespeare dataset (ii) Prompt2: short text within Shakespeare dataset (iii) Prompt3: long text within Shakespeare dataset (iv) Prompt4: text with numbers (v) Prompt5: long text not in the Shakespeare dataset. Each prompt (Prompt 1-5) consists of five distinct exemplars. The solid lines represent the mean performance across the five exemplars, while the shaded area indicates the standard deviation. We observe that the effectiveness of text prompts varies greatly depending on the exemplars used. } 
\label{fig:mixed_text_prompt_performance}
\vspace{-2mm}
\end{figure}

The results presented in Figure~\ref{fig:mixed_text_prompt_performance} show notable variations in evaluation accuracy for addition, depending on the chosen text prompts. Longer text prompts (Prompt 5) typically result in a more significant decline in performance. With the exception of NanoGPT trained on plain addition, the result in Figure~\ref{fig:mixed_performance} indicates that employing text prompts followed by test addition queries tends to have an adverse impact on the overall model performance, whereas incorporating relevant few-shot exemplars (1/2/3-shot) is beneficial. This aligns well with our intuition on the benefits on in-context learning. %

\begin{figure}[ht] 
\centering
\subfloat[\scriptsize{NanoGPT, Test accuracy on plain addition}]{\includegraphics[width=0.38\textwidth]{fig/few_shot/nanogpt_mixed_add_accuracy2.pdf}}
\hspace{0.5cm}
\subfloat[\scriptsize{GPT-2, Test accuracy on plain addition}]{\includegraphics[width=0.38\textwidth]{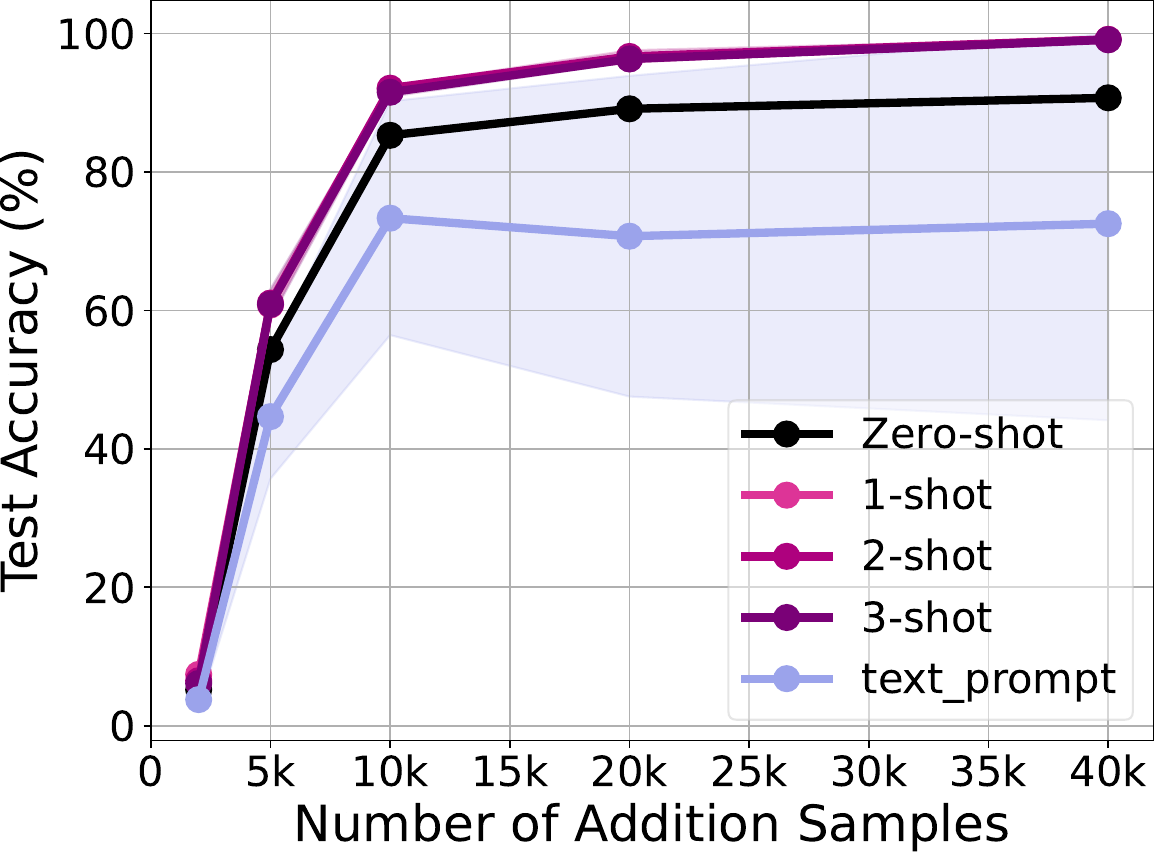}}\\
\vspace{-0.1mm}
\subfloat[\scriptsize{NanoGPT, Test accuracy on detailed scratchpad}]{\includegraphics[width=0.38\textwidth]{fig/few_shot/nanogpt_mixed_ar_accuracy2.pdf}}
\hspace{0.5cm}
\subfloat[\scriptsize{GPT-2, Test accuracy on detailed scratchpad}]{\includegraphics[width=0.38\textwidth]{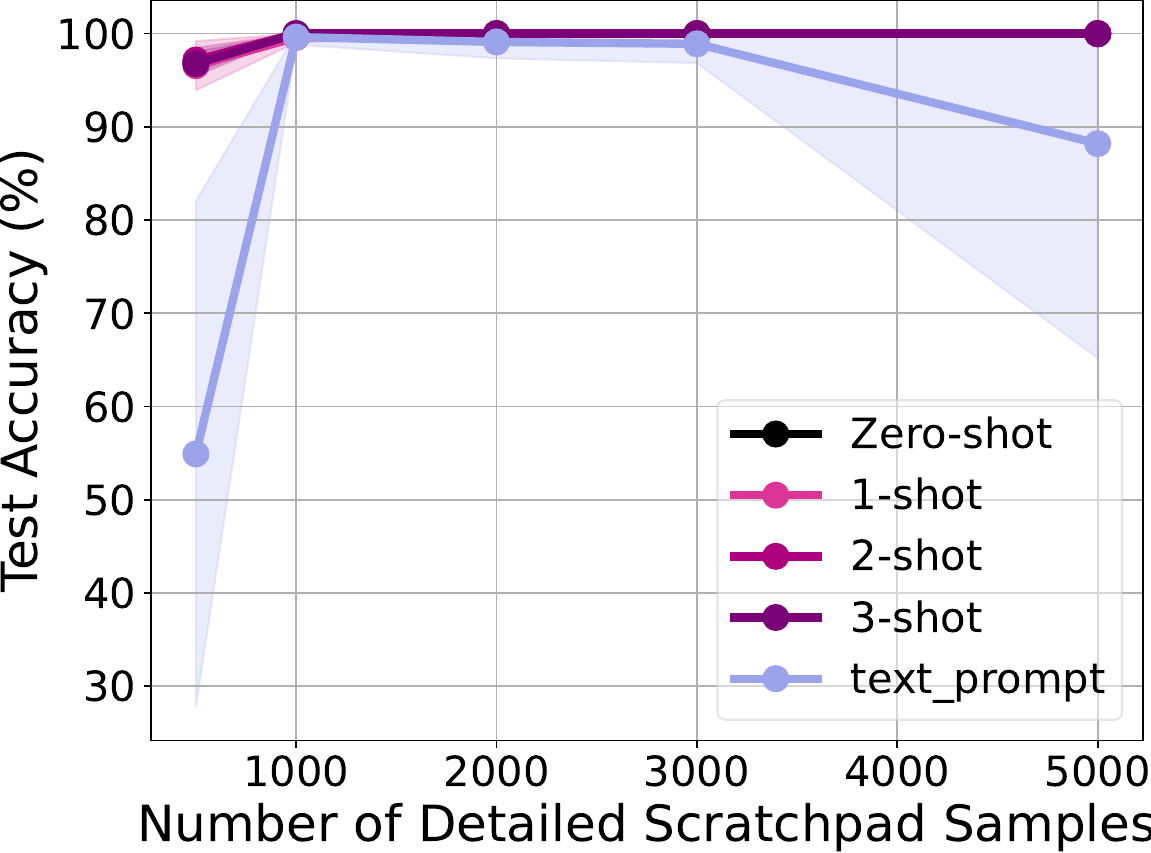}}
    \caption{Performance of NanoGPT and GPT-2 model trained with entire Shakespeare dataset and a varying number of samples of plain addition, and addition with detailed scratchpad dataset. Performance is evaluated on test prompts formatted as plain addition and detailed scratchpad. Few-shot experiments are based on an average of 5 exemplars, while text prompts involve an average of 25 exemplars. The shaded area represents the standard deviation. Our observations indicate that few-shot prompting consistently improves performance, whereas test prompts generally have a negative impact.} 
\label{fig:mixed_performance}
\end{figure}

\newpage
\subsection{Analyzing the results on Sine/Sqrt}\label{sec:sin_sqrt_analysis}
Since sine and sqrt are arguably more complicated functions than the remaining arithmetic tasks, we decided to more carefully analyze their performance. As shown in Figure~\ref{fig:analyzing_sin_sqrt}, $\sin$ shows excellent performance across all data formats around $\sin(x) = 0$. We conjecture that this is because $\sin(x) \approx x$ for $x \approx 0$, which is easy to learn. We also note that accuracy once again improves close to $\pm1$ potentially for similar reasons.

\begin{figure}[ht] 
\centering
\vspace{-6mm}
\subfloat[{Test accuracy on Sine}]{\includegraphics[width=0.40\textwidth]{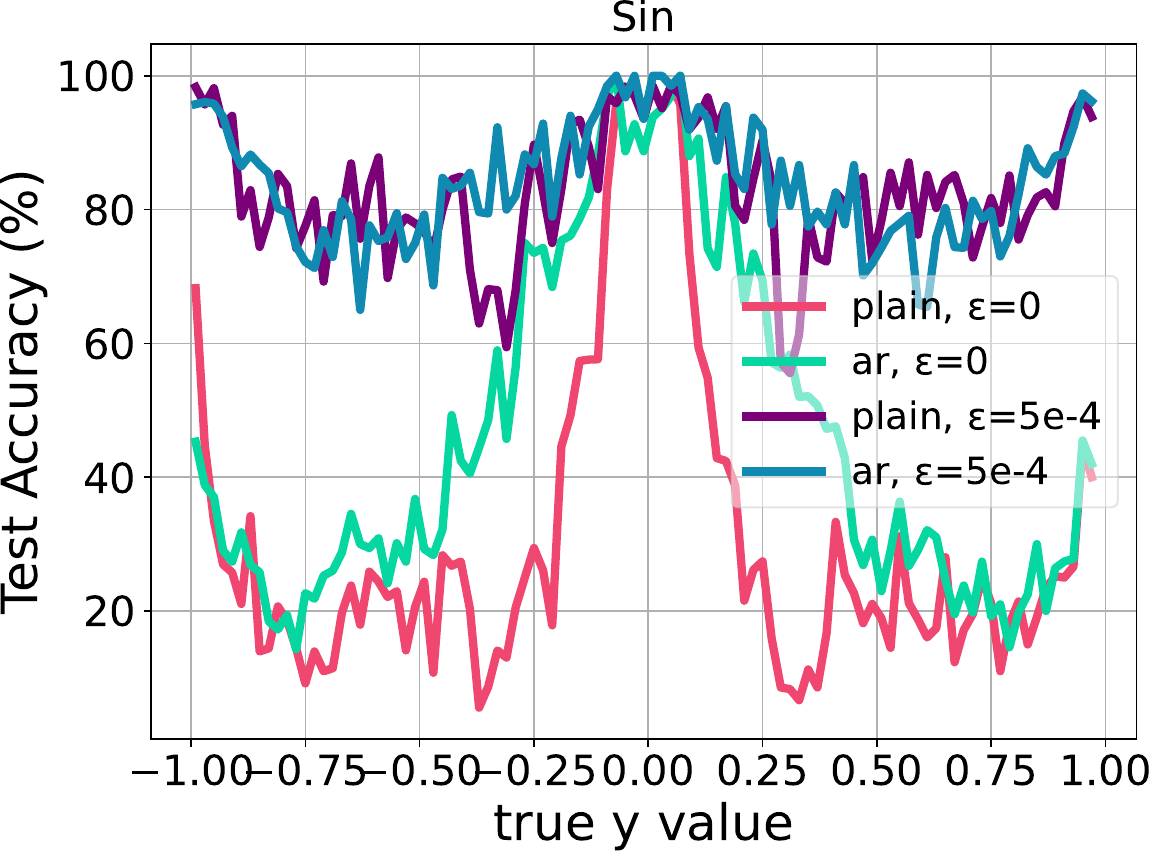}}
\hspace{0.5cm}
\subfloat[{Test accuracy on Square root}]{\includegraphics[width=0.40\textwidth]{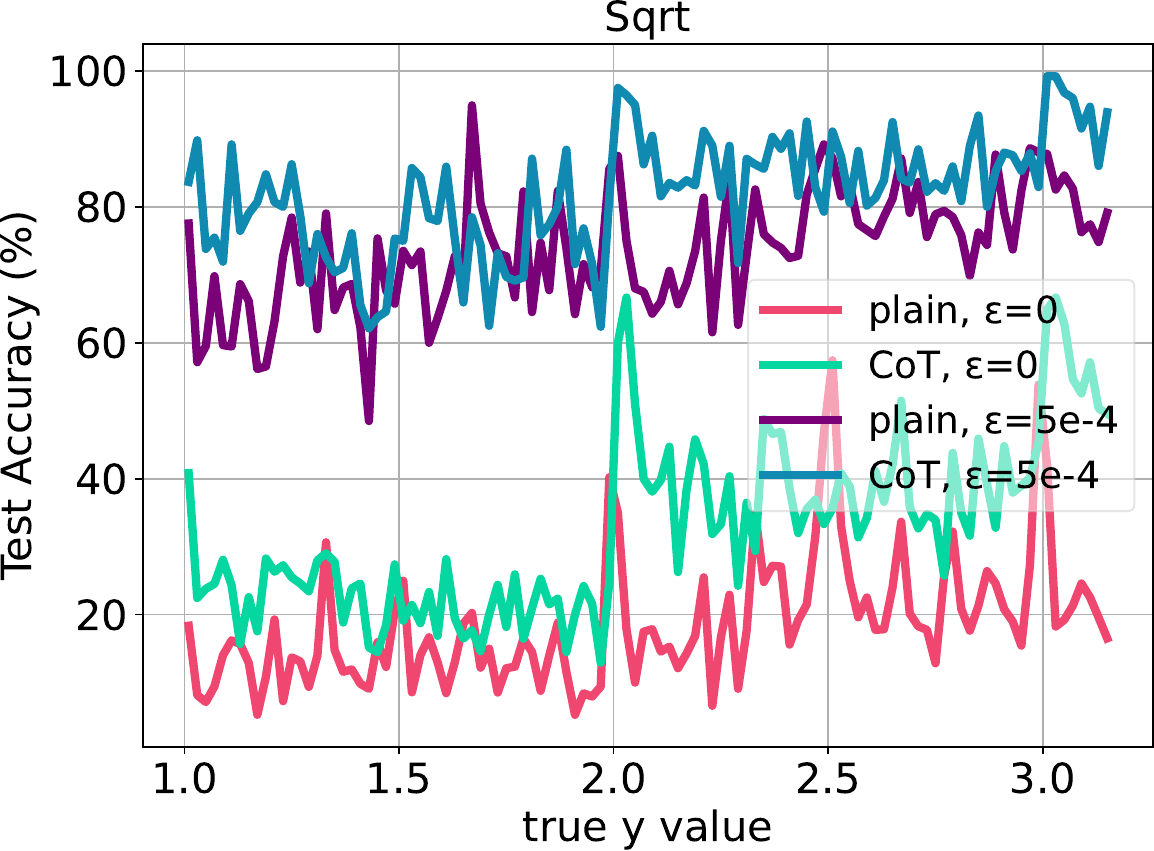}}
\vspace{0.1mm}
    \caption{Error analysis of sine and square root functions, considering varying error tolerance (eps) thresholds to determine correct output. The sine function demonstrates excellent performance across all data formats, particularly around $\sin(x) = 0$, where $\sin(x) \approx x$ for $x \approx 0$. Additionally, we observe improved accuracy near $\pm1$.} 
\label{fig:analyzing_sin_sqrt}
\vspace{-2mm}
\end{figure}

%% file: _Appendix_Setup.tex
\section{Experimental Setup}\label{sec:exp_setup}
In this section, we summarize the datasets, models and hyperparameters used for experiments. All of our experiments on NanoGPT and GPT-2 models are run using PyTorch 2.1 and CUDA 11.7 on Nvidia 2808 TIs and NVIDIA 3090s. Detailed dependencies are provided on our github repository\footnote{\url{https://github.com/lee-ny/teaching_arithmetic}}.

\subsection{Dataset}\label{appdx:dataset}
In this section, we explain the details of the datasets used for our experiments. For arithmetic tasks, we construct our own datasets as described below while we use the standard shakespeare~\citep{shakespeare} dataset for text.

\paragraph{Arithmetic Tasks} 
As mentioned above, for all arithmetic tasks, we prepare our own datasets. We refer to the training dataset for a binary operator $f(\cdot)$ as $\gD_\text{train}=\{(x^{1}_i, x^{2}_i), y_i\}_{i=1}^{N}$ where $y_i = f(x^1_i, x^2_i)$. Similarly, the test dataset $\gD_\text{test}$ is constructed by randomly sampling pairs of operands that do not appear in $\gD_\text{train}$.
During both training and inference, we then apply different formatting techniques (see Section~\ref{sec:exp}), to construct the final sequence that is input to the model. 
We would like to repeat that both the careful choice of samples in the training dataset as well as their formatting play a crucial role in the final performance of the model.

\paragraph{Text} 
For text data, we use the Shakespeare dataset which was introduced by \citet{shakespeare} originally featured in the blog post ``The Unreasonable Effectiveness of Recurrent Neural Networks''. It consists of 40,000 lines of dialogue carefully curated from  William Shakespeare's plays. The dataset comprises of a total of 1,115,394 characters and 64 unique tokens(when using the character-level tokenizer that we employed in all NanoGPT experiments).

\subsubsection{Data Balancing}
As mentioned in Section~\ref{sec:data_sampling}, we carefully sample our data to ensure that they are ``\emph{balanced}'' with respect to the number of carries and number of digits. As mentioned earlier, sampling the operands uniformly at random would lead to an extremely skewed dataset. To avoid this, we try to \textbf{(i) Balance digits} by sampling lower-digit numbers with higher weights and \textbf{(ii) Balance carry-ons} by sampling such that we have equal number of examples with 0, 1, 2 and 3 carry-on operations. 

Specifically, we create a balanced dataset of $10,000$ samples. This dataset includes all $100$ 1-digit additions and a random sampling of $900$ 2-digit additions (including both $(2+1)$ and $(1+2)$ digit additions) and $9,000$ 3-digit additions. For the 3-digit addition samples, we employ rejection sampling to ensure an equal distribution of carry-ons $(0, 1, 2, \text{or } 3)$. For the test dataset, we uniformly sample $10,000$ addition examples that do not overlap with the train dataset. Results in Figure~\ref{fig:data_sampling} and Table~\ref{tab:data_sampling} demonstrate a clear advantage of the employed data balancing methods. 

For the train dataset, we follow a specific approach based on the number of examples. For sample sizes smaller than $10,000$ (\eg $500, 1,000, 2,000, 3,000, 4,000, 5,000$), we include all 1-digit additions and a proportionate number of 2-digit samples (\eg for a total of $5,000$ samples, we include $900 \times 5,000/10,000 = 450$ two-digit additions). The remaining samples are filled with 3-digit additions from the constructed train dataset of 10,000 samples. For sample sizes larger than 10,000 (\eg 20,000, 40,000), we include all examples from the 10,000-sample train dataset and then add additional samples as needed. Similar to before, we perform rejection sampling to maintain an equal number of carry operations. Table~\ref{table:num_examples}. provides detailed information on the number of samples with 1-digit, 2-digit, and 3-digit additions, as well as the number of carry-ons.

For the other arithmetic operations (subtraction, multiplication, sine, and square root), we construct the train dataset using the following approach:
(i) For subtraction, we use the same pairs of operands that were used for addition.
(ii) For multiplication, we include all 100 cases of a 1-digit number multiplied by a 1-digit number. Additionally, we randomly sample multiplications involving operands of up to 2 digits.
(iii) For sine, we sample a random number in $[\pi/2, \pi/2]$ and truncate it to $4$ decimal places.
(iv) For square root, we sample a random number between $[1, 10]$ and truncate it to $4$ decimal places.
For the test dataset, we sample $10,000$ data points ($7,000$ for multiplication) that do not overlap with the train dataset.

\begin{table}[ht!]
\center
\caption{Performance of addition on various data sampling methods used: (i) Random - uniform sampling of operands; (ii) Balanced digits - sampling more 1 and 2-digit operations ; (iii) Balanced carry - balancing the dataset to contain an equal number of carry-on operations. Experiments on addition with zero-padding each operand and output to have $3$ and $4$ digits, respectively. We observe that balancing the dataset can significantly improve the performance or arithmetic operations.}
\vspace{2mm}
\label{tab:data_sampling}
\centering
\small
\setlength{\tabcolsep}{4pt}
\begin{tabular}{lrrrrrrr}
\toprule
Data Sampling   & \multicolumn{1}{l}{Overall} & \multicolumn{1}{l}{1-digit} & \multicolumn{1}{l}{2-digit} & \multicolumn{1}{l}{Carry-0} & \multicolumn{1}{l}{Carry-1} & \multicolumn{1}{l}{Carry-2} & \multicolumn{1}{l}{Carry-3} \\ \midrule
Random          & 97.74                       & 98.00                       & 96.20                       & 95.88                       & 98.61                       & 98.74                       & 94.98                       \\
Balanced Digits & 98.13                       & \textbf{100.00}             & \textbf{99.70}              & \textbf{98.87}              & \textbf{98.64}              & 98.13                       & 95.93                       \\
Balanced Carry-Ons  & \textbf{98.29}              & \textbf{100.00}             & \textbf{99.70}              & 98.38                       & 97.56                       & \textbf{99.02}              & \textbf{98.22}              \\ \hline
\end{tabular}
\end{table}

\begin{table}[ht!]
\caption{Number of examples of digit $1/2/3$ and $0/1/2/3$ carry-ons for NanoGPT experiments on addition for different number of samples varying from $500$ to $40,000$. %
}
\footnotesize
\vspace{1mm}
\centering
\small
\setlength{\tabcolsep}{4pt}
\begin{tabular}{c|ccc|cccc}
\toprule
Total number & 1-digit & 2-digit & 3-digit & 0-carry-ons & 1-carry-ons & 2-carry-ons & 3-carry-ons \\
\midrule 
500 & 100 &  45 &  355 &  163 & 141 & 97 & 99 \\
1000 & 100 & 90 & 810 & 283 & 268 & 236 & 213 \\
2000 & 100 & 180 & 1720 & 535 & 502 & 481 & 482 \\
3000 & 100 & 270 & 2630 & 781 & 782 & 748 & 689 \\
4000 & 100 & 360 & 3540 & 1020 & 1016 & 958 & 1006 \\
5000 & 100 & 450 & 4450 & 1279 & 1271 & 1229 & 1221 \\ 
\textbf{10000} &  \textbf{100} &  \textbf{900} &  \textbf{9000} &  \textbf{2500} & \textbf{2500} & \textbf{2500} & \textbf{2500} \\
20000 & 121 & 1937 & 17942 & 5000 & 5000 & 5000 & 5000\\
40000 & 132 & 3939 & 35929 & 10000 & 10000 & 10000 & 10000 \\
\bottomrule
\end{tabular}
\label{table:num_examples}
\end{table}

\subsubsection{Data Formatting}\label{sec:appendix-data-formatting}
For each of the four formatting techniques, as applied to each arithmetic operation we provide the details below. \textbf{(i) Plain} refers to the simplest formatting where we simply create a sequence as the mathematical representation of the corresponding operation (\eg $\mathsf{A_3A_2A_1+B_3B_1B_1=C_3C_2C_1}$). For \textbf{(ii) Reverse}, we simply reverse the digits of the output so that they appear in increasing order from LSB to MSB (\eg $\mathsf{\$A_3A_2A_1+B_3B_1B_1=C_1C_2C_3\$}$). \textbf{(iii) Simplified Scratchpad} and \textbf{(iv) Detailed Scratchpad} provide algorithmic reasoning steps like \citep{nye2021show, zhou2022teaching} so as to help the model get more ``information'' per sample. Our intuition is that this approach nudges the model towards actually learning the algorithm of addition or subtraction rather than merely trying to fit the training examples. Refer to Appendix~\ref{sec:prompt_examples} for detailed examples of data formatting for each arithmetic operation.

\paragraph{Addition} We focus on additions of positive numbers up to 3-digits, in which the plain formatting would look like $\mathsf{A_3A_2A_1+B_3B_1B_1=C_3C_2C_1}$. 
For experiments on comparing data sampling presented in Figure~\ref{fig:data_sampling}, we pad the two operands and the output with zero, to be of length 3 and 4 respectively. For all other experiments, we \textbf{do not utilize zero-padding. }
For Scratchpad-based methods \textbf{(iii, iv)}, 
we provide the digit-wise addition (denoted as $A$) and carry-on (denoted as $C$) information for intermediate steps from the least significant bit (LSB) to the most significant bit (MSB).

\paragraph{Subtraction} We consider subtraction of positive numbers up to $3$ digits, written as $\mathsf{A_3A_2A_1 - B_3B_2B_1 = C_3C_2C_1}$ for plain formatting. 
As with addition, Scratchpad-based methods \textbf{(iii, iv)}, present the intermediate steps of digit-wise subtraction and carry-ons\footnote{As explained in Section~\ref{sec:data_sampling}, we use the term ``carry-on" to refer to the ``borrow" operation}. These steps are performed from the least significant bit (LSB) to the most significant bit (MSB). If the final result after computing all the digit-wise subtractions is negative, we subtract the number in the most significant bit (MSB) position multiplied by 10 to the power of (number of digits in the output - 1) from the remaining digits in the output. In Section~\ref{sec:appendix_subtraction_detailed_scratchpad}, we present an alternative version of the detailed scratchpad formatting for subtraction. 

\paragraph{Multiplication} We consider multiplication of positive numbers only up to 2-digits. Examples with (i) plain formatting look like: $\mathsf{A_2A_1 * B_2B_1 = C_4C_3C_2C_1}$ while (ii) reverse is formatted as $\mathsf{A_2A_1 * B_2B_1 = C_1C_2C_3C_4}$. For (iv) detailed scratchpad method, we simplify each intermediate step by performing a series of multiplications between the first operand and each digit of the second operand, starting from the least significant bit (LSB) and moving towards the most significant bit (MSB). For each step, we multiply the result by an exponentiation of 10 corresponding to the relative digit position.

\paragraph{Sine} We consider decimal numbers in the range of $[-\pi/2, \pi/2]$, truncated to 4-digits of precision with (i) plain formatting: $\sin(\mathsf{A_0.A_1A_2A_3A_4})=\mathsf{B_0.B_1B_2B_3B_4}$. For (iv) detailed scratchpad, we include the individual steps of the Taylor series expansion for sine, which is represented as $\sin(x) = x - \frac{1}{3!}x^3 + \frac{1}{5!}x^5 - \frac{1}{7!}x^7 + \cdots $. It is important to note that these intermediate steps involve exponentiation, which may not be any easier to compute than the sine operation itself.

\paragraph{Square Root}  We consider decimal numbers in the range of $[1, 10)$, truncated to 4-digits of precision with the format, with (i) plain formatting: $\text{sqrt}(\mathsf{A_0.A_1A_2A_3A_4})=\mathsf{B_0.B_1B_2B_3B_4}$. For (iv) detailed scratchpad, We present each step of Newton's method for computing the square root function. The iterative formula is given by $x_n = \frac{1}{2} (x_{n-1} + \frac{x}{x_{n-1}})$, where $x_0$ is initialized as the floor of the square root value of the operand $x$. It is important to note that these intermediate steps involve a division operation, which can be as complex as the square root operation itself.

\subsection{Model}\label{appdx:model}

For all experiments, we use a Decoder-only Transformer architecture. Specifically, we primarily use the NanoGPT model, a scaled-down variant of the GPT-2 model with half the number of self-attention layers, heads, and embedding dimension. Note that we use character-level tokenization instead of using the OpenAI's BPE tokenizer (Tiktoken) of vocabulary size $50257$, making the vocabulary size significantly smaller. We use a learnable absolute positional embedding initialized randomly, following the GPT-2 model. Are results are generated using a temperature of 0.8.

In the case of arithmetic tasks performed on plain and reverse formatting, we set a context length of $256$ for NanoGPT experiments. The length of a single train example falls within the range of $13$ to $15$, approximately. However, when conducting experiments on scratchpad formatting, we increase the context length to 1024. This adjustment allows us to accommodate more examples per batch. In the case of simplified scratchpad, the length of each train example is approximately 64, while the detailed scratchpad has a length of approximately 281. For GPT-2 experiments we fix the context length to 1024 for all experiments. See Table~\ref{table:model_config} for details on model configuration. 

For experiments on fine-tuning a pretrained large language model, we use OpenAI's GPT-3 model - Ada, Curie, and Davinci. 

\begin{figure}[h]
    \centering
    \includegraphics[width=0.25\textwidth]{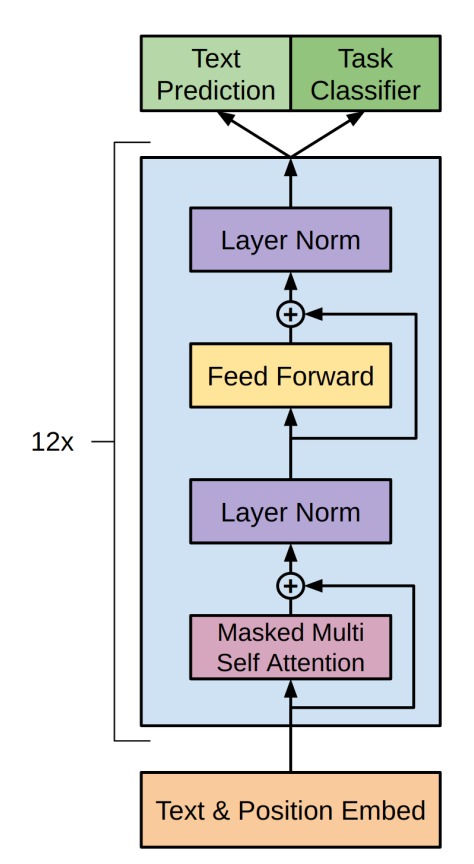}
    \caption{The GPT-2 Architecture. Image from ~\citep{Radford2018ImprovingLU}. NanoGPT model is a smaller model with half the number of self-attention layers, multi-heads, and embedding dimensions.}
\label{fig:gpt2_arch}
\end{figure}

\begin{table}[ht!]
\caption{NanoGPT and GPT-2 model configuration}
\footnotesize
\vspace{1mm}
\centering
\small
\setlength{\tabcolsep}{4pt} %
\renewcommand{\arraystretch}{0.5}
     {
        \begin{tabular}{cccccc} %
        \toprule
        Model & Input Formatting & Context Length & Self-Attn Layers & Num Heads & Embedding Dim \\
        \midrule 
        \multirow{2}{*}{NanoGPT} & Plain, Reverse &  256 &  6 &  6 &  384 \\
        & Scratchpad & 1024 &  6 &  6 &  384 \\
        \multirow{2}{*}{GPT-2} & Plain, Reverse & 1024  &  12 &  12 &  768 \\
        & Scratchpad & 1024 &  12 &  12 &  768 \\
        \bottomrule
        \end{tabular}
    }
\label{table:model_config}
\end{table}

\subsection{Hyperparameter Configurations}\label{appdx:config}

In this section, we provide a detailed overview of the hyperparameter configuration used in our experiments in Table~\ref{table:hyperparam_nanogpt} and \ref{table:hyperparam_gpt2}. To enhance memory efficiency and training speed, we employ flash attention. For most experiments, we utilize the bfloat16 data type. However, when working with Nvidia 2080 GPUs, which do not support bfloat16, we switch to float16. It is worth noting that we did not observe significant differences in training and evaluation performance between the two data types.

For the GPT-2 experimentation, we reduced the batch size to 8 to accommodate the GPU memory limitations. However, to mitigate the impact of the smaller batch size, we employed gradient accumulation steps. This approach involves taking multiple steps between gradient updates, effectively increasing the \textit{effective} batch size to 64. For specific hyperparameter details, please refer to Table~\ref{table:hyperparam_gpt2}.

\begin{table}[ht]
\caption{Hyper Parameters used for NanoGPT experiments on arithmetic tasks}
\footnotesize
\vspace{1mm}
\centering
\small
\setlength{\tabcolsep}{4pt} %
\renewcommand{\arraystretch}{0.5}
{
\begin{tabular}{ccccccccc}
\toprule
Input Format & Batch Size & Optimizer & LR & Betas & Iterations & Warmup Iter & Wt decay & Dropout \\
\midrule 
Plain, Reverse &  256 &  AdamW &  0.001 &  $(0.9,0.99)$ & 5000 & 100 & 0.1 & 0.2 \\
Scratchpad & 16 &  AdamW &  0.001 &  $(0.9,0.99)$ & 50000 & 0 & 0.1 & 0.2 \\
\bottomrule
\end{tabular}
}
\label{table:hyperparam_nanogpt}
\end{table}

\begin{table}[ht]
\caption{Hyper Parameters used for GPT-2 experiments on arithmetic tasks}
\footnotesize
\vspace{1mm}
\centering
\small
\setlength{\tabcolsep}{4pt}
\begin{tabular}{ccccccccc}
\toprule
Input Format & Batch Size & Optimizer & LR & Betas & Iterations & Warmup Iter & Wt decay & Dropout \\
\midrule 
Plain, Reverse &  64 &  AdamW &  0.0005 &  $(0.9,0.99)$ & 5000 & 100 & 0.1 & 0.2 \\
Scratchpad & 64 &  AdamW &  0.0005 &  $(0.9,0.99)$ & 20000 & 0 & 0.1 & 0.2 \\
\bottomrule
\end{tabular}
\label{table:hyperparam_gpt2}

\end{table}

\begin{table}[ht!]
\caption{Hyper Parameters used for tandem training experiments in Section~\ref{sec:exp3}. }
\footnotesize
\vspace{1mm}
\centering
\small
\setlength{\tabcolsep}{4pt}
\begin{tabular}{ccccccccc}
\toprule
Model & Batch Size & Optimizer & LR & Betas & Iterations & Warmup Iter & Wt decay & Dropout \\
\midrule 
NanoGPT &  16 &  AdamW &  0.001 &  $(0.9,0.99)$ & 5000 & 0 & 0.1 & 0.2 \\
GPT-2 & 40 &  AdamW &  0.0006 &  $(0.9,0.95)$ & 50000 & 2000 & 0.1 & 0.2 \\
\bottomrule
\end{tabular}
\label{table:hyperparam_tandem}
\end{table}

\begin{figure}[h!] 
\centering
\includegraphics[width=1\textwidth]{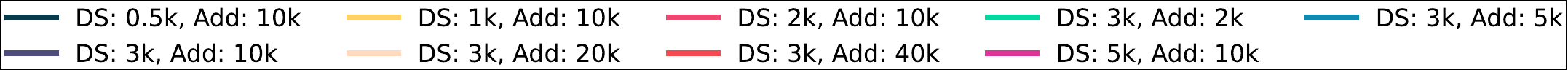}
\subfloat[\scriptsize{NanoGPT, plain addition}]{\includegraphics[width=0.32\textwidth]{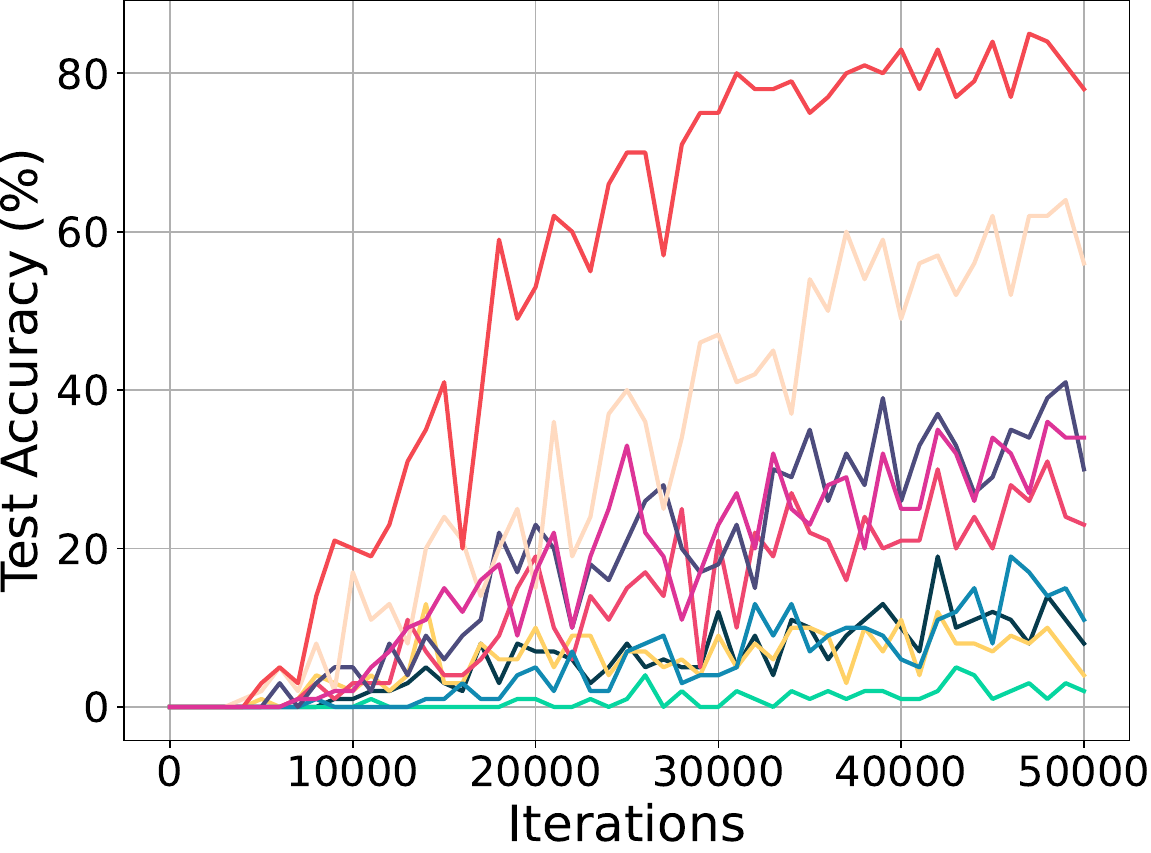}}
\hspace{0.1mm}
\subfloat[\scriptsize{NanoGPT, detailed scratchpad addition}]{\includegraphics[width=0.32\textwidth]{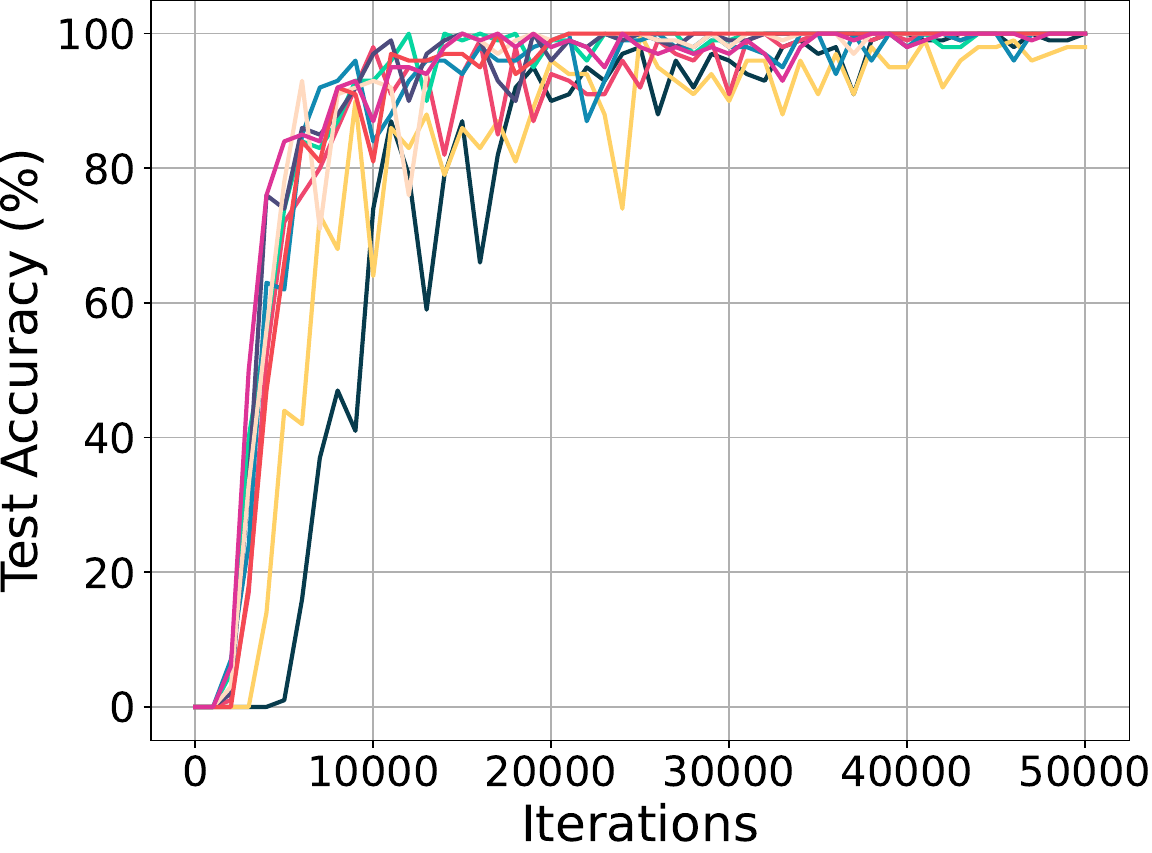}}
\hspace{0.1mm}
\subfloat[\scriptsize{NanoGPT, Perplexity}]{\includegraphics[width=0.32\textwidth]{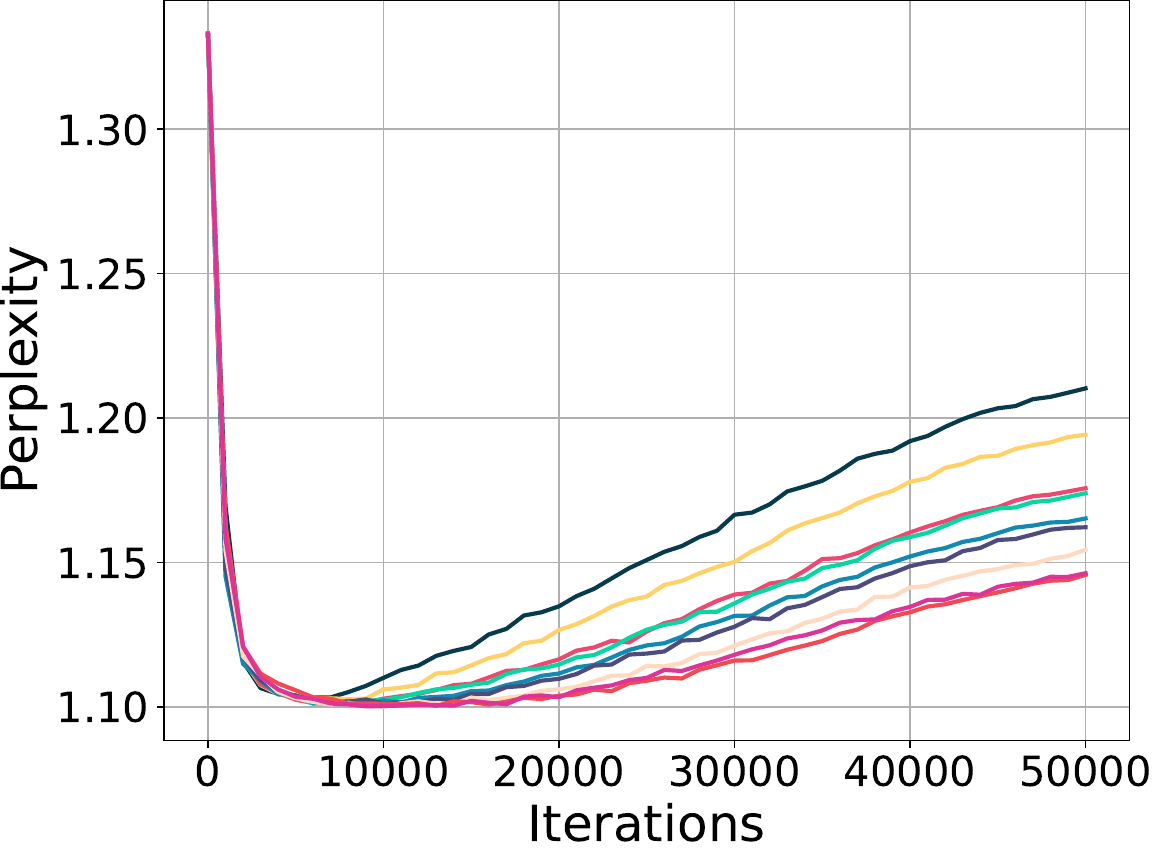}}
\vspace{-0.5mm}
\subfloat[\scriptsize{GPT-2, plain addition}]{\includegraphics[width=0.32\textwidth]{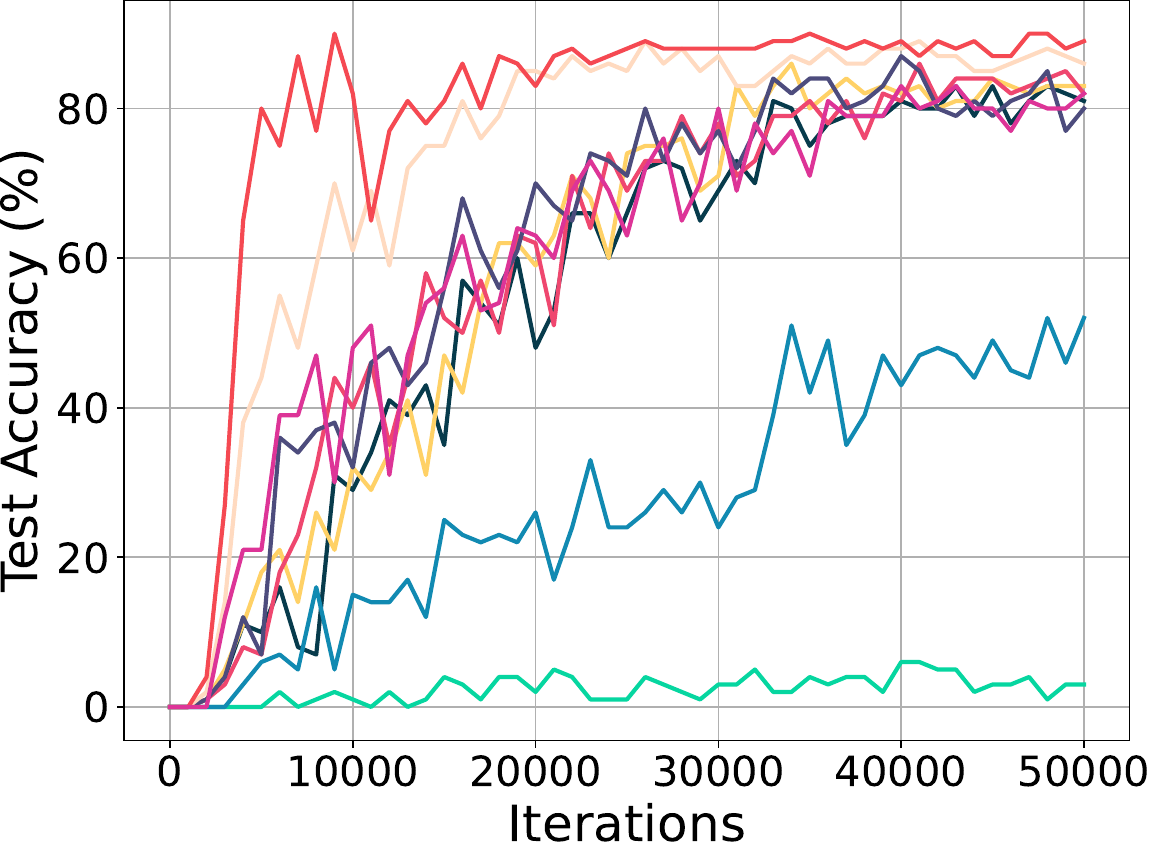}}
\hspace{0.1mm}
\subfloat[\scriptsize{GPT-2, detailed scratchpad addition}]{\includegraphics[width=0.32\textwidth]{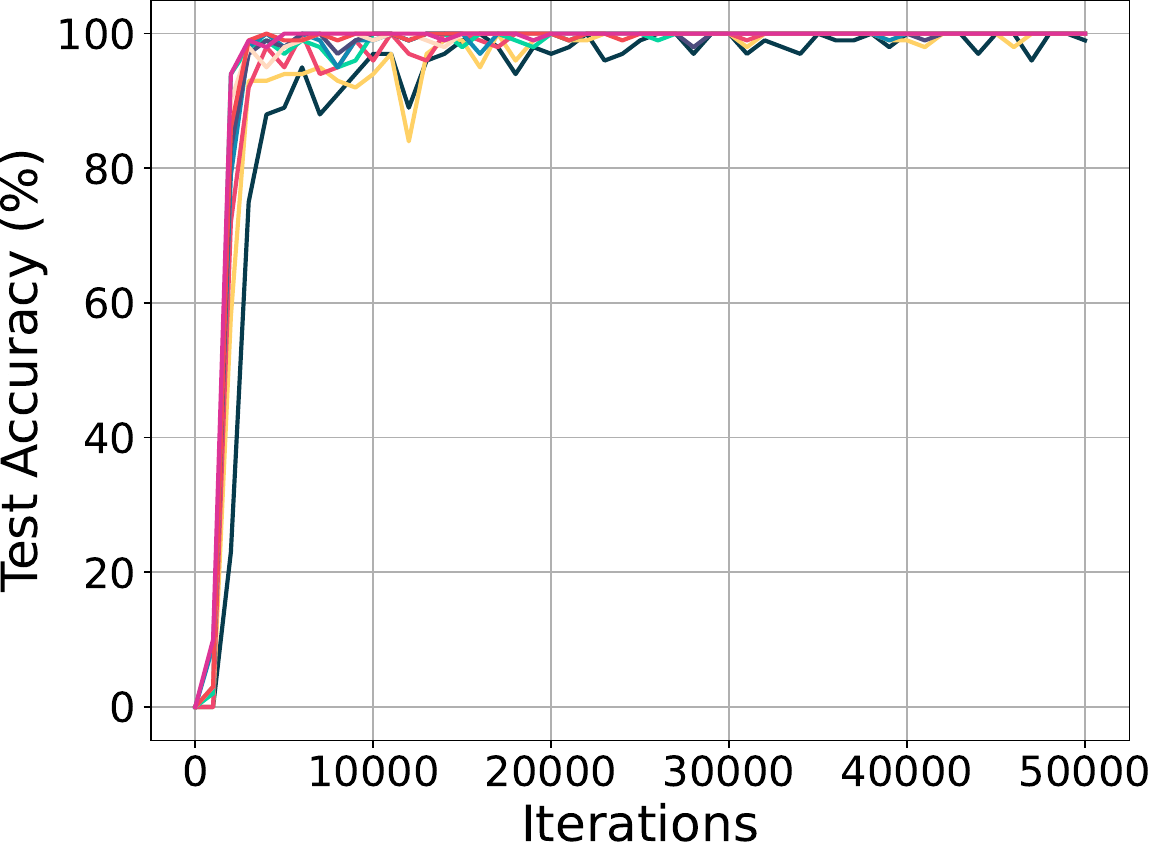}}
\hspace{0.1mm}
\subfloat[\scriptsize{GPT-2, Perplexity}]{\includegraphics[width=0.32\textwidth]{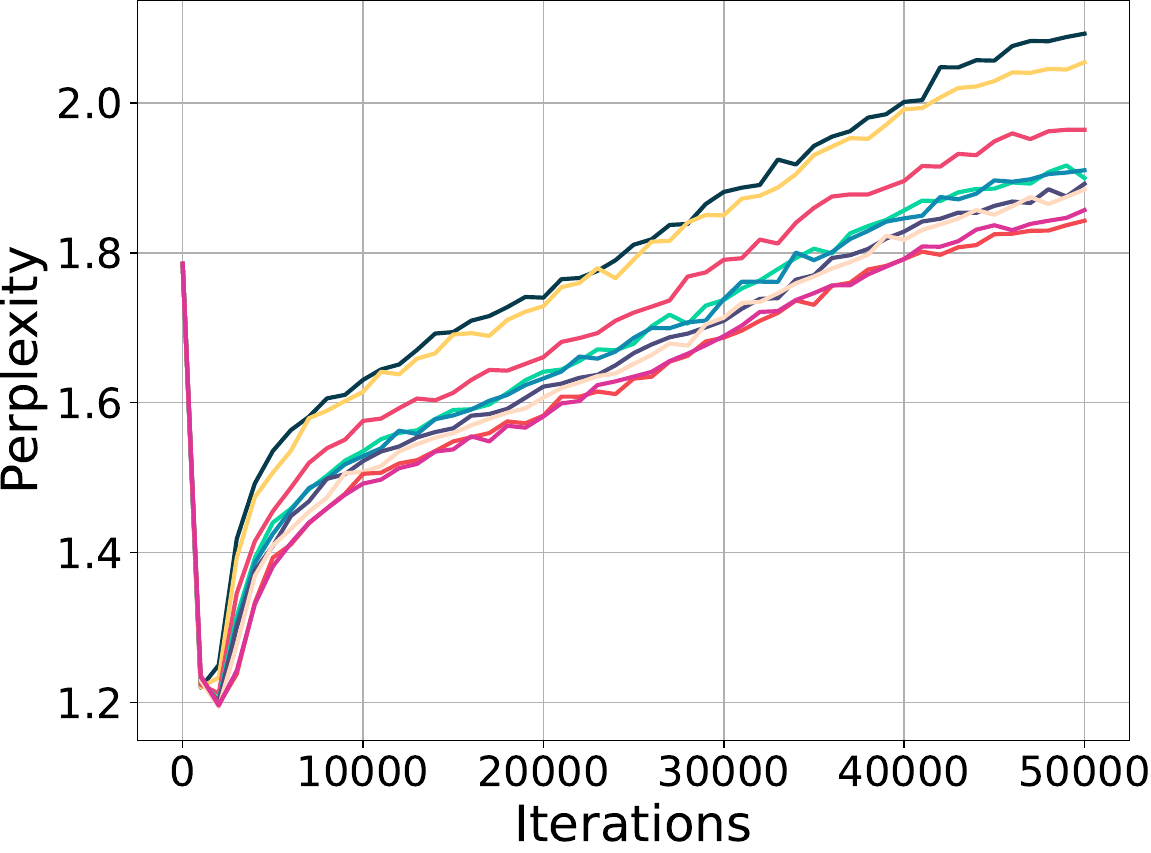}}
    \caption{Training loss curves for NanoGPT and GPT-2 trained with varying numbers of plain (Add) and detailed scratchpad (DS) samples as well as the shakespeare dataset as described in Section~\ref{sec:exp3}. As we can see, the model continues to improve in addition accuracy as the number of iterations increases. However, the training perplexity on Shakespeare also tends to increase, which indicates some overfitting. However, we note that the model still outputs ``reasonable'' text when prompted with shakespeare text.} 
\label{fig:train_curve_nanogpt}
\vspace{10ex}
\end{figure}

%% file: _Appendix_Prompts.tex
\newpage
\section{Prompt Examples}\label{sec:prompt_examples}
In this section, we provide three examples of each formatting (plain, reverse, simplified scratchpad, detailed scratchpad) of arithmetic operations ($+,-,\times,\sin{},\sqrt{})$. 

\subsection{Addition}
\begin{AIbox}[breakable]{\bf{\large Addition Examples}}
\vspace{5mm}
\begin{minipage}[t]{0.35\linewidth}
\centering
\textbf{Plain}
\begin{lstlisting}[language=markdown]
266+738=(*@\highlighttext{1004}@*)
980+743=(*@\highlighttext{1723}@*)
41+34=(*@\highlighttext{75}@*)
\end{lstlisting}

\vspace{2ex}
\textbf{Reverse}
\begin{lstlisting}[language=markdown]
$913+524=(*@\highlighttext{1437\$}@*)
$226+598=(*@\highlighttext{824\$}@*)
$35+58=(*@\highlighttext{93\$}@*)
\end{lstlisting}

\vspace{2ex}
\textbf{Simplified Scratchpad}
\begin{lstlisting}[language=markdown]
Input:
922+244
Target:
(*@\highlighttext{A->6 , C->0}@*)
(*@\highlighttext{A->6 , C->0}@*)
(*@\highlighttext{A->1 , C->1.}@*)
(*@\highlighttext{1166}@*)
Input:
285+43
Target:
(*@\highlighttext{A->8 , C->0}@*)
(*@\highlighttext{A->2 , C->1}@*)
(*@\highlighttext{A->3 , C->0.}@*)
(*@\highlighttext{328}@*)
Input:
993+849
Target:
(*@\highlighttext{A->2 , C->1}@*)
(*@\highlighttext{A->4 , C->1}@*)
(*@\highlighttext{A->8 , C->1.}@*)
(*@\highlighttext{1842}@*)
\end{lstlisting}

\end{minipage}
\begin{minipage}[t]{0.65\linewidth}
\centering
\textbf{Detailed Scratchpad}
\begin{lstlisting}[language=markdown]
Input:
396+262
Target:
(*@\highlighttext{<scratch>}@*)
(*@\highlighttext{[3,9,6] has 3 digits.}@*)
(*@\highlighttext{[2,6,2] has 3 digits.}@*)
(*@\highlighttext{[3,9,6] + [2,6,2] , A=[] , C=0 , 6+2+0=8 , A->8 , C->0}@*)
(*@\highlighttext{[3,9] + [2,6] , A=[8] , C=0 , 9+6+0=15 , A->5 , C->1}@*)
(*@\highlighttext{[3] + [2] , A=[5,8] , C=1 , 3+2+1=6 , A->6 , C->0}@*)
(*@\highlighttext{[] + [] , A=[6,5,8] C=0 , END}@*)
(*@\highlighttext{</scratch>}@*)
(*@\highlighttext{6 5 8}@*)
Input:
796+890
Target:
(*@\highlighttext{<scratch>}@*)
(*@\highlighttext{[7,9,6] has 3 digits.}@*)
(*@\highlighttext{[8,9,0] has 3 digits.}@*)
(*@\highlighttext{[7,9,6] + [8,9,0] , A=[] , C=0 , 6+0+0=6 , A->6 , C->0}@*)
(*@\highlighttext{[7,9] + [8,9] , A=[6] , C=0 , 9+9+0=18 , A->8 , C->1}@*)
(*@\highlighttext{[7] + [8] , A=[8,6] , C=1 , 7+8+1=16 , A->6 , C->1}@*)
(*@\highlighttext{[] + [] , A=[6,8,6] C=1 , END}@*)
(*@\highlighttext{</scratch>}@*)
(*@\highlighttext{1 6 8 6}@*)
Input:
788+989
Target:
(*@\highlighttext{<scratch>}@*)
(*@\highlighttext{[7,8,8] has 3 digits.}@*)
(*@\highlighttext{[9,8,9] has 3 digits.}@*)
(*@\highlighttext{[7,8,8] + [9,8,9] , A=[] , C=0 , 8+9+0=17 , A->7 , C->1}@*)
(*@\highlighttext{[7,8] + [9,8] , A=[7] , C=1 , 8+8+1=17 , A->7 , C->1}@*)
(*@\highlighttext{[7] + [9] , A=[7,7] , C=1 , 7+9+1=17 , A->7 , C->1}@*)
(*@\highlighttext{[] + [] , A=[7,7,7] C=1 , END}@*)
(*@\highlighttext{</scratch>}@*)
(*@\highlighttext{1 7 7 7}@*)
\end{lstlisting}
\end{minipage}
\end{AIbox}

\newpage
\subsection{Subtraction}
\begin{AIbox}[breakable]{\bf{\large Subtraction Examples}}
\vspace{5mm}
\begin{minipage}[t]{0.35\linewidth}
\centering
\textbf{Plain}
\begin{lstlisting}[language=markdown]
266-738=(*@\highlighttext{-472}@*)
980-743=(*@\highlighttext{237}@*)
41-34=(*@\highlighttext{7}@*)
\end{lstlisting}

\vspace{2ex}
\textbf{Reverse}
\begin{lstlisting}[language=markdown]
$913-524=(*@\highlighttext{983\$}@*)
$226-598=(*@\highlighttext{273-\$}@*)
$35-58=(*@\highlighttext{32-\$}@*)
\end{lstlisting}

\vspace{2ex}
\textbf{Simplified Scratchpad}
\begin{lstlisting}[language=markdown]
Input:
396-262
Target:
(*@\highlighttext{A->4 , C->0}@*)
(*@\highlighttext{A->3 , C->0}@*)
(*@\highlighttext{A->1 , C->0}@*)
(*@\highlighttext{100+34=134.}@*)
(*@\highlighttext{134}@*)
Input:
796-890
Target:
(*@\highlighttext{A->6 , C->0}@*)
(*@\highlighttext{A->0 , C->0}@*)
(*@\highlighttext{A->-1 , C->-1}@*)
(*@\highlighttext{-100+6=-94.}@*)
(*@\highlighttext{-94}@*)
Input:
788-989
Target:
(*@\highlighttext{A->9 , C->-1}@*)
(*@\highlighttext{A->9 , C->-1}@*)
(*@\highlighttext{A->-3 , C->-1}@*)
(*@\highlighttext{-300+99=-201.}@*)
(*@\highlighttext{-201}@*)
\end{lstlisting}
\begin{minted}[escapeinside=||]{markdown}
||
\end{minted}
\end{minipage}
\begin{minipage}[t]{0.66\linewidth}
\centering
\textbf{Detailed Scratchpad}
\begin{lstlisting}[language=markdown]
Input:
396-262
Target:
(*@\highlighttext{<scratch>}@*)
(*@\highlighttext{[3,9,6] has 3 digits.}@*)
(*@\highlighttext{[2,6,2] has 3 digits.}@*)
(*@\highlighttext{[3,9,6] - [2,6,2] , A=[] , C=0 , 6-2-0=4 , A->4 , C->0}@*)
(*@\highlighttext{[3,9] - [2,6] , A=[4] , C=0 , 9-6-0=3 , A->3 , C->0}@*)
(*@\highlighttext{[3] - [2] , A=[3,4] , C=0 , 3-2-0=1 , A->1 , C->0}@*)
(*@\highlighttext{[] - [] , A=[1,3,4]}@*)
(*@\highlighttext{100+34=134 , END}@*)
(*@\highlighttext{</scratch>}@*)
(*@\highlighttext{1 3 4}@*)
Input:
796-890
Target:
(*@\highlighttext{<scratch>}@*)
(*@\highlighttext{[7,9,6] has 3 digits.}@*)
(*@\highlighttext{[8,9,0] has 3 digits.}@*)
(*@\highlighttext{[7,9,6] - [8,9,0] , A=[] , C=0 , 6-0-0=6 , A->6 , C->0}@*)
(*@\highlighttext{[7,9] - [8,9] , A=[6] , C=0 , 9-9-0=0 , A->0 , C->0}@*)
(*@\highlighttext{[7] - [8] , A=[0,6] , C=0 , 7-8-0=-1 , A->-1 , C->-1}@*)
(*@\highlighttext{[] - [] , A=[-1,0,6]}@*)
(*@\highlighttext{</scratch>}@*)
(*@\highlighttext{-9 4}@*)
Input:
788-989
Target:
(*@\highlighttext{<scratch>}@*)
(*@\highlighttext{[7,8,8] has 3 digits.}@*)
(*@\highlighttext{[9,8,9] has 3 digits.}@*)
(*@\highlighttext{[7,8,8] - [9,8,9] , A=[] , C=0 , 8-9-0+10=9 , A->9 , C->-1}@*)
(*@\highlighttext{[7,8] - [9,8] , A=[9] , C=-1 , 8-8-1+10=9 , A->9 , C->-1}@*)
(*@\highlighttext{[7] - [9] , A=[9,9] , C=-1 , 7-9-1=-3 , A->-3 , C->-1}@*)
(*@\highlighttext{[] - [] , A=[-3,9,9]}@*)
(*@\highlighttext{-300+99=-201 , END}@*)
(*@\highlighttext{</scratch>}@*)
(*@\highlighttext{-2 0 1}@*)
\end{lstlisting}
\end{minipage}
\end{AIbox}

\newpage
\subsection{Multiplication}
\vspace{-2mm}
\begin{AIbox}[breakable]{\bf{\large Multiplication Examples}}
\vspace{5mm}
\begin{minipage}[t]{0.22\linewidth}
\centering
\textbf{Plain}
\begin{lstlisting}[language=markdown]
5*32=(*@\highlighttext{160}@*)
66*76=(*@\highlighttext{5016}@*)
67*74=(*@\highlighttext{4958}@*)
\end{lstlisting}

\vspace{2ex}
\textbf{Reverse}
\begin{lstlisting}[language=markdown]
$5*32=(*@\highlighttext{061\$}@*)
$66*76=(*@\highlighttext{6105\$}@*)
$67*74=(*@\highlighttext{8594\$}@*)
\end{lstlisting}
\end{minipage}
\begin{minipage}[t]{0.78\linewidth}
\centering
\textbf{Detailed Scratchpad}
\begin{lstlisting}[language=markdown]
Input:
22*52
Target:
(*@\highlighttext{<scratch>}@*)
(*@\highlighttext{[2,2] has 2 digits.}@*)
(*@\highlighttext{[5,2] has 2 digits.}@*)
(*@\highlighttext{[2,2] * 2 , A=[4,4] , k=1 , B=[4,4] , C=0+44=44}@*)
(*@\highlighttext{[2,2] * 5 , A=[1,1,0] , k=10 , B=[1,1,0,0] , C=44+1100=1144 , END}@*)
(*@\highlighttext{</scratch>}@*)
(*@\highlighttext{1 1 4 4}@*)
Input:
8*69
Target:
(*@\highlighttext{<scratch>}@*)
(*@\highlighttext{[8] has 1 digits.}@*)
(*@\highlighttext{[6,9] has 2 digits.}@*)
(*@\highlighttext{[8] * 9 , A=[7,2] , k=1 , B=[7,2] , C=0+72=72}@*)
(*@\highlighttext{[8] * 6 , A=[4,8] , k=10 , B=[4,8,0] , C=72+480=552 , END}@*)
(*@\highlighttext{</scratch>}@*)
(*@\highlighttext{5 5 2}@*)
Input:
52*34
Target:
(*@\highlighttext{<scratch>}@*)
(*@\highlighttext{[5,2] has 2 digits.}@*)
(*@\highlighttext{[3,4] has 2 digits.}@*)
(*@\highlighttext{[5,2] * 4 , A=[2,0,8] , k=1 , B=[2,0,8] , C=0+208=208}@*)
(*@\highlighttext{[5,2] * 3 , A=[1,5,6] , k=10 , B=[1,5,6,0] , C=208+1560=1768 , END}@*)
(*@\highlighttext{</scratch>}@*)
(*@\highlighttext{1 7 6 8}@*)
\end{lstlisting}
\end{minipage}
\end{AIbox}

\vspace{-3mm}
\subsection{Sine}
\vspace{-2mm}
\begin{AIbox}[breakable]{\bf{\large Sine Examples}}
\vspace{5mm}
\begin{minipage}[t]{0.33\linewidth}
\textbf{Plain}
\begin{lstlisting}[language=markdown]
sin(1.0313)=(*@\highlighttext{0.8579}@*)
sin(-0.6909)=(*@\highlighttext{-0.6373}@*)
sin(-0.5719)=(*@\highlighttext{-0.5413}@*)
\end{lstlisting}
\end{minipage}
\begin{minipage}[t]{0.66\linewidth}
\textbf{Detailed Scratchpad}
\begin{lstlisting}[language=markdown]
Input:
sin(1.0313)
Target:
(*@\highlighttext{<scratch>}@*)
(*@\highlighttext{x\_0=1.0313}@*)
(*@\highlighttext{x\_1: x\_0 - 1/3! * (x\^{}3) , x\_1=0.8484}@*)
(*@\highlighttext{x\_2: x\_1 + 1/5! * (x\^{}5) , x\_2=0.8581}@*)
(*@\highlighttext{x\_3: x\_2 - 1/7! * (x\^{}7) , x\_3=0.8578}@*)
(*@\highlighttext{x\_4: x\_3 + 1/9! * (x\^{}9) , x\_4=0.8578 , END}@*)
(*@\highlighttext{</scratch>}@*)
(*@\highlighttext{0.8578}@*)
Input:
sin(-0.6909)
Target:
(*@\highlighttext{<scratch>}@*)
(*@\highlighttext{x\_0=-0.6909}@*)
(*@\highlighttext{x\_1: x\_0 - 1/3! * (x\^{}3) , x\_1=-0.636}@*)
(*@\highlighttext{x\_2: x\_1 + 1/5! * (x\^{}5) , x\_2=-0.6374}@*)
(*@\highlighttext{x\_3: x\_2 - 1/7! * (x\^{}7) , x\_3=-0.6374}@*)
(*@\highlighttext{x\_4: x\_3 + 1/9! * (x\^{}9) , x\_4=-0.6375 , END}@*)
(*@\highlighttext{</scratch>}@*)
(*@\highlighttext{-0.6375}@*)
Input:
sin(-0.5719)
Target:
(*@\highlighttext{<scratch>}@*)
(*@\highlighttext{x\_0=-0.5719}@*)
(*@\highlighttext{x\_1: x\_0 - 1/3! * (x\^{}3) , x\_1=-0.5408}@*)
(*@\highlighttext{x\_2: x\_1 + 1/5! * (x\^{}5) , x\_2=-0.5414}@*)
(*@\highlighttext{x\_3: x\_2 - 1/7! * (x\^{}7) , x\_3=-0.5414}@*)
(*@\highlighttext{x\_4: x\_3 + 1/9! * (x\^{}9) , x\_4=-0.5415 , END}@*)
(*@\highlighttext{</scratch>}@*)
(*@\highlighttext{-0.5415}@*)
\end{lstlisting}
\end{minipage}
\end{AIbox}

\subsection{Square Root}

\begin{AIbox}[breakable]{\bf{\large Square Root Examples}}
\vspace{5mm}
\begin{minipage}[t]{0.33\linewidth}
\textbf{Plain}
\begin{lstlisting}[language=markdown]
sqrt(7.2726)=(*@\highlighttext{2.6967}@*)
sqrt(3.6224)=(*@\highlighttext{1.9032}@*)
sqrt(1.0895)=(*@\highlighttext{1.0437}@*)
\end{lstlisting}
\end{minipage}
\begin{minipage}[t]{0.66\linewidth}
\textbf{Detailed Scratchpad}
\begin{lstlisting}[language=markdown]
Input:
sqrt(7.1042)
Target:
(*@\highlighttext{<scratch>}@*)
(*@\highlighttext{x\_0=2}@*)
(*@\highlighttext{x\_1: 1/2*(2+7.1042/2)=2.776, x\_1=2.776}@*)
(*@\highlighttext{x\_2: 1/2*(2.776+7.1042/2.776)=2.6675, x\_2=2.6675}@*)
(*@\highlighttext{x\_3: 1/2*(2.6675+7.1042/2.6675)=2.6653, x\_3=2.6653}@*)
(*@\highlighttext{x\_4: 1/2*(2.6653+7.1042/2.6653)=2.6653, x\_4=2.6653 , END}@*)
(*@\highlighttext{</scratch>}@*)
(*@\highlighttext{2.6653}@*)
Input:
sqrt(6.2668)
Target:
(*@\highlighttext{<scratch>}@*)
(*@\highlighttext{x\_0=2}@*)
(*@\highlighttext{x\_1: 1/2*(2+6.2668/2)=2.5667, x\_1=2.5667}@*)
(*@\highlighttext{x\_2: 1/2*(2.5667+6.2668/2.5667)=2.5041, x\_2=2.5041}@*)
(*@\highlighttext{x\_3: 1/2*(2.5041+6.2668/2.5041)=2.5033, x\_3=2.5033}@*)
(*@\highlighttext{x\_4: 1/2*(2.5033+6.2668/2.5033)=2.5033, x\_4=2.5033 , END}@*)
(*@\highlighttext{</scratch>}@*)
(*@\highlighttext{2.5033}@*)
Input:
sqrt(8.3216)
Target:
(*@\highlighttext{<scratch>}@*)
(*@\highlighttext{x\_0=2}@*)
(*@\highlighttext{x\_1: 1/2*(2+8.3216/2)=3.0804, x\_1=3.0804}@*)
(*@\highlighttext{x\_2: 1/2*(3.0804+8.3216/3.0804)=2.8909, x\_2=2.8909}@*)
(*@\highlighttext{x\_3: 1/2*(2.8909+8.3216/2.8909)=2.8847, x\_3=2.8847}@*)
(*@\highlighttext{x\_4: 1/2*(2.8847+8.3216/2.8847)=2.8847, x\_4=2.8847 , END}@*)
(*@\highlighttext{</scratch>}@*)
(*@\highlighttext{2.8847}@*)
\end{lstlisting}
\end{minipage}
\end{AIbox}

\subsection{Noisy Simple Scratchpad}
We provide one example for each case of adding noise in the simplified scratchpad experiments discussed in Section~\ref{sec:ablation_noisy}.

\begin{AIbox}[breakable]{\bf{\large Noisy Simple Scratchpad Examples}}
\vspace{5mm}
{\tt \footnotesize We provide one example for each case of adding noise in the simplified scratchpad experiments discussed in Section~\ref{sec:ablation_noisy}. The input prompt is highlighted in light blue, while the remaining part is highlighted in light green. We construct the dataset to have either correct or random digit-sum A and carry information C. For all cases, the final answer remains accurate.\\\\}
\textbf{Prompt:}
\begin{lstlisting}[language=markdown]
Input:
686+886
Target:
\end{lstlisting}
\vspace{2ex}
\begin{minipage}[t]{0.25\linewidth}
\textbf{Correct A \& C}
\begin{lstlisting}[language=markdown]
(*@\highlighttext{A->2 , C->1}@*)
(*@\highlighttext{A->7 , C->1}@*)
(*@\highlighttext{A->5 , C->1.}@*)
(*@\highlighttext{1572}@*)
\end{lstlisting}
\end{minipage}
\begin{minipage}[t]{0.25\linewidth}
\textbf{Random C}
\begin{lstlisting}[language=markdown]
(*@\highlighttext{A->2 , C->0}@*)
(*@\highlighttext{A->7 , C->0}@*)
(*@\highlighttext{A->5 , C->1.}@*)
(*@\highlighttext{1572}@*)
\end{lstlisting}\end{minipage}
\begin{minipage}[t]{0.24\linewidth}
\textbf{Random A}
\begin{lstlisting}[language=markdown]
(*@\highlighttext{A->0 , C->1}@*)
(*@\highlighttext{A->9 , C->1}@*)
(*@\highlighttext{A->9 , C->1.}@*)
(*@\highlighttext{1572}@*)
\end{lstlisting}
\end{minipage}
\begin{minipage}[t]{0.24\linewidth}
\textbf{Random A \& C}
\begin{lstlisting}[language=markdown]
(*@\highlighttext{A->8 , C->1}@*)
(*@\highlighttext{A->1 , C->0}@*)
(*@\highlighttext{A->2 , C->1.}@*)
(*@\highlighttext{1572}@*)
\end{lstlisting}
\end{minipage}
\end{AIbox}

\subsection{Example data for GPT-3 fine-tuning}
We provide an example from the training dataset consisting of one prompt-completion pair used for fine-tuning the GPT-3 model using OpenAI's API. The prompt is highlighted in light grey, while the completion is highlighted in light green. Note that for plain and reverse formatting, we include spacing between digits to ensure consistent tokenization of numbers. ``\#\#\#'' is used as the stop sequence for generation.

\subsubsection{Addition}
\begin{AIbox}[breakable]{\bf{\large Addition Examples}}
\vspace{5mm}
\begin{minipage}[t]{0.35\linewidth}
\centering
\textbf{Plain}
\begin{lstlisting}[language=markdown]
6 7 7 + 8 9 8 =(*@\highlighttext{1 5 7 5\#\#\#}@*)
\end{lstlisting}

\vspace{2ex}
\textbf{Reverse}
\begin{lstlisting}[language=markdown]
7 4 9 + 7 8 5 =(*@\highlighttext{ 4 3 5 1\#\#\#}@*)
\end{lstlisting}

\vspace{2ex}
\textbf{Simplified Scratchpad}
\begin{lstlisting}[language=markdown]
Input:
32+981
Target:
(*@\highlighttext{A->3 , C->0}@*)
(*@\highlighttext{A->2 , C->1}@*)
(*@\highlighttext{A->0 , C->1.}@*)
(*@\highlighttext{1013\#\#\#}@*)
\end{lstlisting}
\end{minipage}
\begin{minipage}[t]{0.66\linewidth}
\centering

\textbf{Detailed Scratchpad}
\begin{lstlisting}[language=markdown]
Input:
356+787
Target:
(*@\highlighttext{<scratch>}@*)
(*@\highlighttext{[3,5,6] has 3 digits.}@*)
(*@\highlighttext{[7,8,7] has 3 digits.}@*)
(*@\highlighttext{[3,5,6] + [7,8,7] , A=[] , C=0 , 6+7+0=13 , A->3 , C->1}@*)
(*@\highlighttext{[3,5] + [7,8] , A=[3] , C=1 , 5+8+1=14 , A->4 , C->1}@*)
(*@\highlighttext{[3] + [7] , A=[4,3] , C=1 , 3+7+1=11 , A->1 , C->1}@*)
(*@\highlighttext{[] + [] , A=[1,4,3] C=1 , END}@*)
(*@\highlighttext{</scratch>}@*)
(*@\highlighttext{1 1 4 3\#\#\#}@*)
\end{lstlisting}

\end{minipage}
\end{AIbox}

\subsubsection{Subtraction}

\begin{AIbox}[breakable]{\bf{\large Subtraction Examples}}
\vspace{5mm}
\begin{minipage}[t]{0.35\linewidth}
\centering
\textbf{Plain}

\begin{lstlisting}[language=markdown]
2 0 4 - 5 0 1 =(*@\highlighttext{ - 2 9 7\#\#\#}@*)
\end{lstlisting}

\vspace{2ex}
\textbf{Reverse}
\begin{lstlisting}[language=markdown]
7 3 4 - 9 6 7 =(*@\highlighttext{ 3 3 2 -\#\#\#}@*)
\end{lstlisting}

\vspace{2ex}
\textbf{Simplified Scratchpad}
\begin{lstlisting}[language=markdown]
Input:
695-489
Target:
(*@\highlighttext{A->6 , C->-1}@*)
(*@\highlighttext{A->0 , C->0}@*)
(*@\highlighttext{A->2 , C->0}@*)
(*@\highlighttext{200+6=206.}@*)
(*@\highlighttext{206\#\#\#}@*)
\end{lstlisting}
\end{minipage}
\begin{minipage}[t]{0.66\linewidth}
\centering
\textbf{Detailed Scratchpad}
\begin{lstlisting}[language=markdown]
Input:
848-367
Target:
(*@\highlighttext{<scratch>}@*)
(*@\highlighttext{[8,4,8] has 3 digits.[3,6,7] has 3 digits.}@*)
(*@\highlighttext{[8,4,8] - [3,6,7] , A=[] , C=0 , 8-7-0=1 , A->1 , C->0}@*)
(*@\highlighttext{[8,4] - [3,6] , A=[1] , C=0 , 4-6-0+10=8 , A->8 , C->-1}@*)
(*@\highlighttext{[8] - [3] , A=[8,1] , C=-1 , 8-3-1=4 , A->4 , C->0}@*)
(*@\highlighttext{[] - [] , A=[4,8,1]}@*)
(*@\highlighttext{400+81=481 , END}@*)
(*@\highlighttext{</scratch>}@*)
(*@\highlighttext{4 8 1\#\#\#}@*)
\end{lstlisting}

\end{minipage}
\end{AIbox}

\newpage
\subsubsection{Sine}
\begin{AIbox}[breakable]{\bf{\large Sine Examples}}
\vspace{5mm}
\begin{minipage}[t]{0.35\linewidth}
\centering
\textbf{Plain}

\begin{lstlisting}[language=markdown]
sin(-0.8649)
(*@\highlighttext{ -0.7611\#\#\#}@*)
\end{lstlisting}
\end{minipage}
\begin{minipage}[t]{0.66\linewidth}
\centering
\textbf{Detailed Scratchpad}
\begin{lstlisting}[language=markdown]
Input:
sin(-1.3516)
Target:
(*@\highlighttext{x\_0=-1.3516}@*)
(*@\highlighttext{x\_1: -1.3516 - 1/3! * (x*x*x) , x\_1=-0.9401}@*)
(*@\highlighttext{x\_2: -0.9401 + 1/5! * (x*x*x*x*x) , x\_2=-0.9777}@*)
(*@\highlighttext{x\_3: -0.9777 - 1/7! * (x*x*x*x*x*x*x) , x\_3=-0.9761}@*)
(*@\highlighttext{x\_4: -0.9761 + 1/9! * (x*x*x*x*x*x*x*x*x) , x\_4=-0.9762 , END}@*)
(*@\highlighttext{</scratch>}@*)
(*@\highlighttext{-0.9762\#\#\#}@*)
\end{lstlisting}

\end{minipage}
\end{AIbox}

\subsubsection{Square Root}
\begin{AIbox}[breakable]{\bf{\large Square Root Examples}}
\vspace{5mm}
\begin{minipage}[t]{0.35\linewidth}
\centering
\textbf{Plain}

\begin{lstlisting}[language=markdown]
sqrt(1.2178)
(*@\highlighttext{ 1.1035\#\#\#}@*)
\end{lstlisting}
\end{minipage}
\begin{minipage}[t]{0.66\linewidth}
\centering
\textbf{Detailed Scratchpad}
\begin{lstlisting}[language=markdown]
Input:
sqrt(5.5808)
Target:
(*@\highlighttext{<scratch>}@*)
(*@\highlighttext{x\_0=2}@*)
(*@\highlighttext{x\_1: 1/2*(2+5.5808/2)=2.3952, x\_1=2.3952}@*)
(*@\highlighttext{x\_2: 1/2*(2.3952+5.5808/2.3952)=2.3625, x\_2=2.3625}@*)
(*@\highlighttext{x\_3: 1/2*(2.3625+5.5808/2.3625)=2.3623, x\_3=2.3623}@*)
(*@\highlighttext{x\_4: 1/2*(2.3623+5.5808/2.3623)=2.3623, x\_4=2.3623 , END}@*)
(*@\highlighttext{</scratch>}@*)
(*@\highlighttext{2.3623\#\#\#}@*)
 \end{lstlisting}

\end{minipage}
\end{AIbox}